%% file: _main.tex
\documentclass{article} %
\PassOptionsToPackage{hypertexnames=false}{hyperref}
\input{_header}
\usepackage{nicefrac}       %
\usepackage{microtype}      %

\usepackage{autonum}

\newif\ifFINAL

\FINALtrue %
\iclrpreprint

\ifFINAL
   
  \def\introguide#1{}

  \renewcommand{\junghyun}[1]{{}}
  \renewcommand{\kj}[1]{{}}
  \renewcommand{\shkim}[1]{{}}
  \renewcommand{\seiyun}[1]{{}}
  \renewcommand{\minsoo}[1]{{}}
  \renewcommand{\ejlee}[1]{{}}
\else
  \usepackage{transparent}
  \usepackage[inline]{showlabels}
  \usepackage{rotating}
  \renewcommand{\showlabelfont}%
  {\transparent{0.8}\scriptsize\bf\slshape\color{Lavender}}
\fi

\title{Pointwise or Pairwise: When Do Pairwise Losses Help Reward Learning, Provably?}

\author{Junghyun Lee \\
Kim Jaechul Graduate School of AI, KAIST\\
Seoul 02455, Republic of Korea \\
\texttt{jh\_lee00@kaist.ac.kr}
\And
Minsoo Ha\thanks{Equal contributions, authorship in $\alpha$-$\beta$ order} \\
Graduate School of AI, POSTECH\\
Pohang 37673, Republic of Korea \\
\texttt{minsoo0926@postech.ac.kr}
\And
Sanghwa Kim$^*$ \\
Kim Jaechul Graduate School of AI, KAIST\\
Seoul 02455, Republic of Korea \\
\texttt{tkdghk9667@kaist.ac.kr}
\And
Yeongjong Kim$^*$ \\
Samsung Research\\
Seoul 06765, Republic of Korea \\
\texttt{yeong-j.kim@samsung.com}
\And
Eun Jee Lee$^*$ \\
Independent Researcher\\
Cambridge, MA, USA\\
\texttt{ejlee.mail@gmail.com}
\And
Seiyun Shin$^*$, \ Kwang-Sung Jun \\
Graduate School of AI, POSTECH\\
Pohang 37673, Republic of Korea \\
\texttt{seiyun923@gmail.com}\\
\texttt{kwangsungjun@postech.ac.kr}
}

\begin{document}

\maketitle
\begin{abstract}
    Pairwise losses are increasingly used for reward learning even when pointwise rewards are observed, with mixed empirical results.
    \textit{When and why do pairwise losses outperform pointwise losses?}
    We study this question in a grouped offline contextual-bandit setting allowing multiple actions per context, capturing many reward learning scenarios. %
    We compare \ValueRegression{} (\VR), which regresses observed rewards pointwise, with \ValueDifferenceRegression{} (\VDR), which regresses reward differences between a pair of actions sampled under the same context.
    We consider a semiparametric model where the mean reward is the sum of a learnable action-dependent component and an arbitrary context-dependent yet \emph{action-independent} nuisance, capturing context-specific disturbances.
    Using a unified localized analysis, we prove finite-sample regression guarantees for finite and linear function classes and translate them into offline-regret bounds.
    For finite classes, \VDR eliminates the misspecification term in the \VR bound and improves a reward-scale-dependent error term by averaging over actions within each context, a benefit absent from the corresponding \VR term.
    For linear classes, neither method uniformly dominates: within-context differencing removes nuisance-induced bias but may increase estimation variance relative to using absolute rewards when the misspecification is sufficiently low. %
    This yields a feature geometry-dependent \emph{bias--variance tradeoff}, which we corroborate with numerical experiments. 
\end{abstract}

\input{001Introduction}

\input{002Localized}

\input{003Finite}

\input{004Linear}

\input{005Conclusion}

\subsection*{\texorpdfstring{AI Disclosure Statement}{AI Disclosure Statement}}

In this work, we used generative AI tools to generate synthetic datasets; help develop theoretical models or conceptual frameworks; formulate mathematical claims and provide critical ingredients for their proofs; assist in writing proofs; propose or refine hypotheses; design or provide feedback on research methodology or experiments; implement methods; assist with translation; support qualitative and thematic data analysis; and interpret results. Dataset cleaning and reformatting were not applicable to this work. We also used generative AI tools to create or modify scientific figures or images; suggest experimental parameters; create or edit software code; draft portions of the paper; summarize or analyze existing literature; discover research topics or identify gaps; brainstorm; source or search for information; improve the paper’s readability; identify relevant literature; suggest the paper’s structure; and propose titles or keywords. We reviewed all AI-assisted work: all LLM-generated proofs were verified and polished for readability by at least two authors, all cited literature was manually checked, and LLM-generated code was verified and tested for correctness by two authors. We take responsibility for the final content of this work, including all text, claims, and artifacts produced with the aid of generative AI.

\subsubsection*{Author Contributions}
Junghyun Lee led the project, was responsible for the overall writing, and refined and finalized the proofs. Junghyun Lee also contributed to unifying the analyses of finite and linear function classes through localization.

Minsoo Ha reviewed the literature on pointwise and pairwise reward learning, organized prior methods by loss structure and downstream application, and contributed to writing parts of the introduction. Minsoo Ha also contributed to extending the \StarEstimator to the grouped-data setting for finite function classes.

Sanghwa Kim derived the initial theoretical upper bounds for \VR and \VDR and the lower bound for \VR under finite function classes, laying the foundation for the integrated theory. Sanghwa Kim also contributed to verifying the mathematical proofs.

Yeongjong Kim designed and implemented the synthetic offline linear contextual bandit experiment, generated the data for the VR--VDR comparison, and conducted the $\ell_2$-regularized linear experiments on LLM representations. Yeongjong Kim also generated the MCTS training data for the mathematical-reasoning experiment.

Eun Jee Lee designed and conducted the MCTS-based mathematical-reasoning experiment, including training linear and MLP reward heads under the VR and binary-preference VDR objectives and performing the MCTS evaluation. Eun Jee Lee also extended the VR--VDR comparison on LLM representations to additional reward models and analyzed the results across choices of scorer head and reward model.

Seiyun Shin developed the lower-bound and separation results, including the construction showing that VR can incur constant regret while VDR succeeds. Seiyun Shin also contributed to establishing the agnostic minimax lower bound and was responsible for the final verification and presentation of these theoretical results.

Kwang-Sung Jun supervised the project, provided overall research direction, and coordinated the development of the work.

\bibliography{references}
\bibliographystyle{iclr2027_conference}

\newpage
\appendix
\tableofcontents
\newpage

\crefname{appendix}{Appendix}{Appendices}
\Crefname{appendix}{Appendix}{Appendices}
\crefalias{section}{appendix}
\crefalias{subsection}{appendix}
\crefalias{subsubsection}{appendix}

\input{900RelatedWorks}

\newpage
\input{901Regret}
\newpage
\input{902Regression}
\newpage
\input{903Finite}
\newpage
\input{904Star_Estimator}
\newpage
\input{905Linear}
\newpage
\input{906LowerBounds}

\newpage
\input{907Experiments}

\end{document}

%% file: _header.tex
\PassOptionsToPackage{sort}{natbib}
\usepackage{iclr2027_conference,times}
\usepackage[utf8]{inputenc} %
\usepackage[T1]{fontenc}    %
\input{math_commands.tex}

\DeclareMathOperator{\EE}{\mathbb{E}} %

\DeclareMathOperator{\tr}{\mathrm{\normalfont tr}}

\usepackage{doi}
\usepackage[many]{tcolorbox}
\usepackage[table,dvipsnames]{xcolor}
\usepackage{pifont}

\definecolor{lightgrey}{rgb}{0.9, 0.9, 0.9}

\tcbset{
  baseboxstyle/.style={
    boxrule=0pt,          
    frame hidden,          
    sharp corners,         
    before skip=10pt,      
    after skip=10pt,       
    left=2pt, right=2pt,   
    top=2pt, bottom=2pt    
  }
}

\tcolorboxenvironment{definition}{baseboxstyle, colback=black!5}   
\tcolorboxenvironment{theorem}{baseboxstyle, colback=blue!5}      
\tcolorboxenvironment{corollary}{baseboxstyle, colback=blue!5}   
\tcolorboxenvironment{lemma}{baseboxstyle, colback=teal!5} 
\tcolorboxenvironment{proposition}{baseboxstyle, colback=blue!5}

\definecolor{darkblue}{RGB}{0,75,130}
\definecolor{darkred}{RGB}{150,50,0}
\definecolor{darkgreen}{RGB}{0,95,70}
\definecolor{tab:blue}{RGB}{0,114,178}
\definecolor{tab:red}{RGB}{213,94,0}
\definecolor{tab:green}{RGB}{0,158,115}
\definecolor{tab:orange}{RGB}{230,159,0}
\definecolor{blue}{RGB}{0,114,178}
\definecolor{red}{RGB}{213,94,0}
\definecolor{green}{RGB}{0,158,115}
\definecolor{orange}{RGB}{230,159,0}

\hypersetup{
    colorlinks = true,
    citecolor  = darkblue,
    linkcolor  = darkred,
    filecolor  = darkblue,
    urlcolor   = darkgreen,
}

\newcommand{\junghyun}[1]{{\color{tab:red}{\bf junghyun:} #1}}
\newcommand{\kj}[1]{{\color{RedOrange}[KJ: #1]}}
\newcommand{\seiyun}[1]{{\color[rgb]{0.05,0.15,0.55} \textbf{SY:} #1}}
\newcommand{\minsoo}[1]{{\color{teal}[Minsoo: #1]}}
\newcommand{\ejlee}[1]{{\color{tab:red}{\bf Eunjee:} #1}}
\newcommand{\shkim}[1]{{\color{tab:blue}{\bf sanghwa:} #1}}

\usepackage{amsfonts, amsmath, amssymb, amsthm}
\usepackage{dsfont}
\usepackage{thmtools, thm-restate}
\usepackage{hyperref}
\usepackage{url,doi}
\usepackage[capitalize,noabbrev]{cleveref}
\usepackage{adjustbox,tabularx,booktabs,threeparttable,makecell,diagbox}

\crefname{tcb@cnt@theorem}{Theorem}{Theorems}
\Crefname{tcb@cnt@theorem}{Theorem}{Theorems}
\crefname{definition}{definition}{definitions}
\Crefname{definition}{Definition}{Definitions}
\crefname{assumption}{assumption}{assumptions}
\Crefname{assumption}{Assumption}{Assumptions}

\def\bignorm#1{\left\lVert #1 \right\rVert}
\def\bignormop#1{\left\lVert #1 \right\rVert_{\rm op}}

\newcommand{\introguide}[1]{{\color{pink}[Guide: #1]}}
\def\chrulefill{\leavevmode\leaders\hrule height 0.7ex depth \dimexpr0.4pt-0.7ex\hfill\kern0pt}
\let\cite\citep

\usepackage[normalem]{ulem}
\usepackage{color}
\usepackage{xspace}

\newenvironment{revised}
{\colorlet{kjsavedrevised}{.}\color{MidnightBlue}}%
{\color{kjsavedrevised}}%

\DeclareRobustCommand{\algname}[1]{\ifmmode\text{\textnormal{\textsc{#1}}}\else\textnormal{\textsc{#1}}\fi}
\DeclareRobustCommand{\VR}{\algname{VR}\xspace}
\DeclareRobustCommand{\VDR}{\algname{VDR}\xspace}
\DeclareRobustCommand{\ValueRegression}{\algname{Value Regression}\xspace}
\DeclareRobustCommand{\ValueDifferenceRegression}{\algname{Value Difference Regression}\xspace}
\DeclareRobustCommand{\STAR}{\algname{Star}\xspace}
\DeclareRobustCommand{\StarEstimator}{\algname{Star Estimator}\xspace}
\DeclareRobustCommand{\MCTS}{\algname{MCTS}\xspace}
\DeclareRobustCommand{\Greedy}{\algname{Greedy}\xspace}
\DeclareRobustCommand{\Pessimism}{\algname{Pessimism}\xspace}
\DeclareRobustCommand{\VRGreedy}{\algname{VR-Greedy}\xspace}

\DeclareRobustCommand{\VDRGreedy}{\algname{VDR-Greedy}\xspace}
\newcommand{\piref}{\pi_{\mathrm{ref}}}
\newcommand{\StarClass}{\gH_{\mathrm{star}}}
\newcommand{\StarNet}{\gH_{\mathrm{star},N_x}}

\allowdisplaybreaks

%% file: math_commands.tex
\usepackage{amsmath,amsfonts,bm,bbm,mathrsfs,dsfont}
\usepackage{amsthm,amssymb,mathtools}

\def\eqref#1{(\ref{#1})}

\def\vzero{{\bm{0}}}

\def\vtheta{{\bm{\theta}}}
\def\vphi{{\bm{\phi}}}

\def\vc{{\bm{c}}}

\def\ve{{\bm{e}}}

\def\vm{{\bm{m}}}

\def\vq{{\bm{q}}}
\def\vr{{\bm{r}}}
\def\vs{{\bm{s}}}
\def\vt{{\bm{t}}}
\def\vu{{\bm{u}}}
\def\vv{{\bm{v}}}
\def\vw{{\bm{w}}}

\def\vy{{\bm{y}}}
\def\vz{{\bm{z}}}

\def\mD{{\bm{D}}}

\def\mG{{\bm{G}}}
\def\mH{{\bm{H}}}
\def\mI{{\bm{I}}}

\def\mM{{\bm{M}}}

\def\mU{{\bm{U}}}

\def\mW{{\bm{W}}}

\def\mSigma{{\bm{\Sigma}}}

\DeclareMathAlphabet{\mathsfit}{\encodingdefault}{\sfdefault}{m}{sl}
\SetMathAlphabet{\mathsfit}{bold}{\encodingdefault}{\sfdefault}{bx}{n}

\def\gA{{\mathcal{A}}}

\def\gC{{\mathcal{C}}}
\def\gD{{\mathcal{D}}}

\def\gF{{\mathcal{F}}}

\def\gH{{\mathcal{H}}}

\def\gL{{\mathcal{L}}}

\def\gO{{\mathcal{O}}}

\def\gT{{\mathcal{T}}}

\def\gX{{\mathcal{X}}}

\def\gZ{{\mathcal{Z}}}

\def\sR{{\mathbb{R}}}
\def\sS{{\mathbb{S}}}

\newcommand{\Rmax}{R_{\max}}

\newcommand{\E}{\mathbb{E}}

\newcommand{\KL}{D_{\mathrm{KL}}}

\newcommand{\Var}{\mathrm{Var}}

\newcommand{\Cov}{\mathrm{Cov}}

\DeclareMathOperator*{\argmax}{arg\,max}
\DeclareMathOperator*{\argmin}{arg\,min}
\DeclareMathOperator*{\esssup}{ess\,sup}

\DeclareMathOperator{\sign}{sign}

\theoremstyle{plain}
\newtheorem{theorem}{Theorem}[section]
\newtheorem{proposition}{Proposition}[section]
\newtheorem{lemma}{Lemma}[section]

\newtheorem{definition}{Definition}[section]
\newtheorem{assumption}{Assumption}
\newtheorem{remark}{Remark}

\def\1{\mathbbm{1}}

\newcommand{\Ber}{\mathrm{Bernoulli}}

\newcommand{\Reg}{\mathrm{Reg}}
\newcommand{\indicator}{\mathds{1}}

\def\ddefloop#1{\ifx\ddefloop#1\else\ddef{#1}\expandafter\ddefloop\fi}

\def\ddef#1{\expandafter\def\csname #1#1\endcsname{\ensuremath{\mathbb{#1}}}}
\ddefloop ABCDFGHIJKLMNORSTUWXYZ\ddefloop 

\def\ddef#1{\expandafter\def\csname c#1\endcsname{\ensuremath{\mathcal{#1}}}}
\ddefloop ABCDEFGHIJKLMNOPQRSTUVWXYZ\ddefloop

\def\ddef#1{\expandafter\def\csname b#1\endcsname{\ensuremath{{\mathbf{#1}}}}}
\ddefloop ABCDEFGHIJKLMNOPQRSTUVWXYZ\ddefloop  
\def\ddef#1{\expandafter\def\csname b#1\endcsname{\ensuremath{{\boldsymbol{#1}}}}}
\ddefloop abcdeghijklnopqrtsuvwxyz\ddefloop   %

\def\ddef#1{\expandafter\def\csname h#1\endcsname{\ensuremath{\hat{#1}}}}
\ddefloop ABCDEFGHIJKLMNOPQRSTUVWXYZabcdefghijklmnopqrsuvwxyz\ddefloop %
\def\ddef#1{\expandafter\def\csname hc#1\endcsname{\ensuremath{\hat{\mathcal{#1}}}}}
\ddefloop ABCDEFGHIJKLMNOPQRSTUVWXYZ\ddefloop
\def\ddef#1{\expandafter\def\csname hb#1\endcsname{\ensuremath{\hat{\mathbf{#1}}}}}
\ddefloop ABCDEFGHIJKLMNOPQRSTUVWXYZ\ddefloop %
\def\ddef#1{\expandafter\def\csname hb#1\endcsname{\ensuremath{\hat{\boldsymbol{#1}}}}}
\ddefloop abcdefghijklmnopqrstuvwxyz\ddefloop %

\def\ddef#1{\expandafter\def\csname t#1\endcsname{\ensuremath{\tilde{#1}}}}
\ddefloop ABCDEFGHIJKLMNOPQRSTUVWXYZabcdefgijklmnpqtsuvwxyz\ddefloop %
\def\ddef#1{\expandafter\def\csname tc#1\endcsname{\ensuremath{\tilde{\mathcal{#1}}}}}
\ddefloop ABCDEFGHIJKLMNOPQRSTUVWXYZ\ddefloop
\def\ddef#1{\expandafter\def\csname tb#1\endcsname{\ensuremath{\tilde{\mathbf{#1}}}}}
\ddefloop ABCDEFGHIJKLMNOPQRSTUVWXYZ\ddefloop
\def\ddef#1{\expandafter\def\csname tb#1\endcsname{\ensuremath{\tilde{\boldsymbol{#1}}}}}
\ddefloop abcdefghijklmnopqrstuvwxyz\ddefloop %

\def\ddef#1{\expandafter\def\csname bar#1\endcsname{\ensuremath{\bar{#1}}}}
\ddefloop ABCDEFGHIJKLMNOPQRSTUVWXYZabcdefghijklmnopqrtsuvwxyz\ddefloop
\def\ddef#1{\expandafter\def\csname barc#1\endcsname{\ensuremath{\bar{\mathcal{#1}}}}}
\ddefloop ABCDEFGHIJKLMNOPQRSTUVWXYZ\ddefloop
\def\ddef#1{\expandafter\def\csname barb#1\endcsname{\ensuremath{\bar{\mathbf{#1}}}}}
\ddefloop ABCDEFGHIJKLMNOPQRSTUVWXYZ\ddefloop
\def\ddef#1{\expandafter\def\csname barb#1\endcsname{\ensuremath{\bar{\boldsymbol{#1}}}}}
\ddefloop abcdefghijklmnopqrstuvwxyz\ddefloop %

\def\ddef#1{\expandafter\def\csname war#1\endcsname{\ensuremath{\overline{#1}}}}
\ddefloop ABCDEFGHIJKLMNOPQRSTUVWXYZabcdefghijklmnopqrtsuvwxyz\ddefloop
\def\ddef#1{\expandafter\def\csname warc#1\endcsname{\ensuremath{\overline{\mathcal{#1}}}}}
\ddefloop ABCDEFGHIJKLMNOPQRSTUVWXYZ\ddefloop
\def\ddef#1{\expandafter\def\csname warb#1\endcsname{\ensuremath{\overline{\mathbf{#1}}}}}
\ddefloop ABCDEFGHIJKLMNOPQRSTUVWXYZ\ddefloop
\def\ddef#1{\expandafter\def\csname warb#1\endcsname{\ensuremath{\overline{\boldsymbol{#1}}}}}
\ddefloop abcdefghijklmnopqrstuvwxyz\ddefloop %

\def\epsilon{\varepsilon}

\usepackage{pgffor}
\def\greeksymbols{alpha,beta,gamma,gam,delta,dt,eps,epsilon,zeta,eta,theta,th,iota,kappa,kap,lambda,lam,mu,nu,xi,pi,rho,sigma,sig,tau,phi,chi,psi,omega,om,Gamma,Gam,Delta,Dt,Theta,Th,Lambda,Lam,Pi,Sigma,Sig,Phi,Psi,Omega,Om}
\def\greeksymbolsnoeta{alpha,beta,gamma,gam,delta,dt,eps,epsilon,zeta,theta,th,iota,kappa,kap,lambda,lam,mu,nu,xi,pi,rho,sigma,sig,tau,phi,chi,psi,omega,om,Gamma,Gam,Delta,Dt,Theta,Th,Lambda,Lam,Pi,Sigma,Sig,Phi,Psi,Omega,Om} %

\foreach \x in \greeksymbolsnoeta{\expandafter\xdef\csname b\x\endcsname{\noexpand\ensuremath{\noexpand\boldsymbol{\csname \x\endcsname}}}}

\foreach \x in \greeksymbols{\expandafter\xdef\csname h\x\endcsname{\noexpand\ensuremath{\noexpand\hat{\csname \x\endcsname}}}}
\foreach \x in \greeksymbolsnoeta{\expandafter\xdef\csname hb\x\endcsname{\noexpand\ensuremath{\noexpand\hat{\noexpand\boldsymbol{ \csname \x\endcsname}}}}}

\foreach \x in \greeksymbols{\expandafter\xdef\csname bar\x\endcsname{\noexpand\ensuremath{\noexpand\bar{\csname \x\endcsname}}}}
\foreach \x in \greeksymbolsnoeta{%
\expandafter\xdef\csname barb\x\endcsname{\noexpand\ensuremath{\noexpand\bar{\noexpand\boldsymbol{ \csname \x\endcsname}}}}
}

\foreach \x in \greeksymbols{\expandafter\xdef\csname t\x\endcsname{\noexpand\ensuremath{\noexpand\tilde{\csname \x\endcsname}}}}
\foreach \x in \greeksymbolsnoeta{\expandafter\xdef\csname tb\x\endcsname{\noexpand\ensuremath{\noexpand\tilde{\noexpand\boldsymbol{ \csname \x\endcsname}}}}}

\usepackage{commath}

\providecommand{\normz}[2][-1]{
\ensuremath{\mathinner{
\ifthenelse{\equal{#1}{-1}}{ %
\!\left\|#2\right\|}{}
\ifthenelse{\equal{#1}{0}}{ %
\|#2\|}{}
\ifthenelse{\equal{#1}{1}}{ %
\bigl\|#2\bigr\|}{}
\ifthenelse{\equal{#1}{2}}{ %
\Bigl\|#2\Bigr\|}{}
\ifthenelse{\equal{#1}{3}}{ %
\biggl\|#2\biggr\|}{}
\ifthenelse{\equal{#1}{4}}{ %
\Biggl\|#2\Biggr\|}{}
}} %
}  %
\usepackage{algorithm, multicol, algorithmicx}
\usepackage{algpseudocode}

%% file: 001Introduction.tex
\section{\texorpdfstring{Introduction}{Introduction}}
\label{sec:introduction}
Modern machine learning increasingly relies on learned reward functions to turn limited supervision into reusable objectives that can be optimized at scale. Examples include reinforcement learning from human feedback (RLHF)~\citep{christiano2017deep}, Best-of-$N$ test-time scaling~\citep{huang2025best}, reward- or value-guided reasoning and search~\citep{guan2025rstar,li25process}, and, outside language modeling, automated essay scoring and speech-quality assessment~\citep{xie2022automated,cao2026scores}. In many of these settings, supervision is \emph{pointwise and numerical}, such as human ratings~\citep{wang24helpsteer}, verifier scores~\citep{lambert25tulu}, or search-derived value estimates~\citep{guan2025rstar}.
In such cases, it is natural to regress directly on the observed rewards.

Several methods, on the other hand, \emph{deliberately} construct pairwise targets, such as preferences or reward differences between actions under the same context, even when pointwise values are available.
In language-model reasoning, \citet{guan2025rstar} convert search-derived Q-values into high-Q/low-Q pairs and use a Bradley--Terry loss~\citep{bradleyterry1952} to train a process reward model.
\citet{cui2023ultrafeedback} construct preference pairs from numerical GPT-4 ratings and use a margin-adjusted Bradley--Terry loss to train a reward model.
Beyond language modeling, pairwise squared losses on numerical label differences have been used for magnitude-preserving ranking~\citep{cortes2007magnitude} and as a component of regression objectives~\citep{zhu2024gradient}. 
Related pairwise objectives also appear even when numerical labels, such as essay scores and speech-quality scores, are available~\citep{xie2022automated,cao2026scores}.
In these examples, pairing is an \emph{algorithmic choice} rather than an inherent property of the raw feedback. %

Despite the growing use of pairwise losses, the empirical benefits of pairing are not consistent.
Some studies report ablation results supporting the advantages of pairing over pointwise objectives.
For example, pairwise preference training improves mathematical reasoning accuracy over pointwise Q-value regression \citep{guan2025rstar} and speech quality ordering accuracy over pointwise score regression \citep{cao2026scores}.
However, some other studies report that pairing does not always improve over pointwise alternatives. %
In reward-learning ablations on Atari games, 
neither preference-based training nor pointwise regression to returns consistently yields better downstream returns~\citep{christiano2017deep}.
In ranking experiments, either objective can outperform the other depending on the choice of kernel~\citep{pahikkala2013efficient}.
These mixed results motivate a more rigorous understanding of which properties of the learning problem determine whether pairing helps downstream decisions.

We therefore ask:
\begin{center}
\textbf{\emph{Given pointwise reward data, when does learning a reward function with pairwise losses lead to provably better downstream decisions than learning with pointwise losses?
}}
\end{center}
We make significant progress towards answering this question through the lens of contextual bandits along with representative pairwise and pointwise losses.

\subsection{\texorpdfstring{Theoretical Setup: Grouped Offline Reward Learning}{Theoretical Setup: Grouped Offline Reward Learning}}
Let $\gX$ be a context (prompt) space, $\gA$ be an action (response) space with $|\gA| < \infty$, and $\gF \subseteq \{ f : \gX \times \gA \rightarrow \sR \}$ be a value function space.
$r^\star : \gX \times \gA \rightarrow \sR$ is an \emph{unknown} reward function.

\paragraph{Grouped dataset.}
With a prompt distribution $\rho \in \Delta(\gX)$ and behavior policy $\piref(\cdot \mid x) \in \Delta(\gA)$, the learner is given a (passively) collected, \emph{grouped} offline dataset
\begin{equation}
    \mathcal D
    \coloneq \{(x_i,a_{i,j},r_{i,j})\}_{i\in[N_x],\,j\in[N_a]},
    \qquad N\coloneq N_xN_a,
\end{equation}
where the contexts $x_i\sim\rho$ are i.i.d.; conditional on $x_i$, the actions $a_{i,j}\sim\piref(\cdot\mid x_i)$ are i.i.d.; and
$r_{i,j}=r^\star(x_i,a_{i,j})+\eta_{i,j}$ satisfying $|r_{i,j}| \leq {\color{red}\Rmax}$ \emph{a.s.} for a known ${\color{red}\Rmax} > 0$.
Conditional on $\gZ\coloneq\{(x_i,a_{i,j})\}_{i,j}$, the $\eta_{i,j}$'s are independent, their conditional laws depend on $\gZ$ only through $(x_i,a_{i,j})$, and they satisfy $\E[\eta_{i,j}\mid\gZ]=0$, $\E[\eta_{i,j}^2\mid\gZ] \leq \sigma^2$, and $|\eta_{i,j}| \leq 2{\color{red}\Rmax}$.\footnote{Indeed, by triangle inequality, $|\eta_{i,j}| \leq |r_{i,j}| + |r^\star(x_i, a_{i,j})| \leq {\color{red}\Rmax} + |\E[r_{i,j} \mid x_i, a_{i,j}]| \leq 2{\color{red}\Rmax}$.}
It may be that $N_a\geq2$, as is the case when multiple candidate responses are sampled for reward modeling or group-based policy optimization~\citep{guan2025rstar,cui2023ultrafeedback,shao2024deepseekmath}.
Henceforth, let $\E_{x,a} \coloneq \E_{x \sim \rho, a \sim \piref(\cdot\mid x)}$ whenever there is no ambiguity.

\paragraph{Offline regret.}
The performance of a policy $\pi : \gX \to \Delta(\gA)$ is measured with the \textbf{offline regret}:
\begin{equation}
\label{eqn:offline-regret}
    \Reg(\pi) \coloneq  V(r^\star) - V(\pi), \quad
    V(\pi) \coloneq  \E_{x\sim\rho, a\sim\pi(\cdot\mid x)}\left[ r^\star\left( x, a \right) \right],
\end{equation}
where $V(\pi)$ is the \textbf{value} of the policy $\pi$, and with a slight notation abuse, for $g \in \gF$, $V(g)$ is the value of the \Greedy policy $x \mapsto \argmax_{a \in \gA} g(x, a)$, following the standard convention in value-based offline RL~\citep{chen2019information}.
Define the \textbf{\emph{centering operator}} $\gC : f \mapsto \gC f$ as
\begin{equation}
\label{eq:centering-operator}
    \gC f(x,a) \coloneq f(x, a) - \E_{a \sim \piref(\cdot\mid x)}[f(x, a)],
\end{equation}
which is action-gap-preserving in the sense that $\gC f(x,a) - \gC f(x,a') = f(x,a) - f(x,a')$.
Then, based on standard coverage-based arguments~\citep{rashidinejad2022offline}, we show that a bound on the \emph{centered prediction error} of a given value estimator $\hat{f}$, $\E[(\gC\hat{f} - \gC r^\star)^2]$, implies an offline regret bound via \Greedy or \Pessimism; see \Cref{prop:regret} in \Cref{app:offline-regret}.

\paragraph{Weak realizability.}
We impose the following semiparametric-type weak realizability assumption on $r^\star$ (this subsumes realizability as it may be that $b^\star = 0$):
\begin{assumption}[Weak Realizability]
\label{asm:weak_realizability}
    There exist $f^\star \in \gF$ and $b^\star : \gX \rightarrow \sR$ such that $r^\star(x, a) = f^\star(x, a) + b^\star(x)$.
    In particular, we do not assume realizability on either $b^\star$ or $r^\star$; only on $f^\star$.
\end{assumption}

As it may be that $r^\star \not\in \gF$, we define the \textbf{misspecification error}:
\begin{equation}
    {\color{red}\varepsilon_{\rm aprx}}
    \coloneq \inf_{f \in \gF}
    \E_{x,a}\bigl[(f(x,a)-r^\star(x,a))^2\bigr] \geq 0.
    \label{eq:approx-error}
\end{equation}
This arises in many scenarios.
Under Bradley--Terry preferences~\citep{bradleyterry1952}, adding a prompt-only term changes neither the preference model nor the optimal KL-regularized policy~\citep{rafailov2023direct}.
Action-independent nuisances likewise arise in semiparametric contextual bandits and regression~\citep{greenewald2017action,krishnamurthy2018semiparametric,kim2019contextual,kim2025semiparametric,carranza2023flexible,nie2021quasi,foster2023orthogonal}.

\paragraph{Algorithms: \VR and \VDR.}
We now introduce the two value estimators studied in this paper, which formalize the pointwise and pairwise approaches described at the beginning.
For the pointwise loss, we consider \ValueRegression{} (\VR), which directly regresses the observed rewards:
\begin{equation}
\label{eqn:VR}
    \hat f_N^{\VR} \in \argmin_{f \in \gF} \left\{ \widehat{\gL}^{\VR}(f) \coloneq
    \frac{1}{N_x} \sum_{i=1}^{N_x} \frac{1}{N_a} \sum_{j=1}^{N_a}
    \bigl(f(x_i,a_{i,j})-r_{i,j}\bigr)^2 \right\}.
    \tag{\VR}
\end{equation}
For the pairwise loss, there are many variations such as the Bradley--Terry loss~\citep{bradleyterry1952}, margin loss~\citep{chen2024step}, and squared difference loss~\citep{xie2022automated}.
For ease of exposition, we choose a simple variant, \ValueDifferenceRegression{} (\VDR), that regresses the reward difference: denoting $\Delta_{i,j,k}(f) \coloneq f(x_i,a_{i,j})-f(x_i,a_{i,k})$ and $\widehat{\Delta}_{i,j,k}(r) \coloneq r_{i,j} - r_{i,k}$,
\begin{equation}
\label{eqn:VDR}
    \hat{f}_N^{\VDR}
    \in \argmin_{f \in \gF} \left\{ \widehat{\gL}^{\VDR}(f) \coloneq 
    \frac{1}{N_x} \sum_{i=1}^{N_x}
    \frac{1}{\binom{N_a}{2}}
    \sum_{1 \leq j < k \leq N_a}
    \left(
        \Delta_{i,j,k}(f)
        -\widehat{\Delta}_{i,j,k}(r)
    \right)^2 \right\}.
    \tag{\VDR}
\end{equation}
Using the centering identity of \citet[Corollary 2]{zhu2024gradient}, $\widehat{\gL}^{\rm VDR}$ can be evaluated exactly in linear time by centering residuals within each context (the essence is the identity $\EE[(X-Y)^2] = 2\EE[(X-\EE[X])^2]$ where $X$ and $Y$ are i.i.d). 
Pair-selection strategies are also used in preference learning~\citep{guan2025rstar,pukdee2026reward}, but are outside our scope.

\begin{table}[t]
    \centering
    \normalsize
    \renewcommand{\arraystretch}{1.8} %
    \caption{\label{tab:results}
    Centered squared prediction error bounds for \ValueRegression~\eqref{eqn:VR} and \ValueDifferenceRegression~\eqref{eqn:VDR} over finite and linear classes.
    We omit universal constants; for $\gF_{\rm lin}$, we also suppress $\log\frac{1}{\delta}$ factors.
    For $\gF_{\rm lin}$, ${\color{red}\varepsilon_{\mathrm{z}}}$ and ${\color{red}\varepsilon_{\mathrm{blk}}}$ measure the geometric impact of misspecification on \VR
    (${\color{red}\varepsilon_{\mathrm{blk}}}\coloneq{\color{red}\varepsilon_{\mathrm{x}}}+{\color{red}\varepsilon_{\mathrm{a}}}/N_a$ in Eqn.~\eqref{eq:vr-linear-residual-quantities}).
    ${\color{blue}\alpha_0}\geq1$ (resp. ${\color{blue}\alpha_{\rm C}}\geq1$) quantifies how well-conditioned the (resp. centered) features are; see \Cref{asm:linear-vr-design} (resp. \Cref{asm:linear-vdr-design}).
    $1\leq {\color{blue}d_{\rm C}} \leq d$ is the dimension of the identifiable subspace for centered features.
    }
    \begin{adjustbox}{max width=\textwidth}
    \begin{threeparttable}
    \begin{tabular}{|c|c|c|} \hline 
         & \eqref{eqn:VR} & \eqref{eqn:VDR} \\
        \hline
        $|\gF|<\infty$\tnote{$\dagger$}
        & \makecell{$\displaystyle 2{\color{red}\varepsilon_{\rm aprx}} +
          \left(\frac{F^2 + F {\color{red}\Rmax}}{N_x}
          +\frac{\sigma^2}{N}\right)
          \log\!\frac{|\gF|}{\delta}$ \\ \textbf{\small(Eqn.~\eqref{eq:vr-finite-rate} of \Cref{thm:finite})}} 
        & \makecell{$\displaystyle \left( \frac{F^2}{N_x} + \frac{\sigma^2 + F{\color{red}\Rmax}}{N} \right) \log\frac{|\gF|}{\delta}$ \\ \textbf{\small(Eqn.~\eqref{eq:vdr-finite-rate} of \Cref{thm:finite})}} \\
        \hline
        $\gF_{\rm lin}$\tnote{$\ddagger$} 
        & \makecell{$\displaystyle {\color{red}\varepsilon_{\rm aprx}}
          + \frac{{\color{red}\varepsilon_{\mathrm{blk}}}d}{N_x}
          + \frac{{\color{red}\varepsilon_{\mathrm{z}}}d}{N_x^2}
          + \frac{\sigma^2 d}{N}
          + \frac{{\color{red}\Rmax^2}{\color{blue}\alpha_0^2}d}{N^2}$ \\ \textbf{\small(\Cref{thm:vr-linear})}} 
        & \makecell{$\displaystyle \frac{\sigma^2 {\color{blue}d_{\rm C}}}{N} + \frac{{\color{red}\Rmax^2}{\color{blue}\alpha_{\rm C}^2d_{\rm C}}}{N^2}$ \\ \textbf{\small(\Cref{thm:vdr-linear})}} \\
        \hline
    \end{tabular}\label{tab:vr-vdr-finite-lin}
    \begin{tablenotes}
        \item[$\dagger$] {$F \coloneq \sup_{f \in \gF} \bignorm{f}_\infty$, and $F$ may be much smaller than ${\color{red}\Rmax}$.}        \item[$\ddagger$] {$N_x$ must be sufficiently large; see the theorems. Also, ${\color{red}\varepsilon_{\mathrm{z}}}={\color{red}\varepsilon_{\mathrm{blk}}}=0$ when ${\color{red}\varepsilon_{\rm aprx}}=0$.
        }
    \end{tablenotes}
    \end{threeparttable}
    \end{adjustbox}
\end{table}

\subsection{\texorpdfstring{Our Contributions}{Our Contributions}}

Our main departure from the prior literature is the \emph{grouped} offline dataset $\gD$: standard formulations typically require $N_a = 1$, whereas we allow for $N_a \geq 2$ in which case the $N$ triplets become \emph{non-i.i.d.}
Thus, statistical accuracy may depend on not only $N$ but also $N_x$; the term with $N_x$ governs generalization \emph{across contexts}, while the term with $N_a$ governs concentration of \emph{within-context} variations.
We focus on finite-sample guarantees with sharp dependence on these two sample sizes $N_x$ and $N_a$, as well as on the complexity of $\gF$ ($\log|\gF|$, $d$), misspecification-related terms ($\color{red}\varepsilon_{\rm aprx}, \varepsilon_{\mathrm{z}}, \varepsilon_{\mathrm{x}}, \varepsilon_{\mathrm{a}}$), the reward scale $\color{red}R_{\max}$, and feature geometry-related terms (${\color{blue}\alpha_0,\alpha_{\rm C},d_{\rm C}}$).

Our contributions are as follows, along with the rest of the paper's organization:

\noindent\textbf{Grouped finite-sample analysis (\textbf{\Cref{sec:localised}}).}
Building on localized squared-loss analyses~\citep{bartlett2005local,liang15learning}, we develop a common blockwise analysis recipe for \VR and \VDR that distinguishes terms governed by $1/N_x$ from those governed by $1/N$.

\noindent\textbf{Finite classes (\textbf{\Cref{sec:finite}}, first row of \Cref{tab:results}).}
Under weak realizability, \VDR eliminates the $\color{red}\varepsilon_{\rm aprx}$ appearing in the \VR guarantee and improves the $F{\color{red}R_{\max}}$ term from the context scale $1/N_x$ to the full-sample scale $1/N$.
We complement this by reporting an instance where \VR fails yet \VDR succeeds in achieving $\widetilde{\gO}(1/\sqrt{N})$ offline regret.

\noindent\textbf{Linear classes (\Cref{sec:linear}, second row of \Cref{tab:results}).}
Next, we consider the linear class
$\gF_{\mathrm{lin}}\coloneq\{(x,a)\mapsto\langle\vphi(x,a),\vtheta\rangle:\vtheta\in\sR^d\}$
with a known feature map $\vphi:\gX\times\gA\to\sR^d$.
We show that \emph{neither method uniformly dominates}: their relative performance depends on the \emph{feature geometry}.
While \VDR eliminates the misspecification terms ${\color{red}\varepsilon_{\rm aprx},\varepsilon_{\mathrm{z}},\varepsilon_{\mathrm{x}},\varepsilon_{\mathrm{a}}}$ present in the \VR bound, their $N^{-2}$ terms scale with different geometric quantities: ${\color{blue}\alpha_0^2}$ for \VR, determined by the raw features, and ${\color{blue}\alpha_{\rm C}^2}$ for \VDR, determined by feature differences.
This yields a geometric bias--variance tradeoff: \VDR removes nuisance-induced misspecification error (\emph{bias}), but its \emph{variance} depends on the geometry of feature differences, which can be less favorable than the raw feature geometry governing \VR.
Numerical experiments on toy and real-world LLM tasks illustrate this tradeoff.

%% file: 002Localized.tex
\section{\texorpdfstring{A Localized Analysis for Grouped Regression}{A Localized Analysis for Grouped Regression}}
\label{sec:localised}

Let $\vr = (r_{i,j})_{i \in [N_x], j \in [N_a]}$ denote the vector of realized (noisy) rewards.
We consider the following generic empirical and population squared losses:
\begin{equation}
    \widehat{\gL}(f) \coloneq \frac{1}{N_x} \sum_{i=1}^{N_x} \bignorm{\gT_i (f - \vr)}_2^2, \qquad \gL(f) \coloneq \E\left[ \bignorm{\gT_i(f - r^\star)}_2^2 \right],
\end{equation}
where $\gT : \sR^{N_a} \rightarrow \sR^m$ is an \emph{algorithm-dependent linear transformation}, $\gT_i f \coloneq \gT\left( (f(x_i, a_{i,1}), \ldots, f(x_i, a_{i,N_a})) \right)$, and $\gT_i \vr \coloneq \gT\left( (r_{i,1}, \ldots, r_{i,N_a}) \right)$.
\VR and \VDR correspond to $\gT^{\VR} : \vu \mapsto \frac{1}{\sqrt{N_a}} \vu$ and $\gT^{\VDR} : \vu \mapsto \binom{N_a}{2}^{-1/2} (u_j - u_k)_{1 \leq j < k \leq N_a}$, respectively, for $\vu \in \sR^{N_a}$.

We now define the following notion of empirical complexity measure:
\begin{definition}[Localized block-offset modulus]
\label{def:localised-modulus}
    For a comparator $f^\circ \in \gF$ and $\lambda, r \geq 0$, define the \textbf{localized block-offset modulus} $\Psi(r; \lambda, f^\circ)$ as follows:
    \begin{equation}
        \Psi(r; \lambda, f^\circ) \coloneq \sup_{\substack{h \in \gF - f^\circ \\ \gL(f^\circ + h) - \gL(f^\circ) \leq r}} \left[ \left( \gL(f^\circ + h) - \gL(f^\circ)\right) - \left( \widehat{\gL}(f^\circ + h) - \widehat{\gL}(f^\circ) \right) - \lambda \widehat{Q}(h) \right],
    \end{equation}
    where $\gF - f^\circ \coloneq \{ f - f^\circ : f \in \gF \}$ and $\widehat{Q}(h) \coloneq \frac{1}{N_x} \sum_{i=1}^{N_x} \bignorm{\gT_i h}_2^2$.
\end{definition}

We now present a fixed-point inequality for the excess loss, whose proof is deferred to \Cref{app:proof-prop-localised}:
\begin{proposition}[Localized fixed-point inequality]
\label{prop:localised}
    Fix any comparator $f^\circ \in \gF$ and any estimator $\hat{f}_N \in \gF$.
    Suppose that the following two conditions hold: denoting $h \coloneq \hat{f}_N - f^\circ$,
    \begin{enumerate}
        \item[$(i)$] There is a constant $\lambda \geq 0$ such that our estimator $\widehat{f}_N$ satisfies
        \begin{equation}
            \widehat{\gL}(\hat{f}_N) - \widehat{\gL}(f^\circ) \leq - \lambda \widehat{Q}(\hat{f}_N - f^\circ);
        \end{equation}
        \item[$(ii)$] There is a deterministic $\psi:\sR\rightarrow\sR$ such that $\Psi(r;\lambda,f^\circ)\leq\psi(r)$ for all $r\geq0$.
    \end{enumerate}
    Then, we have the following fixed-point inequality for the population excess loss:
    \begin{equation}
        \gL(\hat{f}_N) - \gL(f^\circ) \leq \sup \{ r \geq 0 : r \leq \psi(r) \}.
    \end{equation}
\end{proposition}

Although seemingly abstract, \Cref{prop:localised} provides a common recipe for deriving tight statistical rates for \VR and \VDR over generic $\gF$: verify conditions~$(i)$ and~$(ii)$.
In our applications, condition~$(i)$ follows deterministically from the estimator's loss geometry, whereas condition~$(ii)$ is established (with high probability) using problem-specific concentrations.

\paragraph{Relation to prior work.}
The fixed-point perspective in \Cref{prop:localised} follows the classical localization of empirical and Rademacher processes~\citep{koltchinskii-panchenko,bousquet2002phd,bousquet2002local,bartlett2005local,koltchinskii2011}, all of which are based on Talagrand's celebrated concentration inequalities~\citep{talagrand1995,talagrand1996independence,talagrand1996product}.
For squared loss, related formulations include offset Rademacher process~\citep{rakhlin2014online,liang15learning}, quadratic--multiplier and small-ball methods~\citep{mendelson2014without,mendelson2018without}, and empirical-entropy characterizations of sharp rates~\citep{rakhlin2017empirical}.
In our modulus, $-\lambda\widehat{Q}(h)$ corresponds to the quadratic term of the offset Rademacher process~\citep{rakhlin2014online,liang15learning}.

Our main technical distinction is how we verify condition~$(ii)$ under grouped sampling.
Working directly with the \emph{unsymmetrized} block-offset modulus makes the two-level concentration structure explicit.
For finite classes, we use scalar blockwise concentration and union bounds; for linear classes, we combine a lower-isometry event with direct vector and matrix concentration.
This preserves the two sampling scales, separating contributions by $N_x^{-1}$ and $N^{-1}$ with relevant dependencies.

%% file: 003Finite.tex
\section{\texorpdfstring{Finite Class: Nuisance Cancellation and Gains from $N_a \geq 2$}{Finite Class: Nuisance Cancellation and Gains from Na >= 2}}
\label{sec:finite}

For finite classes, we consider $F \coloneq  \sup_{f\in\mathcal{F}} \|f\|_\infty$, which may satisfy $F \ll \color{red}\Rmax$.
Since $f^\star\in\gF$, $F$ controls the within-context action gaps, whereas ${\color{red}\Rmax}$ and $\color{red}\varepsilon_{\rm aprx}$ may also reflect the nuisance $b^\star$.
With this, we present the error bounds for finite classes: assuming $N_a \geq 2$ for \VDR,
\begin{theorem}[\VR \& \VDR Bound for Finite $\gF$]
\label{thm:finite}
    Let $1 \leq |\gF| < \infty$, $\delta\in(0,1)$, and suppose that \Cref{asm:weak_realizability} holds.
    Then, each of the following holds with probability at least $1-\delta$:
    \begin{align}
            \E_{x,a}\left[ \left( \gC\hat{f}_N^{\VR} - \gC r^\star \right)^2 \right]
            &\leq {\color{red}2\varepsilon_{\rm aprx}}
            + 
            \left(\frac{22F^2 + 4 F {\color{red}\Rmax}}{N_x}+\frac{18\sigma^2}{N}\right)
        \log\!\frac{|\gF|}{\delta}, \label{eq:vr-finite-rate} \\
        \E_{x,a}
        \left[
            \left(
                \gC\hat f_N^{\VDR}-\gC r^\star
            \right)^2
        \right]
        &\leq
        \frac{8}{3} \left( \frac{7 F^2}{N_x} + \frac{12\sigma^2 + 8 F {\color{red}\Rmax}}{N} \right) \log\frac{2|\gF|}{\delta}. \label{eq:vdr-finite-rate}
    \end{align}
\end{theorem}
\begin{proof}
    Here, we showcase the localized analysis recipe (\Cref{sec:localised}).
    As the analysis for \VDR proceeds similarly, we only show the proof of \VR here; \VDR's proof is presented in \Cref{app:vdr}.
    
    We set $f^\circ_{\VR} \in \argmin_{f \in \gF} \gL^{\VR}(f)$, which satisfies $\gL^{\VR}(f^\circ_{\VR}) = \color{red}\varepsilon_{\rm aprx}$, i.e.,
    \begin{equation}
    \label{eqn:VR-finite-oracle}
        \E_{x,a}\left[ (\gC\hat{f}_N^{\VR} - \gC r^\star)^2 \right] \leq \E_{x,a}\left[ (\hat{f}_N^{\VR} - r^\star)^2 \right] = \gL^{\VR}(\hat{f}_N^{\VR}) = {\color{red}\varepsilon_{\rm aprx}} + \gL^{\VR}(\hat{f}_N^{\VR}) - \gL^{\VR}(f^\circ_{\VR}),
    \end{equation}
    where the first inequality follows because the variance is bounded by the second moment: \(\E[(\gC g(x,a))^2]=\E_x[\Var[g(x,a)\mid x]] \leq \E[g(x,a)^2]\).
    As $\hat{f}_N^{\VR}$ is an empirical risk minimizer, condition $(i)$ of \Cref{prop:localised} holds with $\lambda = 0.$
    As for $(ii)$, the following lemma, proved in \Cref{app:proof-lem-finite-vr-modulus}, provides such $\psi(\cdot)$ that upper bounds the localized modulus with high probability:
    \begin{lemma}
    \label{lem:finite-vr-modulus}
        Denote $t_\delta \coloneq \log\frac{|\gF|}{\delta}$ and $v_\delta \coloneq \left( \frac{F^2}{N_x} + \frac{\sigma^2}{N} \right)t_\delta.$
        Then, with probability at least $1-\delta$:
    \begin{equation}
    \label{eq:finite-vr-modulus}
        \Psi^{\VR}(r;0,f^\circ_{\VR})
        \leq \psi(r) \coloneq 4\sqrt{ v_\delta \bigl(r+2\varepsilon_{\rm aprx}\bigr) } + \frac{8F(F+\Rmax)t_\delta}{3N_x}, \quad \forall r \geq 0.
    \end{equation}
    \end{lemma}
    We remark that the above lemma is obtained via \emph{one-sided} Bernstein's inequality~\citep{bernstein} over the blocks $i \in [N_x]$, attaining per-block variance of $F^2 + \sigma^2 N_a^{-1}$ instead of the na\"{\i}ve $F^2 + \sigma^2$.
    
    The above lemma and \Cref{prop:localised} imply $\gL^{\VR}(\hat{f}_N^{\VR}) - \gL^{\VR}(f^\circ_{\VR}) \leq \bar r$, where
    \begin{equation}
        \bar r
        \coloneq
        \sup\left\{
            r\geq0:
            r\leq
            4\sqrt{
                v_\delta(r+2\varepsilon_{\rm aprx})
            }
            +
            \frac{
                8F(F+\Rmax)t_\delta
            }{
                3N_x
            }
        \right\}.
    \end{equation}
    Using Young's inequality to separate $4\sqrt{v_\delta(r+{\color{red}2\varepsilon_{\rm aprx}})} \leq \frac13(r+{\color{red}2\varepsilon_{\rm aprx}}) + 12v_\delta$ and rearranging, we obtain $\bar r \leq {\color{red}\varepsilon_{\rm aprx}} + 18v_\delta + \frac{4F(F+\Rmax)t_\delta}{N_x}$.
    We conclude by combining this with Eqn.~\eqref{eqn:VR-finite-oracle}.
\end{proof}

\paragraph{Comparison and discussion.}
The terms scaling with $N_x^{-1}$ reflect variation across contexts; additional action samples reduce within-context fluctuations, but they do not provide information about unseen contexts without further structure on $\gF$.
\VR retains the approximation error ${\color{red}\varepsilon_{\rm aprx}}$, even when this is caused entirely by the action-independent $b^\star$.
In contrast, \VDR removes such dependency and moves the $F{\color{red}\Rmax}$ factor from $N_x^{-1}$ to $N^{-1}$, which can be substantial when ${\color{red}\Rmax} \gg F$.

Also, note that Eqn.~\eqref{eq:vr-finite-rate} is a \emph{non-exact} oracle inequality~\citep{lecue-mendelson,massart-nedelec}, as the coefficient of ${\color{red}\varepsilon_{\rm aprx}}$ is two, not one.
Indeed, it is known that ERM over an arbitrary finite $\gF$ cannot admit an exact oracle inequality with the optimal rate of $\gO(\log(|\gF|/\delta)/N)$~\citep{lecue-rigollet}, while alternate aggregation-type estimators~\citep{audibert07star,audibert09star,liang15learning,lecue-rigollet} can.
In \Cref{app:star}, we prove an oracle inequality for the \StarEstimator~\citep{audibert07star,audibert09star}, adapted to our grouped dataset setting.

\paragraph{Failure of \VR.}
We complement this with a separation result, whose proof is in \Cref{app:vr_failure}:
\begin{theorem}[Failure of \VR]
\label{thm:vr_failure}
    There exists a finite-class instance satisfying Assumption~\ref{asm:weak_realizability} such that, for every fixed $N_a \geq 1$, $\delta \in (0,1)$ and sufficiently large $N_x$, 
    \(
        \Reg(\hat\pi_N^{\VRGreedy}) = \Omega({\color{red}R_{\max}}),
    \)
    with probability at least $1 - \delta$.
    Moreover, on the same instance with $N_a \geq 2$, $\Reg(\hat\pi_N^{\VDRGreedy}) = 0,$ with probability tending to one exponentially fast in $N_x$.
\end{theorem}

%% file: 004Linear.tex
\section{\texorpdfstring{Linear Classes: A Geometric Bias–Variance Tradeoff}{Linear Classes: A Geometric Bias-Variance Tradeoff}}
\label{sec:linear}

Unlike the finite-class setting, we show a geometric bias--variance tradeoff between \VR and \VDR.
For clarity of exposition, we omit universal constants and $\log\frac{1}{\delta}$ dependencies.
Throughout, we use the random variables $X\sim\rho$ and $A,A',A_1,\ldots,A_{N_a} \overset{i.i.d.}{\sim} \piref(\cdot\mid X)$.
We denote $\Cov(\mG)\coloneq\E[(\mG-\E\mG)^2]$ for a random self-adjoint matrix $\mG$.
All proofs are presented in \Cref{app:linear}.

\subsection{\texorpdfstring{\VR for Linear Class}{VR for Linear Class}}
\label{subsec:linear_vr}

Following the bounded statistical-leverage and approximation-error conditions of \citet[Conditions 1 and 3]{hsu2014random}, we similarly impose the following assumptions:
\begin{assumption}
\label{asm:linear-vr-design}
    $\mSigma\coloneq\E[\vphi(X, A)\vphi(X, A)^\top]\succ\vzero$.
    Also, there exists ${\color{blue}\alpha_0} \geq 1$ such that, for the \textbf{\textit{whitened feature}} $\vz(x,a)\coloneq\mSigma^{-1/2}\vphi(x,a)$,
    \begin{equation}
        \esssup \bignorm{\vz(X, A)}_2 \leq {\color{blue}\alpha_0} \sqrt{d}, \qquad
        0\leq {\color{red}\varepsilon_{\mathrm{z}}}
        \coloneq d^{-1} \esssup
        \bignorm{{\color{red}e^\circ}(X, A) \vz(X, A)}_2^2
        <\infty,
    \end{equation}
    where ${\color{red}e^\circ}(x, a) \coloneq f^\circ(x, a) - r^\star(x, a)$ with $f^\circ \in \argmin_{f \in \gF_{\rm lin}} \E_{x,a}[(f(x, a) - r^\star(x, a))^2]$.
\end{assumption}
\begin{remark}[Relation to Design Objective]
    The condition $\esssup \bignorm{\vz(X, A)}_2 \leq {\color{blue}\alpha_0} \sqrt{d}$ is closely related to the G-optimal design~\citep{pukelsheim2005design}: $\bignorm{\vz(x, a)}_2^2 = \bignorm{\vphi(x, a)}_{\mSigma^{-1}}^2 \leq {\color{blue}\alpha_0^2} d$.
    Although $\rho$ and $\piref$ are fixed in our setting, in the active learning scenario, by jointly optimizing for $(\rho,\piref)$, one could achieve ${\color{blue}\alpha_0} = 1$ by the Kiefer--Wolfowitz theorem~\citep{kiefer-wolfowitz}.
\end{remark}

We define the random covariance matrix ${\color{orange}\mG_{N_a}} \coloneq \frac{1}{N_a}\sum_{j=1}^{N_a} \vz(X,A_j)\vz(X,A_j)^\top$, and non-random, across- and within-context residual quantities
\begin{equation}
{\color{red}\varepsilon_{\mathrm{x}}}
\coloneq \frac1d\E_X\!\left[
    \bignorm{\E_A[{\color{red}e^\circ}\vz\mid X]}_2^2
\right], \quad
{\color{red}\varepsilon_{\mathrm{a}}}
\coloneq \frac1d\E_X\!\left[
    \E_A\!\left[
        \bignorm{{\color{red}e^\circ}\vz-\E_A[{\color{red}e^\circ}\vz\mid X]}_2^2
        \mid X
    \right]
\right],
\label{eq:vr-linear-residual-quantities}
\end{equation}
The upright subscripts $\mathrm{x}$ and $\mathrm{a}$ label the between-context and within-context components, respectively; both quantities are population scalars.

With this, we finally present the \VR bound for the linear class.
\begin{theorem}[\VR Bound for $\gF_{\rm lin}$]
\label{thm:vr-linear}
Let $\delta\in(0,1)$, and suppose that \Cref{asm:weak_realizability,asm:linear-vr-design} hold and $N_x\geq
    \max\left\{
       18\bignormop{{\color{orange}\Cov(\mG_{N_a})}},
       2
    \right\}\log\frac{3d}{\delta}.$
Then, with probability at least $1-\delta$,\footnote{In the well-specified case, the noise term can be sharpened to \(\sigma^2d_{\mathrm{eff}}/N\), with \(d_{\mathrm{eff}}\le d\); see \Cref{app:linear-vr-alternate-error-bound}.}
\begin{equation}
    \E_{x,a}\!\left[(\gC\hat f_N^{\VR}-\gC r^\star)^2\right]
    \leq
    {\color{red}\varepsilon_{\rm aprx}}
    +\widetilde{\gO}\left(
        \frac{{\color{red}\varepsilon_{\mathrm{x}}}d}{N_x}
        +\frac{{\color{red}\varepsilon_{\mathrm{z}}}d}{N_x^2}
        +\frac{({\color{red}\varepsilon_{\mathrm{a}}}+\sigma^2)d}{N}
        +\frac{{\color{red}\Rmax^2}{\color{blue}\alpha_0^2}d}{N^2}
    \right).
\label{eq:vr-linear-rate}
\end{equation}
\end{theorem}
\begin{proof}[Proof Sketch]
    Because of the non-i.i.d. nature of our problem setup, we \emph{cannot} directly invoke the existing result for ordinary least squares (OLS) from \citet[Theorem 1]{hsu2014random}.
    Instead, we follow the same recipe as laid out in \Cref{sec:localised}, carefully combining it with \emph{vector/matrix} Bernstein inequalities~\citep{hsu2012bernstein,tropp2015survey} over the blocks $i \in [N_x]$ and a variance computation that separates the feature-diversity and misspecification terms.
\end{proof}

\paragraph{Requirement on $N_x$.}
Let $\mG(X)\coloneq\E_A[\vz\vz^\top\mid X]$.
By the law of total covariance, the block Gram variance decomposes as
\begin{equation}
{\color{orange}\Cov(\mG_{N_a})}
=\Cov_X(\mG(X))
+\frac1{N_a}\E_X\!\left[\Cov_A(\vz\vz^\top\mid X)\right].
\label{eq:vr-linear-gram-decomposition}
\end{equation}
The first term captures variation in the conditional feature covariance across contexts and is unaffected by repeated actions, whereas the second captures within-context variation and decays with $N_a$.
Thus, when $\mG(X)$ is nearly constant, repeated actions improve concentration of the empirical Gram matrix much as additional independent contexts would; otherwise, sufficiently many distinct contexts remain necessary for lower isometry.
Because these quantities are defined after whitening, they are invariant to invertible linear reparameterizations of the features and can be large when feature directions are represented unevenly across contexts.
In the special case $\mG(X)=\mI_d$ almost surely,
$\bignormop{{\color{orange}\Cov(\mG_{N_a})}}\leq({\color{blue}\alpha_0^2}d-1)/N_a$,
so the sample-size requirement in \Cref{thm:vr-linear} is satisfied if $N_x\gtrsim\log\frac{d}{\delta}$ and $N\gtrsim{\color{blue}\alpha_0^2}d\log\frac{d}{\delta}$.
Importantly, the block Gram variance enters only the sample-size condition ensuring lower isometry, not the error bound itself.

\paragraph{Interpreting the error bound.}
The conditional mean $\E_A[{\color{red}e^\circ}\vz\mid X]$ is a context-specific misspecification score and need not vanish: the population normal equation guarantees only that
$\E_X[\E_A[{\color{red}e^\circ}\vz\mid X]]=\vzero$.
Consequently, the across-context term
$\tfrac{{\color{red}\varepsilon_{\mathrm{x}}}d}{N_x}$
can persist as $N_a\to\infty$, whereas the within-context term
$\tfrac{{\color{red}\varepsilon_{\mathrm{a}}}d}{N}$
is averaged over all observations and vanishes in the same limit.
We also note that ${\color{red}\varepsilon_{\mathrm{x}}} = 0$ does \emph{not} imply ${\color{red}\varepsilon_{\mathrm{a}}}=0$ and vice-versa.

The leverage bound yields ${\color{red}\varepsilon_{\mathrm{x}}}+{\color{red}\varepsilon_{\mathrm{a}}}/N_a \leq {\color{blue}\alpha_0^2}{\color{red}\varepsilon_{\rm aprx}}$, and hence the following \emph{crude} upper bound:
\begin{equation}
    \E_{x,a}\!\left[(\gC\hat f_N^{\VR}-\gC r^\star)^2\right] \leq \left( 1 + \widetilde{\gO}\del[2]{ \frac{{\color{blue}\alpha_0^2}d}{N_x} } \right) {\color{red}\varepsilon_{\rm aprx}} + \widetilde{\gO}\left( \frac{{\color{red}\varepsilon_{\mathrm{z}}}d}{N_x^2} \right) + \widetilde{\gO}\left( \frac{\sigma^2 d}{N} + \frac{{\color{red}\Rmax^2}{\color{blue}\alpha_0^2}d}{N^2} \right).
\end{equation}
Thus, for fixed $N_x$, the oracle inequality loses exactness and retains a misspecification-related bottleneck of order $N_x^{-2}$.
The almost-sure quantity ${\color{red}\varepsilon_{\mathrm{z}}}$ cannot be controlled by the $L_2$ error ${\color{red}\varepsilon_{\rm aprx}}$ alone, but satisfies ${\color{red}\varepsilon_{\mathrm{z}}}\leq{\color{blue}\alpha_0^2}\bignorm{{\color{red}e^\circ}}_\infty^2$.

Finally, in the \emph{well-specified} setting,
${\color{red}\varepsilon_{\rm aprx}}={\color{red}\varepsilon_{\mathrm{z}}}
={\color{red}\varepsilon_{\mathrm{x}}}={\color{red}\varepsilon_{\mathrm{a}}}=0$ in \Cref{thm:vr-linear},
and the error bound becomes
$\widetilde{\gO}\!\left(
\frac{\sigma^2d}{N}
+ \frac{{\color{red}\Rmax^2}{\color{blue}\alpha_0^2}d}{N^2}
\right)$.
Thus, in the well-specified case, no leading estimation term scales only as $N_x^{-1}$, and ${\color{red}\Rmax}$ and ${\color{blue}\alpha_0}$ appear only in the lower-order Bernstein correction.

\paragraph{Relation to random-design least squares.}
When \(N_a=1\), the residual and noise terms recover the first-order dependence of classical random-design OLS, including the tight well-specified rate \(\sigma^2d/N\) under homoscedastic noise \citep[Theorem~1 and Remark~10]{hsu2014random}.
Our contribution is the grouped case \(N_a\geq2\), where the decompositions above identify which fluctuations scale with \(N_x\) and which benefit from all \(N=N_xN_a\) observations.
We discuss connections to misspecified linear regression~\citep{amortila2024mitigating,maran2025misspecification} more broadly in \Cref{app:related-work}\textbf{(3)}.

\subsection{\texorpdfstring{\VDR for Linear Class}{VDR for Linear Class}}
\label{subsec:linear_vdr}

Define the \emph{centered feature} $\overline\vphi(x,a)
\coloneq
\vphi(x,a)-\E_{a\sim\piref(\cdot\mid x)}[\vphi(x,a)]$, and its covariance as $\mSigma_{\rm C}
\coloneq
\E_{x,a}\left[\overline\vphi(x,a) \overline\vphi(x,a)^\top\right].$
Let ${\color{blue}d_{\rm C}}\coloneq\operatorname{rank}(\mSigma_{\rm C})$.
Whenever ${\color{blue}d_{\rm C}}\geq1$, let $\mU_{\rm C}\in\sR^{d\times {\color{blue}d_{\rm C}}}$ satisfy
$\mU_{\rm C}^\top\mU_{\rm C}=\mI_{{\color{blue}d_{\rm C}}}$ and
$\operatorname{col}(\mU_{\rm C})=\operatorname{range}(\mSigma_{\rm C})$, and define
\begin{equation}
    \overline{\mSigma}_{\rm C}
    \coloneq
    \mU_{\rm C}^\top\mSigma_{\rm C}\mU_{\rm C}
    \succ\vzero,
    \qquad
    \vz_{\rm C}(x,a)
    \coloneq
    \overline{\mSigma}_{\rm C}^{-1/2}
    \mU_{\rm C}^\top\overline\vphi(x,a)
    \in\sR^{{\color{blue}d_{\rm C}}}.
    \label{eq:centered-whitened-feature}
\end{equation}
Then, $\E[\vz_{\rm C}\vz_{\rm C}^\top]=\mI_{{\color{blue}d_{\rm C}}}$ and
$\E_A[\vz_{\rm C}(X,A)\mid X]=\vzero$.

Unlike \VR, \VDR does not identify directions of the parameter $\vtheta$ that add the same value to every action at a given context, since its loss depends only on within-context differences.
Thus, we state the design assumption on its identifiable subspace as follows:
\begin{assumption}
\label{asm:linear-vdr-design}
    $d_{\rm C}\geq1$, and there is a ${\color{blue}\alpha_{\rm C}} \geq 1$ s.t. $\esssup
    \bignorm{\vz_{\rm C}(X,A)-\vz_{\rm C}(X,A')}_2
    \leq
    {\color{blue}\alpha_{\rm C}}\sqrt{2 d_{\rm C}}.$
\end{assumption}

For $N_a\geq2$, denoting $\overline{\vz}_{{\rm C},N_a}(X)
\coloneq
\frac1{N_a}\sum_{j=1}^{N_a}\vz_{\rm C}(X,A_j),$ define the random, \emph{sample} covariance matrix ${\color{orange}\mG^{\rm C}_{N_a}} \coloneq \frac1{N_a-1}\sum_{j=1}^{N_a}
\left(\vz_{\rm C}(X,A_j)-\overline{\vz}_{{\rm C},N_a}(X)\right)
\left(\vz_{\rm C}(X,A_j)-\overline{\vz}_{{\rm C},N_a}(X)\right)^\top.$
\begin{theorem}[\VDR Bound for $\gF_{\rm lin}$]
\label{thm:vdr-linear}
    Let $N_a\geq2$, $\delta\in(0,1)$, and suppose that
    \Cref{asm:weak_realizability,asm:linear-vdr-design} hold, and $N_x
        \geq
        \max
        \left\{
            18\bignormop{{\color{orange}\Cov(\mG^{\rm C}_{N_a})}},
            2
        \right\}
        \log\frac{2{\color{blue}d_{\rm C}}}{\delta}$.
    Then, w.p. at least $1-\delta$,
    \begin{equation}
        \E_{x,a}
        \left[
            \bigl(
                \gC\hat f_N^{\VDR}-\gC r^\star
            \bigr)^2
        \right]
        \leq
        \widetilde{\gO}\left(
            \frac{\sigma^2{\color{blue}d_{\rm C}}}{N}
            +\frac{{\color{red}\Rmax^2}{\color{blue}\alpha_{\rm C}^2}{\color{blue}d_{\rm C}}}{N^2}
        \right).
    \label{eq:vdr-linear-rate}
    \end{equation}
\end{theorem}

\paragraph{Relation to orthogonalized regression.}
\VDR shares the nuisance-cancellation principle of orthogonalized estimation in semiparametric contextual bandits and regression~\citep{krishnamurthy2018semiparametric,kim2019contextual,choi2023semiparametric,kim2025semiparametric,robinson1988root,chernozhukov2018double,foster2023orthogonal}.
These works typically observe a single action--reward outcome per observed context and construct an orthogonal score via centering or residualization.
Our setting instead observes $N_a\geq2$ action--reward samples per observed context and forms within-context pairs; our analysis separates fluctuations across contexts from those that average over the repeated actions.

\paragraph{Interpreting the centered-feature geometry.}
${\color{blue}d_{\rm C}}$ is the \emph{dimension of the feature subspace generated by within-context action variation}; intercepts and other action-invariant directions lie in $\ker(\mSigma_{\rm C})$.
For every $\vv\in\ker(\mSigma_{\rm C})$, $\E[(\vv^\top\gC\vphi)^2]=0$, so $\vv^\top\gC\vphi=0$ almost surely.
Thus, \VDR identifies $\vtheta^\star$ only modulo this kernel, which suffices to identify $\gC f^\star$ and all action gaps on the support of $\piref$.
In general, ${\color{blue}d_{\rm C}}=d-\dim\ker(\mSigma_{\rm C})$: it equals $d$ without action-invariant directions, $d-1$ with only an intercept, and less with additional such directions.
We assume ${\color{blue}d_{\rm C}}\geq1$ (otherwise all modeled action gaps vanish on the behavior-policy support).
Finally, ${\color{blue}\alpha_{\rm C}}$ controls the maximal leverage of \emph{differences} of whitened centered features on the identifiable subspace.
When both design conditions hold, ${\color{blue}\alpha_{\rm C}}$ can be smaller or larger than ${\color{blue}\alpha_0}$; neither condition dominates, as we show next.

\subsection{Comparing \VR and \VDR: Feature Geometry Perspective}
\label{subsec:linear_comparison}

\paragraph{Comparing the bounds.}
Under \Cref{asm:weak_realizability}, the action-independent nuisance $b^\star(x)$ cancels from every within-context reward difference.  Consequently, the \VDR bound (\Cref{thm:vdr-linear}) contains no misspecification-related terms, whereas the \VR bound (\Cref{thm:vr-linear}) retains ${\color{red}\varepsilon_{\rm aprx}}$ and the residual terms ${\color{red}\varepsilon_{\mathrm{x}}}$, ${\color{red}\varepsilon_{\mathrm{a}}}$, and ${\color{red}\varepsilon_{\mathrm{z}}}$, i.e., \VDR's bound outperforms that of \VR under large misspecification.

In the well-specified case (${\color{red}\varepsilon_{\rm aprx}=\varepsilon_{\mathrm{x}}=\varepsilon_{\mathrm{a}}=\varepsilon_{\mathrm{z}}=0}$), the comparison is instead governed by the feature geometry.  The displayed bounds reduce to $\widetilde{\gO}(\sigma^2d/N+{\color{red}\Rmax^2}{\color{blue}\alpha_0^2}d/N^2)$ for \VR and $\widetilde{\gO}(\sigma^2{\color{blue}d_{\rm C}}/N+{\color{red}\Rmax^2}{\color{blue}\alpha_{\rm C}^2}{\color{blue}d_{\rm C}}/N^2)$ for \VDR.  Since ${\color{blue}d_{\rm C}}\leq d$, the leading variance term in the displayed \VDR bound is no larger; with the sharper \VR bound of \Cref{app:linear-vr-alternate-error-bound}, however, it becomes $\sigma^2{\color{blue}d_{\rm eff}}/N$ with ${\color{blue}d_{\rm eff}}\leq{\color{blue}d_{\rm C}}$.
The lower-order terms have no analogous uniform ordering across feature geometries, and hence, neither upper-bound uniformly dominates the other.

To concretely show this, we provide two simple examples where ${\color{blue}\alpha_{\rm C}}$ can be arbitrarily larger than ${\color{blue}\alpha_0}$, and vice versa.
For $p\in(0,1)$, let $X=1$, $A\sim\Ber(p)$, and set $\vphi(A)=1+A\in\sR$. 
Then, $\mSigma=1+3p$, and $\mSigma_{\rm C}=p(1-p)$, which gives ${\color{blue}\alpha_0}=2/\sqrt{1+3p}=\Theta(1)$ and ${\color{blue}\alpha_{\rm C}}=1/{\sqrt{2p(1-p)}}=\Theta(p^{-1/2}).$
Thus ${\color{blue}\alpha_{\rm C}}/{\color{blue}\alpha_0}\to\infty$ as $p\to0^+$.
Second, let $X\sim\Ber(p)$, $A \sim \mathrm{Unif}(\{-1,+1\})$, and set $\vphi(X,A)=(X,A)\in\sR^2$.
Here $\mSigma=\operatorname{diag}(p,1)$, so ${\color{blue}\alpha_0}=\sqrt{(1+p)/(2p)}=\Theta(p^{-1/2})$.
Centering gives $\overline\vphi=(0,A)$, so $\vz_{\rm C}=A$ and ${\color{blue}\alpha_{\rm C}}=\sqrt{2}$.
Therefore ${\color{blue}\alpha_0}/{\color{blue}\alpha_{\rm C}}\to\infty$ as $p\to0^+$.

\paragraph{Synthetic experiment.}
We consider linear contextual bandits with five-dimensional contexts and three-dimensional cube actions, indexed by a sampling probability $\delta\in[0,1/2]$ (distinct from the confidence level above) and an action-independent nuisance magnitude $\beta\in[0,2]$.
The true value depends only on the first action coordinate. At each context, the behavior policy favors one sign of this coordinate and samples the opposite with probability $\delta$.
When $\beta=0$ and $\delta$ is small, \VR can exploit absolute rewards while \VDR has few comparisons between actions with different true values. Increasing $\beta$ misspecifies \VR through an action-independent shift that cancels from \VDR.

We fit both methods with $(N_x,N_a)=(200,8)$ and evaluate greedy-policy regret over $8$ independent instances; details and ablations are deferred to \Cref{app:controlled-synthetic-experiments}.
The left panel of \Cref{fig:experiments} shows that when $\beta=0$, \VR has lower regret in every instance for every tested $\delta>0$.
At fixed $\delta=0.02$, $\Reg_{\VR}-\Reg_{\VDR}$ changes from negative to positive as $\beta$ grows.
These results illustrate the bias--variance tradeoff: \VR benefits from absolute rewards under correct specification when informative within-context comparisons are rare, while \VDR becomes preferable as nuisance-induced bias grows.

\begin{figure}[t]
  \centering
  \begin{minipage}[c]{0.32\linewidth}
    \centering
    \includegraphics[width=\linewidth]
    {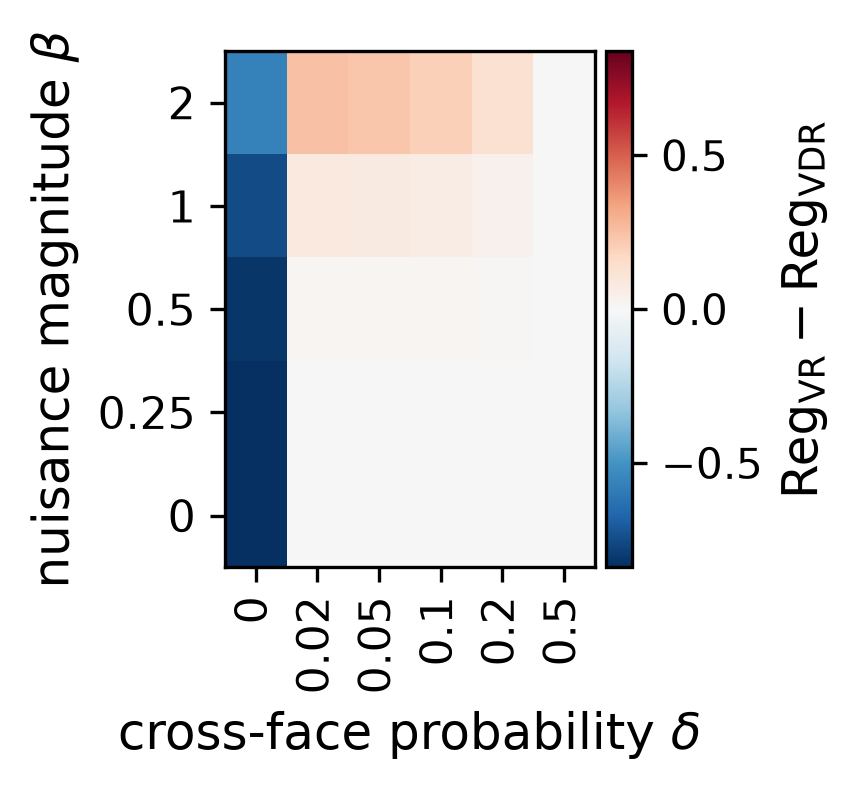}
  \end{minipage}
  \hfill
  \begin{minipage}[c]{0.64\linewidth}
    \centering
    \includegraphics[
      width=\linewidth,
      trim=0 0 0 0,
      clip
    ]{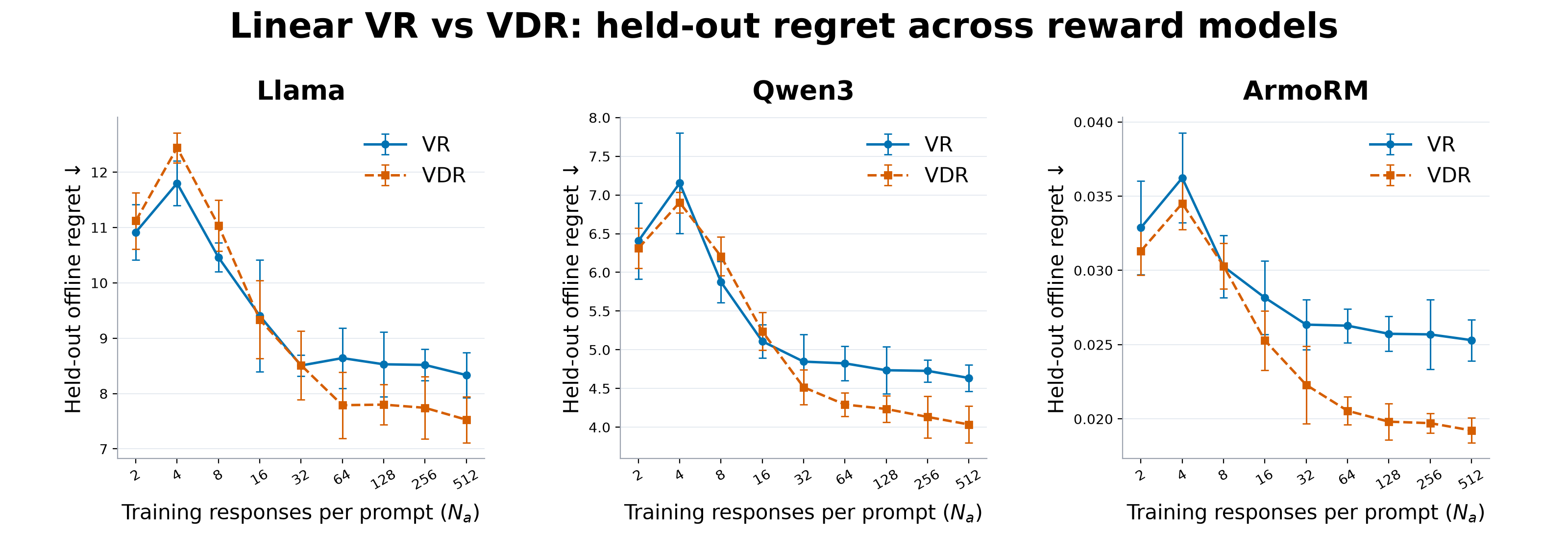}
  \end{minipage}
  \caption{
  \textbf{Left:} synthetic regret difference $\Reg_{\VR}-\Reg_{\VDR}$ as the cross-face probability $\delta$ and nuisance magnitude $\beta$ vary; negative values favor \VR and positive values favor \VDR.
  \textbf{Right:} real-data regret versus the number of labeled responses per training prompt; lower is better. Error bars are pointwise 95\% Student-$t$ intervals over five seeds.
  }
  \label{fig:experiments}
\end{figure}

\paragraph{LLM experiment.}
We compare $\ell_2$-regularized linear \VR and \VDR scorers on frozen LLM representations in a response-selection task.  We collect 1,000 prompts from WildChat, UltraFeedback, GSM8K, MATH, MBPP, HelpSteer2, and TL;DR, and cache 1,000 Llama-3.2-3B-Instruct responses per prompt.  Each response is represented by its 3,072-dimensional mean-pooled final-layer hidden state.  Keeping the prompts, responses, and representations fixed, we repeat the experiment with rewards from Skywork-Reward-V2-Llama-3.1-8B, Skywork-Reward-V2-Qwen3-8B, and ArmoRM-Llama3-8B-v0.1.
Across five seeds, both methods use the same 800 training prompts and sampled responses, are evaluated on 200 held-out prompts, and vary the number of labeled responses per training prompt from $N_a=2$ to $512$; full details are deferred to \Cref{app:controlled-real-experiments}.

The right panel of \Cref{fig:experiments} shows that for $N_a \geq 64$, \VDR has significantly lower mean regret for all three reward models (see paired intervals in \Cref{fig:empirical-paired-three-rewards}, \Cref{app:controlled-real-experiments}).
At $N_a=512$, the reductions are $9.7\%$, $13.0\%$, and $24.0\%$ for Llama, Qwen3, and ArmoRM, respectively.
Thus the preferred objective changes with the amount of within-prompt supervision.
Although outside our theoretical scope, we also report a controlled \MCTS math reasoning experiment in \Cref{app:math-reasoning-experiment}, where a binary-preference variant of \VDR achieves higher observed final-answer accuracy than \VR with both linear and MLP heads.

%% file: 005Conclusion.tex
\section{Conclusion and Future Work}
\label{sec:conclusion}
We establish finite-sample prediction guarantees for \ValueRegression{} (\VR) and \ValueDifferenceRegression{} (\VDR) with grouped offline data (which imply offline-regret bounds under coverage). Our analysis distinguishes two sampling scales: $N_x$ governs generalization across contexts, while $N_a$ controls the averaging of within-context variability. For finite classes, \VDR avoids the approximation-error term in the \VR bound and improves a reward-scale term through repeated action observations.
For linear classes, the comparison depends on feature geometry: differencing removes nuisance-induced bias, while fitting absolute rewards can reduce estimation variance under correct specification.
Our experiments illustrate the resulting bias--variance tradeoff, with neither method uniformly preferred.

We conclude with some potential future directions.
One is to assess whether the weak realizability assumption holds in LLM applications and other domains, and to identify weaker assumptions under which meaningful guarantees remain possible.
Other directions include determining whether analogous tradeoffs arise under the Bradley--Terry model with logistic losses, understanding how pair selection affects \VDR, and extending the analysis beyond pairwise comparisons (e.g., ranking).

%% file: 900RelatedWorks.tex
\section{\texorpdfstring{Related Work}{Related Work}}
\label{app:related-work}

We review related work on reward modeling objectives, regression under misspecification, statistical guarantees for pointwise/pairwise regression, and decision-making in contextual bandits.

\paragraph{(1) Reward modeling objectives.} 
Reward estimators can be trained to predict absolute numerical values, numerical differences between alternatives~\citep{wetzel2022twin,xie2022automated}, or preferences.
These objectives impose different requirements: pointwise \ValueRegression{} (\VR) fits individual reward labels, \ValueDifferenceRegression{} (\VDR) fits their numerical gaps, and preference-ranking objectives fit comparison outcomes without directly matching observed gap magnitudes. 
Numerical-gap objectives have been studied in magnitude-preserving ranking and conditional ranking~\citep{cortes2007magnitude,pahikkala2013efficient}.
Recent works have also demonstrated that combining pointwise regression with pairwise losses can improve predictive performance~\citep{zhu2024gradient} and value-guided reasoning~\citep{chen2024step}.
Preference-based reward modeling instead commonly fits comparison probabilities using the Bradley--Terry model~\citep{bradleyterry1952,christiano2017deep,zhu2023rlhf}, with related comparison objectives used in LLM reward models~\citep{zhang2026bradley,guan2025rstar}.
These findings show that there is an intricate distinction between absolute prediction, gap prediction, and downstream decision quality.
Our work studies \VR and \VDR constructed from the same numerical observations, separating their statistical behavior under a shared, \emph{grouped} sampling model.

\paragraph{(2) Semiparametric estimation.}
Our weak realizability assumption has the semiparametric form $r^\star(x,a)=f^\star(x,a)+b^\star(x)$, where $f^\star\in\gF$ captures action-dependent variation and $b^\star$ is an unknown action-independent nuisance function; neither $b^\star$ nor the complete reward function must belong to $\gF$.
For linear $\gF$, our assumption has the structure of a classical partially linear model~\citep{speckman1988kernel}, with a finite-dimensional action-dependent component and a nonparametric context-dependent nuisance.
Paralleling the partialling-out construction of \citet{robinson1988root}, our \VDR eliminate the nuisance $b^\star$ and target the within-context differences of $f^\star$.
These differences identify its centered component on the behavior-policy support, while action-invariant parameter directions remain unidentified.
The related principle of removing shared variation also appears in partial regression and panel-data within transformations~\citep{frisch1933partial,mundlak1978pooling}.
Modern approaches use orthogonal scores and nuisance estimation to reduce sensitivity to nuisance-estimation errors~\citep{chernozhukov2018double,foster2023orthogonal}; heterogeneous-treatment-effect estimation similarly isolates treatment contrasts through residualized objectives~\citep{nie2021quasi}.
In sequential decision-making, action-centered and semiparametric contextual bandits allow complex action-independent baselines while modeling action effects~\citep{greenewald2017action,krishnamurthy2018semiparametric,kim2019contextual,choi2023semiparametric}.
\citet{carranza2023flexible} develop contextual bandit algorithms based on heterogeneous-treatment-effect oracles, while \citet{kim2025semiparametric} use experimental design and orthogonalized regression for semiparametric bandits.
We share the motivation of removing decision-irrelevant nuisance variation, but our setting instead compares pointwise regression with regression on within-context reward differences, examining how the regression objective and repeated action observation affect estimation and decision guarantees.

\paragraph{(3) Misspecified linear regression.}
Random-design linear regression under misspecification has been studied under different prediction and distribution-shift criteria.
For out-of-sample $L_2$ prediction, \citet{hsu2012random,hsu2014random} analyze ordinary least squares and ridge regression, quantifying the effects of approximation residuals and empirical-covariance errors.
Our \VR analysis (\Cref{thm:vr-linear}) belongs to this line: when $N_a=1$, its bound has the same first-order structure, while for $N_a\geq2$ it further separates residual fluctuations across contexts from those averaged over all observations.
Under adversarial covariate shift, \citet{amortila2024mitigating} show that standard ERM can amplify $L_\infty$ misspecification through distribution mismatch and develop a disagreement-based regression procedure that avoids this amplification.
\citet{maran2025misspecification} instead study uniform prediction for misspecified random-design linear models, identifying the Lebesgue constant of the population $L_2$ projection as the sharp amplification factor for uniform approximation error in $L_\infty$.
Related work characterizes optimal approximation factors for misspecified linear off-policy value-function estimation~\citep{amortila2023optimal}.
For downstream decisions, \citet{lattimore2020learning} quantify the effect of uniform linear approximation error on bandit and RL guarantees, while \citet{krishnamurthy2021adapting} develop contextual bandit algorithms whose regret adapts to average reward-model misspecification using offline regression oracles.
Our \VR analysis extends the prediction-error perspective to repeated observations within contexts.
The \VDR comparison concerns a narrower form of misspecification: action-independent residuals can be removed before regression by differencing, whereas action-dependent residuals generally remain.
Related coverage-dependent approximation errors and decision lower bounds are discussed in Appendix~\ref{app:agnostic}.

\paragraph{(4) Statistical guarantees with grouped observations.}
Classical pairwise-learning theory addresses the dependence caused by reusing observations across pairs, using U-statistic and stability analyses \citep{clemenccon2008ranking,lei2020sharper}. 
Our observations have an additional grouped structure: $N_a$ action--reward observations share each of $N_x$ contexts, and within-context differencing cancels the common nuisance. 
Treating each context block as an independent observation permits block-level analysis, but exploiting repeated action samples requires tracking within-block variation, which we do throughout our proofs.

Our analysis also connects to localization methods for squared-loss learning.
Local Rademacher complexities characterize estimation rates through critical radii~\citep{bousquet2002local,bartlett2005local,koltchinskii2011}, while offset Rademacher complexity uses a negative quadratic term to control stochastic fluctuations~\citep{liang15learning}.
The separation of quadratic and multiplier processes is also central to small-ball analyses of least squares~\citep{mendelson2014without}.
Our localized block-offset modulus follows these principles but is defined directly through the unsymmetrized excess-loss process.
Evaluating it with the grouped sampling structure preserves the distinction between fluctuations across contexts and within contexts.

\paragraph{(5) Offline contextual bandits and RL.} 
Offline contextual bandits study policy selection and learning from previously collected action--reward observations.
Prior work addresses logging-policy mismatch through propensity-weighted objectives and variance control~\citep{swaminathan2015counterfactual}, and establishes policy guarantees through pessimism under limited coverage~\citep{rashidinejad2022offline}.
In batch RL, \citet{chen2019information} study the roles of distribution coverage and function approximation in sample-efficient learning.
Related analyses of approximate policy and value iteration connect value-function approximation errors to policy performance~\citep{munos2003error,farahmand2010error}.
More recently, \citet{ryu2025offline} develop betting-based confidence bounds for off-policy selection and a freezing approach to off-policy learning, obtaining guarantees that adapt to second-order information.
Related policy-optimization methods also use comparison objectives: DPO applies a preference loss to policy log-ratios~\citep{rafailov2023direct}, whereas REBEL fits differences of policy log-ratios onto numerical reward differences~\citep{gao2024rebel}.
Our work studies a complementary question: how pointwise versus pairwise value regression affects regret under structured misspecification and repeated action observations.

%% file: 901Regret.tex
\section{\texorpdfstring{Estimation Error to Offline Regret}{Estimation Error to Offline Regret}}
\label{app:offline-regret}

\subsection{\texorpdfstring{Coverage Coefficients and Regret Guarantees}{Coverage Coefficients and Regret Guarantees}}

In this appendix, we present two offline RL algorithms along with their offline regret guarantees.
To do so, we first define two notions of \textbf{coverage coefficients}, as commonly done in offline RL literature~\citep{chen2019information,rashidinejad2022offline}.
First, the \emph{policy-specific coverage}: for each deterministic $\pi : \gX \rightarrow \gA$, 
\begin{equation}
    C(\pi) \coloneq \E_{x \sim \rho}\left[ \frac{1}{\piref(\pi(x) \mid x)} \right], \quad
    C^\star \coloneq C(\pi^\star),
\end{equation}
where $\pi^\star$ is a fixed deterministic optimal policy, $\pi^\star(x)\in\argmax_{a\in\gA}r^\star(x,a)$, and we take the convention that $1/0 = \infty$.
Next, the \emph{($\gF$-restricted) all-policy coverage}:
\begin{equation}
    C^\infty_\gF \coloneq \E_{x \sim \rho}\left[ \max_{a \in \gA_\gF(x)} \frac{1}{\piref(a \mid x)} \right], \quad \gA_\gF(x) \coloneq \left\{ a \in \gA : \exists f \in \gF \textit{ s.t. } f(x, a) \geq \max_{b \in \gA} f(x, b) \right\}.
\end{equation}
Note that such a restricted variant of the all-policy coefficient in an $\gF$-dependent manner has been considered before, e.g., \citet[Appendix G]{chen2019information}.
For every policy satisfying $\pi(x)\in\gA_\gF(x)$, one has
$C(\pi)\leq C^\infty_\gF$.  There is no general ordering between
$C(\pi)$ and $C^\star=C(\pi^\star)$ for arbitrary $\pi$; we can nevertheless
construct instances in which $C^\star \ll C^\infty_\gF$ (e.g., Remark~\ref{rem:gmin-coverage}).

We now present offline regret upper bounds for the \Greedy and \Pessimism policies:
\begin{proposition}
\label{prop:regret}
    Let $\hat{f}_N \in \gF$ be our estimator, and let $\gC : g \mapsto \gC g$ be the centering operator $\gC g(x,a) \coloneq g(x,a)-\E_{a'\sim\piref(\cdot\mid x)}[g(x,a')]$, which preserves the action gap.
    Suppose that \Cref{asm:weak_realizability} holds and that {\color{blue} $\E_{x,a}[(\gC \hat{f}_N - \gC r^\star)^2]\leq\gamma_N^2$}.
    Then,
    \begin{itemize}
        \item The
        \Greedy policy satisfies
        $$\Reg(\hat{\pi}^{\Greedy}_N)\leq 2 \sqrt{C^\infty_\gF-1} {\color{blue}\gamma_N}, \quad \hat{\pi}_N^{\Greedy}(x) \coloneq \argmax_{a \in \gA} \hat{f}_N(x, a);$$
        \item The \Pessimism policy satisfies
        \[
        \begin{aligned}
            \Reg(\hat{\pi}^{\Pessimism}_N)
            &\leq 2 \sqrt{C^\star-1} {\color{blue}\gamma_N}, \\
            \hat{\pi}_N^{\Pessimism}
            &\in \argmax_{\pi : \gX \rightarrow \gA}
            \left\{
                \E_{x \sim \rho}\left[\hat{f}_N(x, \pi(x))\right]
                - {\color{blue}\gamma_N} \sqrt{C(\pi)-1}
            \right\}.
        \end{aligned}
        \]
    \end{itemize}
\end{proposition}
\begin{proof}[Proof Sketch]
    We utilize the conditional mean-zero property of the centering operator and the Cauchy--Schwarz inequality.
    Although this is standard in offline RL~\citep{rashidinejad2022offline}, we provide the full proof in \Cref{app:regret}.
\end{proof}

\paragraph{Implementation details.}
Despite the benefit of improving the coverage from $C^\infty_\gF$ to $C^\star$, \Pessimism introduces some implementation difficulties.
\Pessimism requires computing $\E_{x \sim \rho}$ and knowing $\gamma_N$ and $C(\pi)$ for each $\pi$.
The former can be approximated to an arbitrary precision via sampling from $\rho$.
The latter requires knowledge of $\piref$ and ${\color{blue}\gamma_N}$, which may be problematic.

For instance, consider the error bound of \VR for finite class (\Cref{thm:finite}), where ${\color{blue}\gamma_N^2} = {\color{red}2\varepsilon_{\rm aprx}}
        + 
        \left(\frac{22F^2 + 4 F {\color{red}\Rmax}}{N_x}+\frac{18\sigma^2}{N}\right)
        \log\!\frac{|\gF|}{\delta}.$
In this case, the learner needs a useful \emph{a priori} upper bound on ${\color{red}\varepsilon_{\rm aprx}}$ in addition to knowing $\piref$.
A worst-case bound of ${\color{red}\varepsilon_{\rm aprx}} \leq 2F^2 + 2{\color{red}\Rmax^2}$ may be utilized, but it can be too crude especially under small misspecification.

\subsection{\texorpdfstring{Proof of \Cref{prop:regret}: Offline Regret of \Greedy and \Pessimism}{Proof of Proposition B.1: Offline Regret of Greedy and Pessimism}}
\label{app:regret}

Recall that $r^\star(x,a) = f^\star(x,a)+b^\star(x)$ by \Cref{asm:weak_realizability}, so $f^\star(x,a) - f^\star(x,a') = r^\star(x,a) - r^\star(x, a')$ for any context $x \in \gX$.
Thus, $r^\star$ and $f^\star$ share the optimal policy $\pi^\star(x) \in \argmax_{a \in \gA} r^\star(x,a) = \argmax_{a \in \gA} f^\star(x,a)$, and $\pi^\star(x) \in \gA_\gF(x)$.
For notational simplicity, we denote $\Delta(x, a) \coloneq \gC \hat{f}_N(x, a) - \gC r^\star(x, a)$.
By the definition of $\gC$, for every $x\in\gX$, $\E_{a\sim\piref(\cdot\mid x)}[\Delta(x,a)\mid x]=0$ and
\begin{equation}
    \E_{x,a}[\Delta(x,a)^2]
    =
    \E_x\left[
        \Var_{a\sim\piref(\cdot\mid x)}
        \bigl(\hat f_N(x,a)-r^\star(x,a)\mid x\bigr)
    \right]
    \leq\gamma_N^2.
    \label{eq:regret-centering}
\end{equation}
If the relevant coverage coefficient is infinite, the corresponding bound is vacuous; hence, in what follows, we only use the inverse-propensity identities in the finite-coverage case.

\paragraph{Part 1: offline regret of the \Greedy policy.}

Since $\gC$ preserves the action gap, for any $f \in \gF$, $f$ and $\gC f$ have the same \Greedy policy $\hat{\pi}$ since $f(x,a) - f(x,a') = \gC f(x,a)-\gC f(x,a')$, and $\hat{\pi}(x) \in \gA_\gF(x)$.
Let $\hat{\pi}(x) \coloneq \hat{\pi}_N^{\Greedy}(x) = \argmax_{a \in \gA} \hat{f}_N(x, a)$, then $\hat{\pi}(x) \in \gA_\gF(x)$.
For any context $x \in \gX$, the suboptimality gap can be bounded by the maximum estimation error across all actions in $\gA_\gF(x)$ as:
\begin{align}
    &r^\star(x, \pi^\star(x)) - r^\star(x, \hat{\pi}(x)) \\
    & = \gC r^\star(x, \pi^\star(x)) - \gC r^\star(x, \hat{\pi}(x))\\
    &= \left( \gC r^\star(x, \pi^\star(x)) - \gC \hat{f}_N(x, \pi^\star(x)) \right) + \left( \gC \hat{f}_N(x, \pi^\star(x)) - \gC \hat{f}_N(x, \hat{\pi}(x)) \right) \\
    &\quad + \left( \gC \hat{f}_N(x, \hat{\pi}(x)) - \gC r^\star(x, \hat{\pi}(x)) \right) \\
    &\leq \left( \gC r^\star(x, \pi^\star(x)) - \gC \hat{f}_N(x, \pi^\star(x)) \right) + \left( \gC \hat{f}_N(x, \hat{\pi}(x)) - \gC r^\star(x, \hat{\pi}(x)) \right) \\
    &\leq 2 \max_{a \in \gA_\gF(x)} |\Delta(x, a)|,
\end{align}
where the first inequality holds since $\hat{\pi}$ is the \Greedy policy of $\gC \hat{f}_N$ as well as of $\hat{f}_N$, and the last inequality holds since $\pi^\star(x), \hat{\pi}(x) \in \gA_\gF(x)$.
By Eqn.~\eqref{eq:regret-centering}, for every $a\in\gA_\gF(x)$,
\begin{align}
    |\Delta(x,a)|
    &=
    \left|
        \E_{a'\sim\piref(\cdot\mid x)}
        \left[
            \left(
                \frac{\indicator\{a'=a\}}{\piref(a\mid x)}-1
            \right)
            \Delta(x,a')
            \,\middle|\,x
        \right]
    \right|
    \notag\\
    &\leq
    \sqrt{
        \frac{1}{\piref(a\mid x)}-1
    }
    \sqrt{
        \E_{a'\sim\piref(\cdot\mid x)}
        [\Delta(x,a')^2\mid x]
    }.
\end{align}
Taking the expectation over $x \sim \rho$ and applying the Cauchy--Schwarz inequality ($\E[Y Z] \leq \sqrt{\E[Y^2]\E[Z^2]}$), we obtain:
\begin{align}
    \Reg(\hat{\pi}_N^{\Greedy}) & = \E_{x \sim \rho} \left[ r^\star\big( x, \pi^\star(x) \big) - r^\star \big( x, \hat{\pi} (x) \big)\right]\\
    &\leq 2 \E_{x \sim \rho}\left[ \sqrt{ \max_{a \in \gA_\gF(x)} \left\{\frac{1}{\piref(a \mid x)}-1\right\} } \sqrt{ \E_{a \sim \piref}[\Delta(x, a)^2 \mid x] } \right] \\
    &\leq 2 \sqrt{ \E_{x \sim \rho}\left[ \max_{a \in \gA_\gF(x)} \left\{\frac{1}{\piref(a \mid x)}-1\right\} \right] } \sqrt{ \E_{x \sim \rho, a \sim \piref}[\Delta(x, a)^2] } \\
    &\leq 2 \sqrt{C^\infty_\gF-1} \gamma_N.
\end{align}

\vspace{1em}
\paragraph{Part 2: offline regret of the \Pessimism policy.}

First, for any deterministic policy $\pi : \gX \rightarrow \gA$, we bound the policy evaluation error. Let $p_\pi(x)\coloneq\piref(\pi(x)\mid x)$. By Eqn.~\eqref{eq:regret-centering} and the Cauchy--Schwarz inequality, we have
\begin{align}
    \left| \E_{x \sim \rho}\left[ \Delta(x, \pi(x)) \right] \right|
    &=
    \left|
        \E_{x,a}
        \left[
            \left(
                \frac{\indicator\{a=\pi(x)\}}{p_\pi(x)}-1
            \right)
            \Delta(x,a)
        \right]
    \right|
    \notag\\
    &\leq
    \sqrt{
        \E_{x,a}
        \left[
            \left(
                \frac{\indicator\{a=\pi(x)\}}{p_\pi(x)}-1
            \right)^2
        \right]
    }
    \sqrt{\E_{x,a}[\Delta(x,a)^2]}
    \notag\\
    &\leq
    \gamma_N\sqrt{C(\pi)-1}.
\end{align}
Thus, for any policy $\pi$, the evaluation error is bounded by:
\begin{equation}
    \label{eq:eval_bound}
    \left| \E_{x \sim \rho} \left[ \gC \hat{f}_N(x, \pi(x)) \right] - \E_{x\sim\rho} \left[ \gC r^\star \big(x, \pi(x) \big) \right] \right| = \left| \E_{x \sim \rho} \left[ \Delta \big(x , \pi(x) \big) \right] \right| \leq \gamma_N \sqrt{C(\pi)-1}.
\end{equation}
Recall that the pessimistic objective is $\E_{x \sim \rho}[\hat{f}_N(x, \pi(x))] - \gamma_N \sqrt{C(\pi)-1}$. Since $\gC$ preserves the action-gap for the fixed context $x$, we have
\begin{align}
    \hat\pi_N^{\Pessimism} &\in \argmax_{\pi:\gX\to\gA} \left\{ \E_{x\sim\rho} \left[ \hat f_N(x,\pi(x)) \right] -\gamma_N\sqrt{C(\pi)-1} \right\} \\
    &= \argmax_{\pi:\gX\to\gA} \left\{ \E_{x\sim\rho} \left[ \gC\hat f_N(x,\pi(x)) \right] -\gamma_N\sqrt{C(\pi)-1} \right\}.
\end{align}

Thus, by the optimality of \(\hat\pi_N^{\Pessimism}\) w.r.t. this objective, it holds
\begin{align}\label{eqn:pessimitic-optimal}
    &\E_{x\sim\rho} \left[ \gC\hat f_N(x,\hat\pi_N^{\Pessimism}(x)) \right] - \gamma_N\sqrt{C(\hat\pi_N^{\Pessimism})-1} \ge \E_{x\sim\rho} \left[ \gC\hat f_N(x,\pi^\star(x)) \right] - \gamma_N\sqrt{C(\pi^\star)-1}.
\end{align}
Together with the bound implied by Eqn.~\eqref{eq:eval_bound},
\begin{align}
    &\Reg (\hat{\pi}_N^{\Pessimism}) \\
    & = V(\pi^\star) - V (\hat{\pi}_N^{\Pessimism})\\
    & = \E_{x\sim\rho}\left[ r^\star \big(x, \pi^\star(x) \big) - r^\star \big(x, \hat{\pi}_N^{\Pessimism}(x) \big) \right]\\
    & = \E_{x\sim\rho}\left[ \gC r^\star \big(x, \pi^\star(x) \big) - \gC r^\star \big(x, \hat{\pi}_N^{\Pessimism}(x) \big) \right]\\
    & \leq \E_{x\sim\rho}\left[ \gC r^\star \big(x, \pi^\star(x) \big) \right] - \left( \E_{x \sim \rho} \left[ \gC \hat{f}_N \big( x ,\hat{\pi}_N^{\Pessimism}(x) \big) \right] - \gamma_N \sqrt{C(\hat{\pi}_N^{\Pessimism})-1} \right) \tag{Eqn.~\eqref{eq:eval_bound}}\\
    & \leq \E_{x\sim\rho}\left[ \gC r^\star \big(x, \pi^\star(x) \big) \right] - \left( \E_{x \sim \rho} \left[ \gC \hat{f}_N \big( x ,\pi^\star(x) \big) \right] - \gamma_N \sqrt{C(\pi^\star)-1} \right) \tag{Eqn.~\eqref{eqn:pessimitic-optimal}}\\
    & \leq \left|  \E_{x \sim \rho} \left[ \gC \hat{f}_N \big( x ,\pi^\star(x) \big) \right] - \E_{x\sim\rho}\left[ \gC r^\star \big(x, \pi^\star(x) \big) \right] \right| + \gamma_N \sqrt{C(\pi^\star)-1}\\
    & \leq \gamma_N \sqrt{C(\pi^\star)-1} + \gamma_N \sqrt{C(\pi^\star)-1} = 2 \sqrt{C^\star-1} \gamma_N. \tag{Eqn.~\eqref{eq:eval_bound}}
\end{align}
which completes the proof for both policies.
\qed

%% file: 902Regression.tex
\section{\texorpdfstring{Auxiliary Results for Grouped Regression}{Auxiliary Results for Grouped Regression}}
\label{app:grouped-regression}

For a comparator $f^\circ\in\gF$ and a perturbation $h\in\gF-f^\circ$, define
the population and empirical excess losses by
\begin{equation}
    \Delta(h;f^\circ)
    \coloneq
    \gL(f^\circ+h)-\gL(f^\circ),
    \qquad
    \widehat\Delta(h;f^\circ)
    \coloneq
    \widehat\gL(f^\circ+h)-\widehat\gL(f^\circ).
\end{equation}

\subsection{\texorpdfstring{\VDR Population Loss and Centered Prediction Error}{VDR Population Loss and Centered Prediction Error}}
\label{app:vdr-centering}

Recall that
$\gC f(x,a)=f(x,a)-\E_{a'\sim\piref(\cdot\mid x)}[f(x,a')]$.
The following identity will be used in both the finite- and linear-class \VDR
proofs.

\begin{lemma}[\VDR population-loss identity]
\label{lem:vdr-centering}
For every $f:\gX\times\gA\to\sR$,
\begin{equation}
    \gL^{\VDR}(f)
    =
    2\E_{x,a}\left[
        \bigl(\gC f(x,a)-\gC r^\star(x,a)\bigr)^2
    \right].
    \label{eq:vdr-population-loss}
\end{equation}
Under \Cref{asm:weak_realizability}, $\gC f^\star=\gC r^\star$ and
$\gL^{\VDR}(f^\star)=0$. In particular,
\begin{equation}
    \E_{x,a}\left[
        \bigl(\gC\hat f_N^{\VDR}(x,a)-\gC r^\star(x,a)\bigr)^2
    \right]
    =
    \frac12\left(
        \gL^{\VDR}(\hat f_N^{\VDR})-\gL^{\VDR}(f^\star)
    \right)
    =
    \frac12\gL^{\VDR}(\hat f_N^{\VDR}).
    \label{eqn:VDR-finite-oracle}
\end{equation}
\end{lemma}

\begin{proof}
Let $A,A' \sim \piref(\cdot\mid X)$ independently and set $g=f-r^\star$. Conditional on $X$,
\begin{align}
    \E\left[
        \bigl(g(X,A)-g(X,A')\bigr)^2
        \,\middle|\,X
    \right]
    &=
    2\E\left[
        \bigl(g(X,A)-\E[g(X,A)\mid X]\bigr)^2
        \,\middle|\,X
    \right] \\
    &=
    2\E\left[(\gC g(X,A))^2\mid X\right].
\end{align}
Taking expectation over $X$ proves Eqn.~\eqref{eq:vdr-population-loss}.
Under \Cref{asm:weak_realizability},
$f^\star(x,a)-r^\star(x,a)=-b^\star(x)$ is action-independent, which gives
$\gC f^\star=\gC r^\star$ and the remaining claims.
\end{proof}

\subsection{\texorpdfstring{Proof of \Cref{prop:localised}: Localized Fixed-Point Inequality}{Proof of Proposition 2.1: Localized Fixed-Point Inequality}}
\label{app:proof-prop-localised}

Fix the comparator $f^\circ\in\gF$ and denote
\begin{equation}
    h \coloneq \hat f_N-f^\circ,
    \qquad
    \hat r \coloneq
    \Delta( h;f^\circ)
    =
    \gL(\hat f_N)-\gL(f^\circ).
\end{equation}
By condition $(i)$ and the definition of the empirical excess loss,
\begin{equation}
    \widehat{\Delta}(h;f^\circ)
    =
    \widehat{\gL}(\hat f_N)
    -
    \widehat{\gL}(f^\circ)
    \leq
    -\lambda\widehat Q( h).
\end{equation}
Therefore,
\begin{align}
    \hat r
    &=
    \left[
        \Delta(h;f^\circ)
        -
        \widehat{\Delta}(h;f^\circ)
        -
        \lambda\widehat Q(h)
    \right]
    +
    \left[
        \widehat{\Delta}(h;f^\circ)
        +
        \lambda\widehat Q(h)
    \right]
    \notag\\
    &\leq
    \Delta(h;f^\circ)
    -
    \widehat{\Delta}(h;f^\circ)
    -
    \lambda\widehat Q(h).
    \label{eq:localised-basic-inequality}
\end{align}

Suppose first that $\hat r>0$.
Since $h\in\gF-f^\circ$ and
$\Delta(h;f^\circ)=\hat r$, the function $h$ belongs to
the set over which
$\Psi(\hat r;\lambda,f^\circ)$ takes its supremum.
Hence, by \eqref{eq:localised-basic-inequality},
\begin{equation}
    \hat r
    \leq
    \Psi(\hat r;\lambda,f^\circ)
    \leq
    \psi(\hat r),
\end{equation}
where the second inequality follows from condition $(ii)$.
Therefore,
\begin{equation}
    \hat r
    \in
    \{r\geq0:r\leq\psi(r)\} \Longrightarrow
    \gL(\hat f_N)-\gL(f^\circ)
    =
    \hat r
    \leq
    \sup\{r\geq0:r\leq\psi(r)\}.
\end{equation}

It remains to consider $\hat r\leq0$.
Since $0\in\gF-f^\circ$ and
\begin{equation}
    \Delta(0;f^\circ)
    =
    \widehat{\Delta}(0;f^\circ)
    =
    \widehat Q(0)
    =
    0,
\end{equation}
we have $\Psi(0;\lambda,f^\circ)\geq0.$
Condition $(ii)$ therefore implies $\psi(0)\geq0$, so that
\begin{equation}
    0\in\{r\geq0:r\leq\psi(r)\}.
\end{equation}
Consequently,
\begin{equation}
    \gL(\hat f_N)-\gL(f^\circ)
    =
    \hat r
    \leq
    0
    \leq
    \sup\{r\geq0:r\leq\psi(r)\}.
\end{equation}
This completes the proof.
\qed

%% file: 903Finite.tex
\section{\texorpdfstring{Deferred Proofs from \Cref{sec:finite}: Finite Class}{Deferred Proofs from Section 3: Finite Class}}
\label{app:finite}

\subsection{\texorpdfstring{\ValueRegression{} (\VR)}{Value Regression (VR)}}

\subsubsection{\texorpdfstring{Proof of \Cref{lem:finite-vr-modulus}: Localized Modulus for \VR in Finite Class}{Proof of Lemma 3.1: Localized Modulus for VR in Finite Class}}
\label{app:proof-lem-finite-vr-modulus}

Recall that
\begin{equation}
    f^\circ
    \coloneq
    f^\circ_{\VR}
    \in
    \argmin_{g\in\gF}
    \E_{x,a}\left[(g-r^\star)^2\right],
    \qquad
    \E_{x,a}\left[(f^\circ-r^\star)^2\right]
    =
    \varepsilon_{\rm aprx}.
\end{equation}
Let
\begin{equation}
    t_\delta
    \coloneq
    \log\frac{|\gF|}{\delta},
    \qquad
    v_\delta
    \coloneq
    \left(
        \frac{F^2}{N_x}
        +
        \frac{\sigma^2}{N}
    \right)t_\delta.
\end{equation}
For any $f\in\gF$, let
\begin{equation}
    h\coloneq f-f^\circ,
    \qquad
    \Delta_f
    \coloneq
    \Delta(h;f^\circ)
    =
    \E_{x,a}\left[(f-r^\star)^2\right]
    -
    \varepsilon_{\rm aprx}.
\end{equation}
Since $f^\circ$ is a population minimizer, $\Delta_f\geq0$.

Since $\lambda=0$, the localized block-offset modulus (\Cref{def:localised-modulus}) specialized
to \VR is simply
\begin{equation}
    \Psi(r;0,f^\circ)
    =
    \sup_{\substack{
        f\in\gF\\
        \Delta_f\leq r
    }}
    \left[
        \Delta_f
        -
        \left(
            \widehat{\gL}(f)-\widehat{\gL}(f^\circ)
        \right)
    \right].
    \label{eq:finite-vr-modulus-rewrite}
\end{equation}
It therefore suffices to uniformly control the deviation between the population and empirical excess losses over $f\in\gF$.

More precisely, we will prove that, with probability at least
$1-\delta$, simultaneously for all $f\in\gF$,
\begin{equation}
    \Delta_f
    -
    \left(
        \widehat{\gL}(f)-\widehat{\gL}(f^\circ)
    \right)
    \leq
    4\sqrt{
        v_\delta
        \left(
            \Delta_f+2\varepsilon_{\rm aprx}
        \right)
    }
    +
    \frac{8F(F+\Rmax)t_\delta}{3N_x}.
    \label{eq:finite-vr-uniform-deviation}
\end{equation}
Once this uniform deviation is established, the desired modulus bound
follows by restricting to $\Delta_f\leq r$.

\paragraph{Reduction to blockwise concentration.}

For each $f\in\gF$, define
\begin{equation}
    U_{i,j}(f)
    \coloneq
    \bigl(f(x_i,a_{i,j})-r_{i,j}\bigr)^2
    -
    \bigl(f^\circ(x_i,a_{i,j})-r_{i,j}\bigr)^2,
    \qquad
    U_i(f)
    \coloneq
    \frac1{N_a}\sum_{j=1}^{N_a}U_{i,j}(f).
\end{equation}
Then $\{U_i(f)\}_{i=1}^{N_x}$ are independent across $i\in[N_x]$, and
\begin{equation}
    \E[U_i(f)]
    =
    \Delta_f,
    \qquad
    \frac1{N_x}\sum_{i=1}^{N_x}U_i(f)
    =
    \widehat{\gL}(f)-\widehat{\gL}(f^\circ).
    \label{eq:finite-vr-U-mean}
\end{equation}

\paragraph{Range and two-level block variance.}
We next bound the range and variance of $U_i(f)$, the two inputs needed
for blockwise Bernstein concentration.
Let
\begin{equation}
    q_f(x,a)
    \coloneq
    (f(x,a)-r^\star(x,a))^2
    -
    (f^\circ(x,a)-r^\star(x,a))^2.
\end{equation}
Writing $h=f-f^\circ$ and using
$r_{i,j}=r^\star(x_i,a_{i,j})+\eta_{i,j}$ gives
\begin{equation}
    U_{i,j}(f)
    =
    q_f(x_i,a_{i,j})
    -
    2h(x_i,a_{i,j})\eta_{i,j}.
    \label{eq:finite-vr-U-decomposition}
\end{equation}
Since $\|f\|_\infty,\|f^\circ\|_\infty\leq F$ and
$|r_{i,j}|\leq\Rmax$, one can show that
\begin{equation}
    |U_{i,j}(f)|
    \leq
    4F(F+\Rmax).
    \label{eq:finite-vr-U-range}
\end{equation}

Moreover,
\begin{align}
    q_f(x,a)^2
    &=
    h(x,a)^2
    \left(
        (f(x,a)-r^\star(x,a))
        +
        (f^\circ(x,a)-r^\star(x,a))
    \right)^2
    \notag\\
    &\leq
    8F^2
    \left(
        (f(x,a)-r^\star(x,a))^2
        +
        (f^\circ(x,a)-r^\star(x,a))^2
    \right),
    \label{eq:finite-vr-q-bound}
\end{align}
while
\begin{equation}
    h(x,a)^2
    \leq
    2
    \left(
        (f(x,a)-r^\star(x,a))^2
        +
        (f^\circ(x,a)-r^\star(x,a))^2
    \right).
    \label{eq:finite-vr-h-bound}
\end{equation}
Since $\E[\eta_{i,j}\mid x_i,a_{i,j}]=0$ and
$\E[\eta_{i,j}^2\mid x_i,a_{i,j}]\leq\sigma^2$,
Eqns.~\eqref{eq:finite-vr-U-decomposition}--\eqref{eq:finite-vr-h-bound}
imply
\begin{align}
    \E[U_{i,j}(f)^2]
    &\leq
    8(F^2+\sigma^2)
    \E_{x,a}\left[
        (f-r^\star)^2+(f^\circ-r^\star)^2
    \right]
    \notag\\
    &=
    8(F^2+\sigma^2)
    \left(
        \Delta_f+2\varepsilon_{\rm aprx}
    \right).
    \label{eq:finite-vr-single-second-moment}
\end{align}

To retain the benefit of the $N_a$ repeated action samples,
define
\begin{equation}
    m_f(x)
    \coloneq
    \E_{a\sim\piref(\cdot\mid x)}
    [q_f(x,a)].
\end{equation}
By Jensen's inequality and Eqn.~\eqref{eq:finite-vr-q-bound},
\begin{equation}
    \Var(m_f(X))
    \leq
    8F^2
    \left(
        \Delta_f+2\varepsilon_{\rm aprx}
    \right).
    \label{eq:finite-vr-between-variance}
\end{equation}
Conditional on $x_i$, the variables
$\{U_{i,j}(f)\}_{j=1}^{N_a}$ are independent with common
conditional mean $m_f(x_i)$. Hence, by the law of total variance,
\begin{align}
    \Var(U_i(f))
    &=
    \left(1-\frac1{N_a}\right)
    \Var(m_f(X))
    +
    \frac1{N_a}\Var(U_{i,1}(f))
    \notag\\
    &\leq
    8
    \left(
        F^2+\frac{\sigma^2}{N_a}
    \right)
    \left(
        \Delta_f+2\varepsilon_{\rm aprx}
    \right).
    \label{eq:finite-vr-block-variance}
\end{align}

\paragraph{Uniform concentration and localization.}
Finally, Eqn.~\eqref{eq:finite-vr-U-range} implies
\begin{equation}
    \E[U_i(f)]-U_i(f)
    \leq
    8F(F+\Rmax)
    \qquad\text{almost surely}.
\end{equation}
Applying the one-sided Bernstein inequality to the independent
mean-zero variables
$\{\E[U_i(f)]-U_i(f)\}_{i=1}^{N_x}$ gives, for every fixed
$f\in\gF$, with probability at least $1-e^{-t}$,
\begin{align}
    \Delta_f
    -
    \frac1{N_x}\sum_{i=1}^{N_x}U_i(f)
    &\leq
    4\sqrt{
        \left(
            \frac{F^2}{N_x}
            +
            \frac{\sigma^2}{N_xN_a}
        \right)
        \left(
            \Delta_f+2\varepsilon_{\rm aprx}
        \right)t
    }
    +
    \frac{8F(F+\Rmax)t}{3N_x}.
    \label{eq:finite-vr-bernstein}
\end{align}
Taking $t=t_\delta$ and applying a union bound over $f\in\gF$, we obtain an event of
probability at least $1-\delta$ on which
Eqn.~\eqref{eq:finite-vr-bernstein} holds simultaneously for every
$f\in\gF$.

Since $N=N_xN_a$, Eqns.~\eqref{eq:finite-vr-U-mean} and \eqref{eq:finite-vr-bernstein} give precisely the uniform deviation in
Eqn.~\eqref{eq:finite-vr-uniform-deviation}.
Restricting Eqn.~\eqref{eq:finite-vr-uniform-deviation} to
$\Delta_f\leq r$ and taking the supremum in
Eqn.~\eqref{eq:finite-vr-modulus-rewrite} yields
\begin{equation}
    \Psi(r;0,f^\circ_{\VR})
    \leq
    4\sqrt{
        v_\delta
        \left(
            r+2\varepsilon_{\rm aprx}
        \right)
    }
    +
    \frac{8F(F+\Rmax)t_\delta}{3N_x},
    \qquad
    \forall r\geq0,
\end{equation}
which proves the claim.
\qed
\subsection{\texorpdfstring{\ValueDifferenceRegression{} (\VDR)}{Value Difference Regression (VDR)}}
\label{app:vdr}

\subsubsection{\texorpdfstring{Proof of the \VDR Bound in \Cref{thm:finite}}{Proof of the VDR Bound in Theorem 3.1}}
\label{app:proof-thm-finite-vdr}

We take $f^\circ_{\VDR}=f^\star$ and $\lambda=0$. The proof has three
steps: empirical risk minimization verifies condition~$(i)$ of
\Cref{prop:localised}, the modulus lemma below verifies condition~$(ii)$
and determines the fixed point, and Eqn.~\eqref{eqn:VDR-finite-oracle}
converts the resulting excess-loss bound into centered prediction error.

Since $\hat f_N^{\VDR}$ is an empirical risk minimizer, condition~$(i)$
holds.
For condition~$(ii)$, let us first define
\begin{equation}
    t_\delta
    \coloneq
    \log\frac{2|\gF|}{\delta},\quad
    \beta_\delta
    \coloneq
    \Biggl[
        \frac{F^2}{N_x}
        \left(
            \frac83
            +
            \frac{
                32\lfloor N_a^2/4\rfloor
            }{
                N_a(N_a-1)
            }
        \right)
        +
        \frac{
            16\sigma^2
        }{
            N_x(N_a-1)
        }
        +
        \frac{
            64F\Rmax
        }{
            3N
        }
    \Biggr]
    t_\delta.
\end{equation}
The following lemma, proved in
\Cref{app:proof-lem-finite-vdr-modulus}, controls the localized modulus.
\begin{lemma}
\label{lem:finite-vdr-modulus}
    With probability at least $1-\delta$,
    \begin{equation}
    \label{eq:finite-vdr-modulus}
        \Psi(r;0,f^\circ_{\VDR})
        \leq
        \frac34r+\frac12\beta_\delta,
        \qquad
        \forall r\geq0.
    \end{equation}
\end{lemma}

On the event in \Cref{lem:finite-vdr-modulus}, condition~$(ii)$ holds.
Therefore, \Cref{prop:localised} gives
\begin{equation}
    \gL^{\VDR}(\hat f_N^{\VDR})-\gL^{\VDR}(f^\star)
    \leq
    \sup
    \left\{
        r\geq0:
        r\leq
        \frac34r+\frac12\beta_\delta
    \right\}
    =
    2\beta_\delta.
\end{equation}
By \Cref{lem:vdr-centering},
\begin{equation}
    \E_{x,a}
    \left[
        \bigl(
            \gC\hat f_N^{\VDR}-\gC r^\star
        \bigr)^2
    \right]
    =
    \frac12\gL^{\VDR}(\hat f_N^{\VDR})
    \leq
    \beta_\delta.
\end{equation}
It remains only to compare $\beta_\delta$ with the displayed rate in
\Cref{thm:finite}. Since $N_a\geq2$ and $N=N_xN_a$,
\begin{equation}
    \frac{\lfloor N_a^2/4\rfloor}{N_a(N_a-1)}
    \leq
    \frac12,
    \qquad
    \frac1{N_x(N_a-1)}
    =
    \frac{N_a}{N(N_a-1)}
    \leq
    \frac2N.
\end{equation}
Consequently,
\begin{align}
    \beta_\delta
    &\leq
    \left(
        \frac{56F^2}{3N_x}
        +
        \frac{32\sigma^2}{N}
        +
        \frac{64F\Rmax}{3N}
    \right)
    \log\frac{2|\gF|}{\delta}
    \notag\\
    &=
    \frac83
    \left(
        \frac{7F^2}{N_x}
        +
        \frac{12\sigma^2+8F\Rmax}{N}
    \right)
    \log\frac{2|\gF|}{\delta},
\end{align}
which proves Eqn.~\eqref{eq:vdr-finite-rate} of \Cref{thm:finite}.
\qed

\subsubsection{\texorpdfstring{Proof of \Cref{lem:finite-vdr-modulus}: Localized Modulus for \VDR in Finite Class}{Proof of Lemma D.1: Localized Modulus for VDR in Finite Class}}
\label{app:proof-lem-finite-vdr-modulus}

Here, let $f^\circ \coloneq f^\circ_{\VDR} = f^\star$, and for any $f\in\gF$, write $g\coloneq f-f^\star.$

Define the population transformed squared error
\begin{equation}
    q_f
    \coloneq
    \Delta(g;f^\star)
    =
    \gL^{\VDR}(f)
    =
    2\E_{x,a}\left[(\gC g(x,a))^2\right]
    =
    2\E_x\left[
        \Var_{a\sim\piref(\cdot\mid x)}
        \bigl(g(x,a)\bigr)
    \right].
    \label{eq:vdr-finite-q}
\end{equation}
By the definition of the localized modulus, it is enough to prove that,
with probability at least $1-\delta$, simultaneously for all
$f\in\gF$,
\begin{equation}
    q_f-\widehat\Delta(g;f^\star)
    \leq
    \frac34q_f+\frac12\beta_\delta.
    \label{eq:vdr-finite-modulus-pointwise}
\end{equation}
Indeed, restricting this inequality to $q_f\leq r$ and taking the
supremum immediately gives Eqn.~\eqref{eq:finite-vdr-modulus}.

For the empirical quantities, let us denote
\begin{equation}
    g_{i,j}\coloneq g(x_i,a_{i,j}),
    \qquad
    \bar g_i
    \coloneq
    \frac1{N_a}\sum_{j=1}^{N_a}g_{i,j},
    \qquad
    \widetilde g_{i,j}
    \coloneq
    g_{i,j}-\bar g_i.
\end{equation}
We will repeatedly use the pairwise-centering identity
\begin{equation}
    \sum_{j<k}(u_j-u_k)^2
    =
    N_a\sum_{j=1}^{N_a}(u_j-\bar u)^2.
    \label{eq:vdr-pairwise-identity}
\end{equation}
It gives the empirical transformed squared error
\begin{equation}
    \widehat q_f
    \coloneq
    \frac1{N_x\binom{N_a}{2}}
    \sum_{i=1}^{N_x}
    \sum_{j<k}
    (g_{i,j}-g_{i,k})^2
    =
    \frac{2}{N_x(N_a-1)}
    \sum_{i,j}\widetilde g_{i,j}^2.
    \label{eq:vdr-finite-qhat}
\end{equation}

Since $r^\star=f^\star+b^\star$ and the action-independent
$b^\star(x_i)$ cancels from every within-context difference,
another application of Eqn.~\eqref{eq:vdr-pairwise-identity} gives
\begin{equation}
    \widehat\Delta(g;f^\star)
    =
    \widehat q_f
    -
    Z_f,
    \qquad
    Z_f
    \coloneq
    \frac{4}{N_x(N_a-1)}
    \sum_{i,j}\widetilde g_{i,j}\eta_{i,j}.
    \label{eq:vdr-finite-empirical-excess}
\end{equation}
Thus, the pointwise deviation has the master decomposition
\begin{equation}
    q_f-\widehat\Delta(g;f^\star)
    =
    (q_f-\widehat q_f)+Z_f.
\end{equation}
We control its two terms separately.

\paragraph{Step 1: control of the empirical quadratic term.}
For each block, define
\begin{equation}
    A_i(f)
    \coloneq
    \frac1{\binom{N_a}{2}}
    \sum_{j<k}
    (g_{i,j}-g_{i,k})^2.
\end{equation}
Then $\{A_i(f)\}_{i=1}^{N_x}$ are independent,
$\E[A_i(f)]=q_f$, and
\begin{equation}
    0\leq A_i(f)\leq K_{N_a},
    \qquad
    K_{N_a}
    \coloneq
    \frac{
        32F^2\lfloor N_a^2/4\rfloor
    }{
        N_a(N_a-1)
    }.
    \label{eq:vdr-finite-K}
\end{equation}
Indeed, $g=f-f^\star$ takes values in an interval of length at
most $4F$, and the maximum empirical variance of $N_a$ numbers
in an interval of length $4F$ gives Eqn.~\eqref{eq:vdr-finite-K}.
Moreover,
\begin{equation}
    q_f\leq8F^2,
    \qquad
    \Var(A_i(f))
    \leq
    K_{N_a}q_f.
\end{equation}
Hence, a one-sided Bernstein inequality followed by a union bound
over $f\in\gF$ gives, with probability at least $1-\delta/2$,
simultaneously for all $f\in\gF$,
\begin{equation}
    q_f-\widehat q_f
    \leq
    \sqrt{
        \frac{
            2K_{N_a}q_ft_\delta
        }{
            N_x
        }
    }
    +
    \frac{
        8F^2t_\delta
    }{
        3N_x
    }.
    \label{eq:vdr-finite-design-transfer}
\end{equation}
Using
$\sqrt{ab}\leq(a+b)/2$ gives\footnote{Note that this is the lower isometry condition of \citet[Definition 5]{liang15learning}.}
\begin{equation}
    \frac12q_f
    \leq
    \widehat q_f+d_\delta,
    \qquad
    d_\delta
    \coloneq
    \left(
        K_{N_a}+\frac83F^2
    \right)
    \frac{t_\delta}{N_x}.
    \label{eq:vdr-finite-lower-isometry}
\end{equation}

\paragraph{Step 2: control of the reward-noise term.}
Let
\begin{equation}
    c_\delta
    \coloneq
    \frac{
        \sigma^2t_\delta
    }{
        N_x(N_a-1)
    },
    \qquad
    b_\delta
    \coloneq
    \frac{
        32F\Rmax t_\delta
    }{
        3N
    }.
\end{equation}
Conditional on all sampled contexts and actions,
the variables
$\{\widetilde g_{i,j}\eta_{i,j}\}_{i,j}$ are independent and
mean zero.  Moreover,
\begin{equation}
    \sum_{i,j}
    \E\left[
        \widetilde g_{i,j}^2\eta_{i,j}^2
        \,\middle|\,
        \{x_k,a_{k,\ell}\}_{k,\ell}
    \right]
    \leq
    \frac{
        \sigma^2N_x(N_a-1)
    }{
        2
    }
    \widehat q_f,
\end{equation}
and, since $g$ takes values in an interval of length at most $4F$,
\begin{equation}
    |\widetilde g_{i,j}\eta_{i,j}|
    \leq
    8F\Rmax\frac{N_a-1}{N_a}.
\end{equation}
Thus, conditional Bernstein and a union bound over $f\in\gF$
imply, with conditional probability at least $1-\delta/2$,
simultaneously for all $f\in\gF$,
\begin{equation}
    Z_f
    \leq
    4\sqrt{c_\delta\widehat q_f}
    +
    b_\delta.
    \label{eq:vdr-finite-noise}
\end{equation}

\paragraph{Step 3: combining the two events.}
Work on the intersection of Eqns.~\eqref{eq:vdr-finite-lower-isometry} and \eqref{eq:vdr-finite-noise}, which has probability at least
$1-\delta$.
By Eqns.~\eqref{eq:vdr-finite-empirical-excess} and \eqref{eq:vdr-finite-noise},
\begin{align}
    q_f-\widehat\Delta(g;f^\star)
    &\leq
    q_f-\widehat q_f
    +
    4\sqrt{c_\delta\widehat q_f}
    +
    b_\delta
    \notag\\
    &\leq
    q_f-\frac12\widehat q_f
    +
    8c_\delta+b_\delta
    \tag{Young's inequality: $4\sqrt{c_\delta\widehat q_f}
    \leq
    \frac12\widehat q_f+8c_\delta$.} \\
    &\leq
    \frac34q_f
    +
    \frac12d_\delta
    +
    8c_\delta
    +
    b_\delta. \tag{Eqn.~\eqref{eq:vdr-finite-lower-isometry}} \\
    &= \frac{3}{4}q_f + \frac{1}{2}\beta_\delta \tag{Definitions of $K_{N_a}$, $d_\delta$, $c_\delta$, and $b_\delta$}
\end{align}

Thus, the preceding display proves
Eqn.~\eqref{eq:vdr-finite-modulus-pointwise}. Restricting it to $q_f\leq r$
and taking the supremum in the definition of
$\Psi(r;0,f^\circ_{\VDR})$ yields
\begin{equation}
    \Psi(r;0,f^\circ_{\VDR})
    \leq
    \frac34r+\frac12\beta_\delta,
    \qquad
    \forall r\geq0.
\end{equation}
This proves the claim.
\qed

%% file: 904Star_Estimator.tex
\section{\texorpdfstring{Exact Oracle Inequality for Finite Class via the \StarEstimator}{Exact Oracle Inequality for Finite Class via the Star Estimator}}
\label{app:star}

\def\starh{\operatorname{star}}

\subsection{\texorpdfstring{\StarEstimator and Its Guarantee}{Star Estimator and Its Guarantee}}

Here, we show that a different squared-loss estimator, namely the \StarEstimator~\citep{audibert07star,audibert09star,liang15learning}, admits an exact oracle inequality in our problem setting, unlike the usual ERM, which achieves a non-exact oracle inequality (see the \VR bound of \Cref{thm:finite}).

Recall that $1\le |\gF|<\infty$ with $\sup_{f\in\gF}\|f\|_\infty\le F$, and that the population and empirical \VR losses are defined as
\begin{equation}
    \gL^{\VR}(f)\coloneq \E_{x,a}\left[(f-r^\star)^2\right],\quad \widehat{\gL}^{\VR}(f)\coloneq\frac1N \sum_{i=1}^{N_x}\sum_{j=1}^{N_a}(f(x_i,a_{i,j})-r_{i,j})^2.
\end{equation}

We now recall the \StarEstimator{} (\STAR), first introduced by \citet{audibert07star} and later revisited by \citet{liang15learning}:
\begin{definition}[\StarEstimator]\label{def:star-estimator}
    For a given class $\gF$ and function $g$, define the \textbf{star hull}
    \begin{equation}
      \starh(\gF, g) \coloneq \{\lambda g + (1-\lambda)f : f\in\gF, \lambda\in[0,1]\}.
    \end{equation}

    The \StarEstimator $\hat f_N^{\STAR}$ is the following two-step squared-loss estimator:
    \begin{equation}
       \hat f_N^{\STAR} \in \argmin_{f\in\starh(\gF, \hat g_N)}\widehat\gL^{\VR}(f), \quad \text{where} \quad \hat g_N\in\argmin_{f\in\gF}\widehat \gL^{\VR}(f).
    \end{equation}
\end{definition}
Although $\starh(\gF,\hat g_N)$ is data-dependent, the output $\hat f_N^{\STAR}$ belongs to the deterministic function class
\begin{equation}
    \StarClass \coloneq \bigcup_{g \in \gF} \starh(\gF, g) = \{\lambda g + (1-\lambda)f : g, f \in \gF, \lambda\in[0,1]\}.
    \label{eq:star-hull}
\end{equation}

Since $\StarClass \subseteq \mathrm{conv}(\gF)$, we immediately have $\sup_{g\in\StarClass}\|g\|_\infty\le F$.

We now state the error bound for \STAR in the finite class, whose proof is presented in \Cref{app:proof-star-estimator}:
\begin{theorem}[\StarEstimator Bound for Finite $\gF$]
\label{thm:star-oracle}
    Let $\delta\in(0,1)$, $t_x
        \coloneq
        \log
        \frac{
            2|\gF|^2(N_x+1)
        }{
            \delta
        }$, and
    \begin{align}
        \rho_\delta
        \coloneq\;
        \left(
            \frac{1331}{27}F^2
            +
            \frac{260}{3}F{\color{red}\Rmax}
            +
            38{\color{red}\Rmax^2}
        \right)
        \frac{t_x}{N_x}
        +
            38\sigma^2
        \frac{t_x}{N}.
        \label{eq:star-rho}
    \end{align}
    Then, with probability at least $1-\delta$,
    \begin{equation}
        \E_{x,a}\left[ \left( \gC\hat{f}_N^{\STAR} - \gC r^\star \right)^2 \right]
        \leq\;
        {\color{red}\varepsilon_{\rm aprx}}
        +\rho_\delta
        = {\color{red}\varepsilon_{\rm aprx}} + \widetilde{\gO}\left( \frac{F^2 + {\color{red}\Rmax^2}}{N_x} + \frac{\sigma^2}{N} \right).
    \label{eq:star-oracle-rate}
    \end{equation}
\end{theorem}

Recall the error bound guarantee of \VR (\Cref{thm:finite}): $2 {\color{red}\varepsilon_{\rm aprx}} + \widetilde{\gO}\left( \frac{F^2 + F{\color{red}\Rmax}}{N_x} + \frac{\sigma^2}{N} \right)$.
Comparing the two, \STAR reduces the coefficient of ${\color{red}\varepsilon_{\rm aprx}}$ from two to one, at the price of ${\color{red}\Rmax^2}$ instead of $F{\color{red}\Rmax}$ in the $N_x^{-1}$ term.
Order-wise, the two bounds are meaningfully different when ${\color{red}\Rmax}\gg F$ and ${\color{red}\Rmax^2}/{N_x}\gg {\color{red}\varepsilon_{\rm aprx}}, \sigma^2/N$, up to logarithmic factors, in which case the \VR bound is smaller than the \STAR bound.

\subsection{\texorpdfstring{Proof of \Cref{thm:star-oracle}: \StarEstimator Bound for Finite Class}{Proof of Theorem E.1: Star Estimator Bound for Finite Class}}
\label{app:proof-star-estimator}

As done for \VR, we set our comparator as $f^\circ \in \argmin_{f \in \gF} \left\{ \gL^{\VR}(f) \coloneq \E\left[ (f - r^\star)^2 \right] \right\}$, which satisfies $\gL^{\VR}(f^\circ) = {\color{red}\varepsilon_{\rm aprx}}$, i.e.,
\begin{equation}
\label{eqn:star-finite-oracle}
    \E_{x,a}\left[ (\hat{f}_N^{\STAR} - r^\star)^2 \right] = {\color{red}\varepsilon_{\rm aprx}} + \underbrace{\gL^{\VR}(\hat{f}_N^{\STAR}) - \gL^{\VR}(f^\circ)}_{(*)}.
\end{equation}
We again use \Cref{prop:localised} to tightly upper-bound $(*)$.

For $h\in\StarClass-f^\circ$, define as before
\begin{equation}
    \Delta(h;f^\circ)
    \coloneq
    \gL^{\VR}(f^\circ+h)-\gL^{\VR}(f^\circ),
\end{equation}
\begin{equation}
    \widehat\Delta(h;f^\circ)
    \coloneq
    \widehat{\gL}^{\VR}(f^\circ+h)
    -
    \widehat{\gL}^{\VR}(f^\circ),
    \qquad
    \widehat Q(h)
    \coloneq
    \frac1N
    \sum_{i,j}h(x_i,a_{i,j})^2.
\end{equation}

First, unlike \VR, \citet[Lemma 1]{liang15learning} showed that the \StarEstimator satisfies condition $(i)$ of \Cref{prop:localised} with \emph{nonzero} $\lambda = \frac{1}{18}$, \emph{deterministically}:\footnote{Precisely, $\gF$ in \Cref{prop:localised} should change to $\StarClass$ for this analysis.}
\begin{equation}
    \widehat\Delta(\hat f_N^{\STAR} - f^\circ;f^\circ)\le -\frac{1}{18}\widehat{Q}(\hat f_N^{\STAR} - f^\circ).
\end{equation}

For the \StarEstimator, we use the localized block-offset modulus with respect to $\StarClass$, not $\gF$, whose definition we recall here:
\begin{equation}
    \Psi^{\STAR}(r)
    \coloneq
    \sup_{\substack{
        h\in\StarClass-f^\circ\\
        \Delta(h;f^\circ)\leq r
    }}
    \left[
        \Delta(h;f^\circ)
        -
        \widehat\Delta(h;f^\circ)
        -
        \frac1{18}\widehat Q(h)
    \right].
    \label{eq:star-modulus}
\end{equation}

The following lemma, whose proof is provided in the next section, verifies condition $(ii)$:
\begin{lemma}
\label{lem:star-modulus}
    Let $\rho_\delta$ be as defined in \Cref{thm:star-oracle}.
    With probability at least $1-\delta$,
    \begin{equation}
        \Psi^{\STAR}(r)
        \leq
        \rho_\delta,
        \qquad
        \forall r\geq0.
        \label{eq:star-modulus-bound}
    \end{equation}
\end{lemma}
On the event \eqref{eq:star-modulus-bound}, the fixed-point argument of \Cref{prop:localised} gives
\begin{equation}
    \gL^{\VR}(\hat f_N^{\STAR})
    -
    \gL^{\VR}(f^\circ)
    \leq
    \sup
    \left\{
        r\geq0:
        r\leq\rho_\delta
    \right\}
    =
    \rho_\delta.
\end{equation}

Combining Eqn.~\eqref{eqn:star-finite-oracle} with the fact that the variance is bounded by the second moment\footnote{\(\E[(\gC g(x,a))^2]=\E_x[\Var[g(x,a)\mid x]]\leq \E[g(x,a)^2]\)} gives,
\begin{equation}
    \E_{x,a}\left[ \left( \gC\hat{f}_N^{\STAR} - \gC r^\star \right)^2 \right]
        \leq\;
        \E_{x,a}\left[ \left( 
        \hat{f}_N^{\STAR} - r^\star \right)^2 \right]
        \leq\;
        {\color{red}\varepsilon_{\rm aprx}}
        +\rho_\delta.
\end{equation}
This concludes the proof.
\qed

\subsection{\texorpdfstring{Proof of \Cref{lem:star-modulus}: Localized Modulus for the \StarEstimator in Finite Class}{Proof of Lemma E.1: Localized Modulus for the Star Estimator in Finite Class}}
To derive a uniform high-probability bound over $\StarClass$, we first construct a finite net of $\StarClass$.
For each element of the net, we apply Bernstein's inequality to control the loss difference and the squared distance from $f^\circ$.
A union bound makes these bounds hold simultaneously over the net.
We then extend them to every $s\in\StarClass$ using the approximation property of the net.
Finally, combining these bounds yields a uniform bound on the localized modulus.

\paragraph{Step 1: construct a finite net over $\StarClass$.}
We construct a finite net of the deterministic class $\StarClass$ in Eqn.~\eqref{eq:star-hull}.
Let 
\begin{equation}
    \StarNet
    \coloneq
    \left\{
        (1-\alpha)f+\alpha g:
        f,g\in\gF,\;
        \alpha\in
        \left\{
            0,\frac1{N_x},\ldots,1
        \right\}
    \right\}.
\end{equation}
Then 
$|\StarNet|\leq|\gF|^2(N_x+1)\label{eq:star-net-size},$
and for any given $s\in\StarClass$, there exist $\alpha_s\in[0,1]$ and $f_1, f_2\in \gF$ such that
$s = (1-\alpha_s)f_1 +\alpha_sf_2$.

Let $\bar\alpha\coloneq\lfloor N_x\alpha_s\rfloor/N_x$ and $h=(1-\bar\alpha)f_1 +\bar\alpha f_2$. Then since $|\alpha_s - \bar\alpha|\le 1/N_x$, 
\begin{align}
    \|s-h\|_\infty &=\|(\bar\alpha-\alpha_s)f_1+ (\alpha_s -\bar\alpha)f_2\|_\infty\\
    &\le \|(\bar\alpha-\alpha_s)f_1\|_\infty +\|(\alpha_s -\bar\alpha)f_2\|_\infty
    \le \frac{2F}{N_x}.
\end{align}
Therefore, for a given $s\in\StarClass$, there exists $h\in\StarNet$ such that
\begin{equation}
    \qquad
    \bignorm{s-h}_\infty
    \leq
    \frac{2F}{N_x}.
    \label{eq:star-net}
\end{equation}

\paragraph{Step 2: concentration bound via Bernstein's inequality.}
Fix $h\in\StarNet$ and denote
\begin{equation}
    d_h
    \coloneq
    h-f^\circ,
    \qquad
    D_h
    \coloneq
    \E_{x,a}[d_h(x,a)^2],
    \qquad
    \widehat D_h
    \coloneq
    \frac1N\sum_{i,j}d_h(x_i,a_{i,j})^2,
\end{equation}
as well as
\begin{align}
    &\ell_h(x,a)\coloneq (h(x,a)-r^\star(x,a))^2 - (f^\circ(x,a)-r^\star(x,a))^2, \\
    &\Delta_h\coloneq
    \gL^{\VR}(h)-\gL^{\VR}(f^\circ),
    \qquad
    \widehat\Delta_h
    \coloneq
    \widehat{\gL}^{\VR}(h)
    -
    \widehat{\gL}^{\VR}(f^\circ).
\end{align}
We first obtain uniform one-sided concentration bounds for $\Delta_h - \widehat\Delta_h$ and $D_h-\widehat D_h$.
For each within-block observation, define the loss difference and the corresponding block average by
\begin{equation}
    Z_{i,j}(h)
    \coloneq
    \bigl(
        h(x_i,a_{i,j})-r_{i,j}
    \bigr)^2
    -
    \bigl(
        f^\circ(x_i,a_{i,j})-r_{i,j}
    \bigr)^2,
    \quad
    Z_i(h)
    \coloneq
    \frac1{N_a}
    \sum_{j=1}^{N_a}
    Z_{i,j}(h).
\end{equation}
Then the variables $\{Z_i(h)\}_{i=1}^{N_x}$ are independent, with $\E[Z_i(h)]=\Delta_h$ and $\frac1{N_x}\sum_iZ_i(h)=\widehat\Delta_h$.
Since every function in $\StarClass$ and $f^\circ$ is bounded by $F$ and $|r_{i,j}|\le {\color{red}\Rmax}$,
\begin{equation}
    |Z_{i,j}(h)|
    \leq
    {F(F+4{\color{red}\Rmax})}.
    \label{eq:star-range}
\end{equation}
Moreover, 
    $\ell_h(x,a)
    =
    d_h(x,a)
    \bigl(
        h(x,a)+f^\circ(x,a)-2r^\star(x,a)
    \bigr),$
    so
\begin{equation}
    \ell_h(x,a)^2
    \leq
    4(F+{\color{red}\Rmax})^2d_h(x,a)^2.
    \label{eq:star-ell-square}
\end{equation}
Since 
$Z_{i,j}(h)=\ell_h(x_i, a_{i,j})-2d_h(x_i,a_{i,j})\eta_{i,j}$,
the conditional mean-zero property of the reward noise gives
\begin{equation}
    \E[Z_{i,j}(h)^2]
    \leq
    4\bigl(
        (F+{\color{red}\Rmax})^2+\sigma^2
    \bigr)D_h.
    \label{eq:star-single-second-moment}
\end{equation}
Let $m_h(x)\coloneq\E_{a\sim\piref(\cdot\mid x)}[\ell_h(x,a)]$.
By Jensen's inequality and Eqn.~\eqref{eq:star-ell-square},
\begin{align}
    \Var(m_h(X)) &\le \E_X[m_h(X)^2] = \E_X[(\E_a[\ell_h(X, a)])^2] \nonumber\\
    &\le \E_X[\E_a[\ell_h(X,a)^2]] \le 4(F+{\color{red}\Rmax})^2\E_{X,a}[d_h (X,a)^2] \nonumber\\
    &\le 4(F+{\color{red}\Rmax})^2D_h
\end{align}
Conditional on $x_i$, the variables
$\{Z_{i,j}(h)\}_{j=1}^{N_a}$ are independent with common
conditional mean $m_h(x_i)$.  Therefore, the law of total variance gives
\begin{align}
    \Var(Z_{1,1}(h))&=\Var(\E[Z_{1,1}(h)|X]) + \E[\Var(Z_{1,1}(h)|X)] \nonumber\\
    &=\Var(m_h(X)) + \E[\Var(Z_{1,1}(h)|X)]
    \label{eq:star-u1-var}
\end{align}
Applying the law of total variance again and using Eqn.~\eqref{eq:star-u1-var},
\begin{align}
    \Var(Z_i(h))&=\Var(\E[Z_i(h)|X]) + \E[\Var(Z_i(h)|X)]\nonumber \\
    &=\Var(m_h(X)) + \frac{1}{N_a}\E\left[\Var(Z_{1,1}(h)\middle| X)\right]\nonumber\\
    &= \left(1-\frac{1}{N_a}\right)\Var(m_h(X)) + \frac{1}{N_a}\Var(Z_{1,1}(h))\nonumber\\
    &\le 4\left((F+{\color{red}\Rmax})^2+\frac{\sigma^2}{N_a}\right)D_h.
    \label{eq:star-block-variance}
\end{align}
Define
\begin{equation}
    a
    \coloneq
    t_x
    \left(
        \frac{(F+{\color{red}\Rmax})^2}{N_x}
        +
        \frac{\sigma^2}{N}
    \right),
    \qquad
    b
    \coloneq
    \frac{t_x}{N_x}.
    \label{eq:star-ab}
\end{equation}
By one-sided Bernstein's inequality, Eqn.~\eqref{eq:star-range}, and a union bound over $h\in\StarNet$, with probability at least $1-\delta/2$, simultaneously for all $h\in\StarNet$,
\begin{equation}
    \Delta_h-\widehat\Delta_h
    \leq
    \sqrt{8aD_h}
    +
    \frac{
        {2F(F+4{\color{red}\Rmax})}b
    }{
        3
    }.
    \label{eq:star-loss-concentration}
\end{equation}

We also require a one-sided comparison between the population and empirical squared distances.
For each block, let
\begin{equation}
    V_i(h)
    \coloneq
    \frac1{N_a}
    \sum_{j=1}^{N_a}
    d_h(x_i,a_{i,j})^2.
\end{equation}
Since $|d_h|\le 2F$,
$0\leq V_i(h)\leq4F^2,\;
    \E[V_i(h)]=D_h,$ and $\Var(V_i(h))
    \leq
    4F^2D_h$.
    
A second one-sided Bernstein inequality and union bound therefore give, with probability at least $1-\delta/2$, simultaneously for all $h\in\StarNet$,
\begin{equation}
    D_h-\widehat D_h
    \leq
    \sqrt{8F^2bD_h}
    +
    \frac{4F^2b}{3}.
    \label{eq:star-distance-concentration}
\end{equation}

\paragraph{Step 3: upper bound on the localized modulus.}
We then work on the intersection of Eqn.~\eqref{eq:star-loss-concentration} and Eqn.~\eqref{eq:star-distance-concentration}, which has probability at least $1-\delta$.

Let $s\in\StarClass$ be arbitrary and choose $h\in\StarNet$, whose existence is proved in Eqn.~\eqref{eq:star-net}.
Since $s, h, f^\circ$ are all bounded by $F$, while $|r^\star| ,|r_{i,j}|\le {\color{red}\Rmax}$, we have:
\begin{align}
    |\Delta(s-f^\circ;f^\circ)-\Delta_h|
    &\leq
    \frac{4F(F+{\color{red}\Rmax})}{N_x},
    \label{eq:star-net-pop} \\
    \left|
        \widehat\Delta(s-f^\circ;f^\circ)
        -
        \widehat\Delta_h
    \right|
    &\leq
    \frac{4F(F+{\color{red}\Rmax})}{N_x},
    \label{eq:star-net-emp} \\
    \left|
        \widehat Q(s-f^\circ)
        -
        \widehat D_h
    \right|
    &\leq
    \frac{8F^2}{N_x}.
    \label{eq:star-net-distance}
\end{align}
Then, with $c_0\coloneq\frac1{18}$, combining
Eqn.~\eqref{eq:star-loss-concentration}--~\eqref{eq:star-net-distance} gives the following:
for every $s\in\StarClass$,
\begin{align}
    &
    \Delta(s-f^\circ;f^\circ)
    -
    \widehat\Delta(s-f^\circ;f^\circ)
    -
    c_0\widehat Q(s-f^\circ)
    \notag\\
    &\leq
    -c_0D_h
    +
    \left(
        \sqrt{8a}
        +
        c_0\sqrt{8F^2b}
    \right)
    \sqrt{D_h}
    +
    \frac{4c_0F^2b}{3}
    +
    \frac{2F(F+4{\color{red}\Rmax})b}{3}
    \notag\\
    &\qquad+
    \frac{
        8F(F+{\color{red}\Rmax})+8c_0F^2
    }{
        N_x
    }.
    \label{eq:star-offset-before-max}
\end{align}
The right-hand side is now independent of the population excess loss of $s$. Maximizing its first two terms over $\sqrt{D_h}\ge 0$ and using $c_0=1/18$ gives
\begin{equation}
    \sup_{D_h\geq0}
    \left\{
        -c_0D_h
        +
        \left(
            \sqrt{8a}
            +
            c_0\sqrt{8F^2b}
        \right)
        \sqrt{D_h}
    \right\}
    =
    36a+\frac19F^2b+4F\sqrt{ab}.
\end{equation}
Using $4F\sqrt{ab}\le 2a+2F^2b$ and $t_x>1$, we obtain from Eqn.~\eqref{eq:star-offset-before-max}
\begin{align}
    &
    \Delta(s-f^\circ;f^\circ)
    -
    \widehat\Delta(s-f^\circ;f^\circ)
    -
    c_0\widehat Q(s-f^\circ)
    \notag\\
    &\leq
    38t_x
    \left(
        \frac{(F+{\color{red}\Rmax})^2}{N_x}
        +
        \frac{\sigma^2}{N}
    \right)
    +
    \left(
        \frac{59}{27}F^2
        +
        \frac{2F(F+4{\color{red}\Rmax})}{3}
    \right)
    \frac{t_x}{N_x}
    \notag\\
    &\qquad+
    \left(
        \frac49F^2
        +
        8F(F+{\color{red}\Rmax})
    \right)
    \frac{t_x}{N_x}.
    \label{eq:star-offset-expanded}
\end{align}
Expanding $(F+{\color{red}\Rmax})^2$ and collecting terms yields exactly
\begin{align}
    &
    \Delta(s-f^\circ;f^\circ)
    -
    \widehat\Delta(s-f^\circ;f^\circ)
    -
    \frac1{18}\widehat Q(s-f^\circ)
    \notag\\
    &\leq
    \underbrace{\left(
        \frac{1331}{27}F^2
        +
        \frac{260}{3}F{\color{red}\Rmax}
        +
        38{\color{red}\Rmax}^2
    \right)
    \frac{t_x}{N_x}+
        38\sigma^2
    \frac{t_x}{N}}_{=\rho_\delta}.
    \label{eq:star-offset-final}
\end{align}
Since Eqn.~\eqref{eq:star-offset-final} holds for every $s\in\StarClass$,
restricting to
$\Delta(s-f^\circ;f^\circ)\leq r$
and taking the supremum in
Eqn.~\eqref{eq:star-modulus} gives
\begin{equation}
    \Psi^{\STAR}(r)
    \leq
    \rho_\delta,
    \qquad
    \forall r\geq0.
\end{equation}
This proves the lemma.
\qed

%% file: 905Linear.tex
\section{\texorpdfstring{Deferred Proofs from \Cref{sec:linear}: Linear Class}{Deferred Proofs from Section 4: Linear Class}}
\label{app:linear}

\subsection{\texorpdfstring{\ValueRegression{} (\VR)}{Value Regression (VR)}}

\subsubsection{\texorpdfstring{Proof of \Cref{thm:vr-linear}: \VR Bound for Linear $\gF_{\rm lin}$}{Proof of Theorem 4.1: VR Bound for a Linear Function Class}}
\label{app:proof-thm-linear-vr}

Write
$f^\circ_{\VR}(x,a)=f^\circ(x,a)=\langle\vphi(x,a),\vtheta^\circ\rangle$
for the population least-squares projection defined in
\Cref{asm:linear-vr-design}, where
\begin{equation}
    \vtheta^\circ \in \argmin_{\vtheta \in \sR^d} \E_{x,a}\left[ \left( \langle\vphi(x,a),\vtheta\rangle - r^\star(x, a) \right)^2 \right].
\end{equation}

Throughout this subsection, $\gL^{\VR}$ and $\widehat{\gL}^{\VR}$ denote
the generic losses in \Cref{sec:localised} instantiated with
$\gT=\gT^{\VR}$. For $h=f-f^\circ_{\VR}$, the corresponding population
and empirical excess losses are
\begin{align}
    \Delta(h;f^\circ_{\VR})
    &\coloneq
    \gL^{\VR}(f^\circ_{\VR}+h)-\gL^{\VR}(f^\circ_{\VR}) \\
    &=
    \E_{x,a}\left[
        (f^\circ_{\VR}+h-r^\star)^2
        -(f^\circ_{\VR}-r^\star)^2
    \right],
\end{align}
and
\begin{align}
    \widehat\Delta(h;f^\circ_{\VR})
    &\coloneq
    \widehat{\gL}^{\VR}(f^\circ_{\VR}+h)
    -\widehat{\gL}^{\VR}(f^\circ_{\VR}) \\
    &=
    \frac1{N_xN_a}\sum_{i=1}^{N_x}\sum_{j=1}^{N_a}
    \left[
        \bigl(f^\circ_{\VR}(x_i,a_{i,j})+h(x_i,a_{i,j})-r_{i,j}\bigr)^2
        -
        \bigl(f^\circ_{\VR}(x_i,a_{i,j})-r_{i,j}\bigr)^2
    \right].
    \label{eq:linear-vr-Deltahat-recall}
\end{align}
The empirical quadratic term in the localized modulus is
\begin{equation}
    \widehat Q(h)
    =
    \widehat Q^{\VR}(h)
    \coloneq
    \frac1{N_xN_a}
    \sum_{i=1}^{N_x}\sum_{j=1}^{N_a}h(x_i,a_{i,j})^2.
    \label{eq:linear-vr-Qhat-recall}
\end{equation}
Thus, for this proof,
\begin{equation}
    \Psi(r;1,f^\circ_{\VR})
    =
    \sup_{\substack{
        h\in\gF_{\rm lin}-f^\circ_{\VR}\\
        \Delta(h;f^\circ_{\VR})\leq r
    }}
    \left[
        \Delta(h;f^\circ_{\VR})
        -\widehat\Delta(h;f^\circ_{\VR})
        -\widehat Q^{\VR}(h)
    \right].
    \label{eq:linear-vr-Psi-recall}
\end{equation}

For the proof, let us denote ${\color{red}\varepsilon_{\mathrm{blk}}}
\coloneq
{\color{red}\varepsilon_{\mathrm{x}}}
+
\frac{{\color{red}\varepsilon_{\mathrm{a}}}}{N_a}$, $t_\delta\coloneq\log\frac{3}{\delta}$, and
\begin{align}
\Gamma_\delta
\coloneq{}&
3\sqrt{\frac{d\,{\color{red}\varepsilon_{\mathrm{blk}}}}{N_x}}
    \left(1+\sqrt{8t_\delta}\right)
+\frac{4\sqrt{{\color{red}\varepsilon_{\mathrm{z}}}d}\,t_\delta}{N_x}
+\sqrt{\frac{3\sigma^2d}{N}}
    \left(1+\sqrt{8t_\delta}\right)
+\frac{8{\color{red}\Rmax}{\color{blue}\alpha_0}\sqrt d\,t_\delta}{N}.
\label{eq:vr-linear-Gamma}
\end{align}

We now state the following lemma, whose proof is deferred to \Cref{app:proof-lem-linear-vr-modulus}, that bounds the localized modulus of \VR in $\gF_{\rm lin}$:
\begin{lemma}[Localized modulus for linear \VR]
\label{lem:linear-vr-modulus}
Suppose the assumptions of \Cref{thm:vr-linear} hold.
Then, with probability at least $1-\delta$,
\begin{equation}
    \Psi(r;1,f^\circ_{\VR})
    \leq
    \frac23r+\frac13\Gamma_\delta^2,
    \qquad
    \forall r\geq0.
    \label{eq:linear-vr-modulus}
\end{equation}
\end{lemma}

The proof of \Cref{thm:vr-linear} then proceeds using the same recipe.

Let
\begin{equation}
    \hat{\vtheta}_N^{\VR}
    \in
    \argmin_{\vtheta\in\sR^d}
    \frac1{N_x}\sum_{i=1}^{N_x}\frac1{N_a}
    \sum_{j=1}^{N_a}
    \left(
        \langle\vtheta,\vphi(x_i,a_{i,j})\rangle-r_{i,j}
    \right)^2.
\end{equation}
Let $\widehat h_{\VR}\coloneq\hat f_N^{\VR}-f^\circ_{\VR}$. The OLS
normal equation gives
\begin{equation}
    \widehat\Delta(\widehat h_{\VR};f^\circ_{\VR})
    =
    \widehat{\gL}^{\VR}(\hat f_N^{\VR})
    -\widehat{\gL}^{\VR}(f^\circ_{\VR})
    =
    -\widehat Q^{\VR}(\widehat h_{\VR}).
\end{equation}
Thus condition~$(i)$ of \Cref{prop:localised} holds with $\lambda=1$.
By \Cref{lem:linear-vr-modulus}, condition~$(ii)$ holds on an event of
probability at least $1-\delta$. On this event,
\begin{equation}
    \gL^{\VR}(\hat f_N^{\VR})-\gL^{\VR}(f^\circ_{\VR})
    \leq
    \sup\left\{
        r\geq0:
        r\leq\frac23r+\frac13\Gamma_\delta^2
    \right\}
    =
    \Gamma_\delta^2.
\end{equation}
Since
$\gL^{\VR}(f^\circ_{\VR})={\color{red}\varepsilon_{\rm aprx}}$
and we have \(\E[(\gC g(x,a))^2]=\E_x[\Var[g(x,a)\mid x]]\leq \E[g(x,a)^2]\), the original argument gives
\begin{equation}
    \E_{x,a}\left[
        \bigl(\gC\hat f_N^{\VR}-\gC r^\star\bigr)^2
    \right]
    \leq
    \E_{x,a}\left[
        \bigl(\hat f_N^{\VR}-r^\star\bigr)^2
    \right]
    \leq
    {\color{red}\varepsilon_{\rm aprx}}+\Gamma_\delta^2.
    \label{eq:linear-vr-contraction-bound}
\end{equation}
Finally, since
${\color{red}\varepsilon_{\mathrm{blk}}}
={\color{red}\varepsilon_{\mathrm{x}}}
+{\color{red}\varepsilon_{\mathrm{a}}}/N_a$ and $N=N_xN_a$, we have that
\begin{equation}
    \Gamma_\delta^2
    =
    \widetilde{\gO}\left(
        \frac{{\color{red}\varepsilon_{\mathrm{x}}}d}{N_x}
        +\frac{{\color{red}\varepsilon_{\mathrm{z}}}d}{N_x^2}
        +\frac{({\color{red}\varepsilon_{\mathrm{a}}}+\sigma^2)d}{N}
        +\frac{{\color{red}\Rmax^2}{\color{blue}\alpha_0^2}d}{N^2}
    \right).
    \label{eq:linear-vr-Gamma-rate}
\end{equation}
Combining Eqns.~\eqref{eq:linear-vr-contraction-bound}--\eqref{eq:linear-vr-Gamma-rate}
proves \Cref{thm:vr-linear}.
\qed

\subsubsection{\texorpdfstring{Proof of \Cref{lem:linear-vr-modulus}: Localized Modulus for Linear \VR}{Proof of Lemma F.1: Localized Modulus for Linear VR}}
\label{app:proof-lem-linear-vr-modulus}

Write
\begin{equation}
    \vq(x,a)\coloneq {\color{red}e^\circ}(x,a)\vz(x,a),
    \qquad
    \vm(x)\coloneq
    \E_{A\sim\piref(\cdot\mid x)}[\vq(x,A)].
\end{equation}
Since \(\vtheta^\circ\) is the population least-squares projection over
\(\sR^d\), its normal equation gives
\begin{equation}
    \E_{x,a}[\vphi(x,a){\color{red}e^\circ}(x,a)]
    =
    \vzero,
    \qquad\text{and hence,}\qquad
    \E_X[\vm(X)]=\vzero.
    \label{eq:linear-vr-population-normal-equation}
\end{equation}

For the observed dataset, we define the following random quantities: for each $i \in [N_x]$,
\begin{equation}
    \mG_i
    \coloneq
    \frac1{N_a}\sum_{j=1}^{N_a}
    \vz(x_i,a_{i,j})\vz(x_i,a_{i,j})^\top,
    \qquad
    \vq_i
    \coloneq
    \frac1{N_a}\sum_{j=1}^{N_a}
    \vq(x_i,a_{i,j}),
\end{equation}
and ``score vectors'' averaged over $i \in [N_x]$:
\begin{equation}
\label{eq:linear-vr-scores}
    \mH
    \coloneq
    \frac1{N_x}\sum_{i=1}^{N_x}\mG_i,
    \qquad
    \vt_{\rm mis}
    \coloneq
    \frac1{N_x}\sum_{i=1}^{N_x}\vq_i,
    \qquad
    \vt_\eta
    \coloneq
    \frac1N\sum_{i=1}^{N_x}\sum_{j=1}^{N_a}
    \vz(x_i,a_{i,j})\eta_{i,j}.
\end{equation}

We now express $\Delta$, $\widehat\Delta$, and $\widehat Q^{\VR}$ in the whitened coordinates. Fix
\(h=f-f^\circ_{\VR}\in\gF_{\rm lin}-f^\circ_{\VR}\), and write
\begin{equation}
    \vw\coloneq\mSigma^{1/2}(\vtheta-\vtheta^\circ),
    \qquad
    h(x,a)=\vz(x,a)^\top \vw.
\end{equation}
The population normal equation gives
\begin{equation}
    \Delta(h;f^\circ_{\VR})
    =
    \E_{x,a}[h(x,a)^2]
    =
    \bignorm{\vw}_2^2.
    \label{eq:linear-vr-excess-is-quadratic}
\end{equation}
Moreover, \(\widehat Q^{\VR}(h)=\vw^\top\mH \vw\), and direct expansion of
the empirical excess loss gives
\begin{equation}
    \widehat\Delta(h;f^\circ_{\VR})
    =
    \vw^\top\mH \vw
    +
    2\vw^\top \vt_{\rm mis}
    -
    2\vw^\top \vt_\eta.
\end{equation}
Consequently, the process to be controlled is
\begin{equation}
    \Delta(h;f^\circ_{\VR})
    -
    \widehat\Delta(h;f^\circ_{\VR})
    -
    \widehat Q^{\VR}(h)
    =
    \bignorm{\vw}_2^2
    -
    2\vw^\top\mH \vw
    +
    2\vw^\top(\vt_\eta-\vt_{\rm mis}).
    \label{eq:linear-vr-modulus-process}
\end{equation}
It is therefore enough to show, uniformly over $h$, that the
right-hand side of Eqn.~\eqref{eq:linear-vr-modulus-process} is at most
$2\Delta(h;f^\circ_{\VR})/3+\Gamma_\delta^2/3$; restricting this bound to
$\Delta(h;f^\circ_{\VR})\leq r$ concludes the proof.

We now prove the uniform bound through a lower-isometry event for \(\mH\) and concentration of the two score vectors, $\vt_{\rm mis}$ and $\vt_\eta$.
We use the vector Bernstein inequality of
\citet[Proposition 1.2]{hsu2012bernstein}: if independent\footnote{The original statement holds for martingale difference vector sequences; for our purpose, it suffices to apply it to independent vector sequences, as done in \citet{hsu2014random}.} mean-zero random
vectors \(\vy_k\) satisfy
\(\sum_k\E\bignorm{\vy_k}_2^2\leq v\) and
\(\bignorm{\vy_k}_2\leq L\) almost surely, then, with probability at
least \(1-e^{-t}\),
\begin{equation}
    \bignorm{\textstyle\sum_k\vy_k}_2
    \leq
    \sqrt v\left(1+\sqrt{8t}\right)
    +
    \frac43Lt.
    \label{eq:linear-vr-vector-bernstein}
\end{equation}
This lets us avoid a separate covering argument.

Recall from \Cref{subsec:linear_vr} that $\mG_{N_a} = \frac{1}{N_a}\sum_{j=1}^{N_a} \vz(X, A_j) \vz(X, A_j)^\top$ where $\vz(x,a) = \mSigma^{-1/2} \vphi(x,a)$ and $\mSigma = \E_{x,a}[\vphi(x,a)\vphi(x,a)^\top]$.

\paragraph{Step 1: lower isometry of the empirical Gram matrix.}
The matrices $\mG_i$ are i.i.d. copies of $\mG_{N_a}$ and satisfy $\E[\mG_i]=\mI_{d}$.
Moreover, $\mG_i\succeq\vzero$, and hence
$\lambda_{\max}(\mI_d-\mG_i)\leq1$.

The one-sided matrix Bernstein inequality
\citep[Theorem 6.6.1]{tropp2015survey} therefore gives, with probability at least
\(1-\delta/3\),
\begin{align}
    \lambda_{\max}(\mI_d-\mH)
    \leq
    \sqrt{
        \frac{
            2\bignormop{\Cov(\mG_{N_a})}
            \log(3d/\delta)
        }{N_x}
    }
    +
    \frac{2\log(3d/\delta)}{3N_x}
    \leq
    \frac23,
\end{align}
where the last step uses the sample-size condition from \Cref{thm:vr-linear}.
Thus, we have\footnote{This is precisely the lower isometry condition of \citet[Definition 5]{liang15learning}.}
\begin{equation}
    \mH\succeq\frac13\mI_d.
    \label{eq:linear-vr-lower-isometry}
\end{equation}

\paragraph{Step 2: control of the misspecification score.}
The blocks are independent, and
Eqn.~\eqref{eq:linear-vr-population-normal-equation} implies
\(\E[\vq_i]=\vzero\).
Moreover, by conditioning on the context,
\begin{equation}
    \E\left[\bignorm{\vq_i}_2^2\right]
    =
    d\left(
        {\color{red}\varepsilon_{\mathrm{x}}}
        +
        \frac{{\color{red}\varepsilon_{\mathrm{a}}}}{N_a}
    \right)
    =
    d{\color{red}\varepsilon_{\mathrm{blk}}}.
    \label{eq:linear-vr-block-second-moment}
\end{equation}
Assumption~\ref{asm:linear-vr-design} gives
\(\bignorm{\vq_i}_2\leq
    \sqrt{{\color{red}\varepsilon_{\mathrm{z}}}d}\).
Applying Eqn.~\eqref{eq:linear-vr-vector-bernstein} to the
independent block averages and using
Eqn.~\eqref{eq:linear-vr-block-second-moment}, we obtain, with
probability at least \(1-\delta/3\),
\begin{equation}
    \bignorm{\vt_{\rm mis}}_2
    \leq
    \underbrace{
    \sqrt{\frac{d{\color{red}\varepsilon_{\mathrm{blk}}}}{N_x}}
       \left(1+\sqrt{8t_\delta}\right)
    +
    \frac{4\sqrt{{\color{red}\varepsilon_{\mathrm{z}}}d}\,t_\delta}{3N_x}
    }_{\eqqcolon B_{\rm mis}}.
    \label{eq:linear-vr-mis-score}
\end{equation}

\paragraph{Step 3: control of the reward-noise score.}
On the event in Eqn.~\eqref{eq:linear-vr-lower-isometry}, condition on
the complete context--action design $\{(x_i,a_{i,j})\}_{i,j}$.
The vectors
\begin{equation}
    \vy_{i,j}
    \coloneq
    \frac1N \mH^{-1/2}\vz(x_i,a_{i,j})\eta_{i,j}
\end{equation}
are conditionally independent and mean zero. Since
\(\sum_{i,j}\vz_{i,j}\vz_{i,j}^\top=N\mH\),
\begin{align}
    \sum_{i,j}
    \E\left[
        \bignorm{\vy_{i,j}}_2^2
        \,\middle|\,
        \{(x_i,a_{i,j})\}_{i,j}
    \right]
    &\leq
    \frac{\sigma^2}{N^2}
    \tr\left(
        \mH^{-1}\sum_{i,j}
        \vz_{i,j}\vz_{i,j}^\top
    \right)
    =
    \frac{\sigma^2d}{N},
    \label{eq:linear-vr-noise-second-moment}
\end{align}
where we denote \(\vz_{i,j}\coloneq\vz(x_i,a_{i,j})\).
Also, using \(|\eta_{i,j}|\leq2{\color{red}\Rmax}\),
Assumption~\ref{asm:linear-vr-design}, and
Eqn.~\eqref{eq:linear-vr-lower-isometry},
\begin{equation}
    \bignorm{\vy_{i,j}}_2
    \leq
    \frac{
        2{\color{red}\Rmax}{\color{blue}\alpha_0}\sqrt{3d}
    }{N}.
\end{equation}
A conditional application of
Eqn.~\eqref{eq:linear-vr-vector-bernstein} therefore yields, with
conditional probability at least \(1-\delta/3\),
\begin{equation}
    \bignorm{\vt_\eta}_{\mH^{-1}}
    \leq
    \underbrace{
    \sigma\sqrt{\frac dN}
       \left(1+\sqrt{8t_\delta}\right)
    +
    \frac{
        8{\color{red}\Rmax}{\color{blue}\alpha_0}\sqrt{3d}\,t_\delta
    }{3N}
    }_{\eqqcolon B_\eta}.
    \label{eq:linear-vr-noise-score}
\end{equation}
Here
\(\bignorm{\vv}_{\mH^{-1}}\coloneq\sqrt{\vv^\top \mH^{-1}\vv}\).
The conditional failure probability remains at most
\(\delta/3\) after averaging over the design. Hence, by a union
bound, the three events above hold simultaneously with probability
at least \(1-\delta\).

\paragraph{Step 4: combining the three events.}
Let
\(D\coloneq\bignorm{\vt_\eta-\vt_{\rm mis}}_{\mH^{-1}}\).
Cauchy--Schwarz and AM-GM inequality gives
\begin{equation}
    2\vw^\top(\vt_\eta-\vt_{\rm mis})
    \leq
    2\bignorm{\vw}_\mH \bignorm{\vt_\eta-\vt_{\rm mis}}_{\mH^{-1}}
    \leq
    \bignorm{\vw}_\mH^2 + D^2.
\end{equation}
Together with Eqn.~\eqref{eq:linear-vr-modulus-process} and \eqref{eq:linear-vr-lower-isometry}, this yields
\begin{align}
    \Delta(h;f^\circ_{\VR})
    -
    \widehat\Delta(h;f^\circ_{\VR})
    -
    \widehat Q^{\VR}(h)
    &\leq
    \bignorm{\vw}_2^2 - \bignorm{\vw}_\mH^2 + D^2
    \notag\\
    &\leq
    \frac23\Delta(h;f^\circ_{\VR})+D^2.
    \label{eq:linear-vr-modulus-bound}
\end{align}
On the same event,
\begin{equation}
    D
    \leq
    \sqrt3\bignorm{\vt_{\rm mis}}_2
    +
    \bignorm{\vt_\eta}_{\mH^{-1}}
    \leq
    \sqrt3 B_{\rm mis}+B_\eta
    =
    \frac{\Gamma_\delta}{\sqrt3},
\end{equation}
where the last equality follows from
Eqn.~\eqref{eq:vr-linear-Gamma}. Thus
\(D^2\leq\Gamma_\delta^2/3\).
Restricting Eqn.~\eqref{eq:linear-vr-modulus-bound} to
\(\Delta(h;f^\circ_{\VR})\leq r\) and taking the supremum in the
definition of the localized block-offset modulus gives
\begin{equation}
    \Psi(r;1,f^\circ_{\VR})
    \leq
    \frac23r+\frac13\Gamma_\delta^2,
    \qquad
    \forall r\geq0.
\end{equation}
This proves the claim.
\qed

\subsubsection{\texorpdfstring{Alternate Error Bound for Linear \VR}{Alternate Error Bound for Linear VR}}
\label{app:linear-vr-alternate-error-bound}

Here, we present an alternate error bound for \VR in $\gF_{\rm lin}$.
To do so, we introduce additional notation regarding centered feature geometry.

With $\overline\vphi(x,a)\coloneq\vphi(x,a)-\E_{a\sim\piref(\cdot\mid x)}[\vphi(x,a)]$, define the following quantities (some of which appear in \Cref{subsec:linear_vdr}):
\begin{align}
    \mSigma_{\rm C}
    &\coloneq
    \E_{x,a}\left[\overline\vphi(x,a) \overline\vphi(x,a)^\top\right],
    &
    \mM_{\rm C}
    &\coloneq
    \mSigma^{-1/2}\mSigma_{\rm C}\mSigma^{-1/2},
    \notag\\
    {\color{blue}\kappa_{\rm C}}
    &\coloneq
    \bignormop{\mM_{\rm C}},
    &
    {\color{blue}d_{\rm eff}}
    &\coloneq
    \tr(\mM_{\rm C}).
    \label{eq:vr-linear-centered-geometry}
\end{align}
Then we have that $\vzero\preceq\mM_{\rm C}\preceq\mI_d$, and hence, $0\leq{\color{blue}\kappa_{\rm C}}\leq1$ and $0\leq{\color{blue}d_{\rm eff}}\leq d$.
Also define the centered bias
\begin{equation}
    {\color{red}\varepsilon_{\VR,C}}
    \coloneq
    \E_{x,a}\left[
        \bigl(\gC f^\circ_{\VR}-\gC r^\star\bigr)^2
    \right]
    =
    \E_X\left[
        \Var_A\bigl({\color{red}e^\circ}(X,A)\mid X\bigr)
    \right]
    \leq
    {\color{red}\varepsilon_{\rm aprx}}.
    \label{eq:vr-linear-centered-bias}
\end{equation}
For notational convenience, let
\begin{align}
\Gamma_{\delta,{\rm C}}
\coloneq{}&
3\sqrt{
    \frac{
        {\color{blue}\kappa_{\rm C}}\,{\color{red}\varepsilon_{\mathrm{blk}}} d
    }{N_x}
}
\left(1+\sqrt{8t_\delta}\right)
+\frac{
    4\sqrt{{\color{blue}\kappa_{\rm C}}{\color{red}\varepsilon_{\mathrm{z}}}d}\,t_\delta
}{N_x}
+
\sqrt{\frac{3\sigma^2{\color{blue}d_{\rm eff}}}{N}}
\left(1+\sqrt{8t_\delta}\right)
+\frac{
    8{\color{red}\Rmax}{\color{blue}\alpha_0}
    \sqrt{{\color{blue}\kappa_{\rm C}}d}\,t_\delta
}{N}
\label{eq:vr-linear-centered-Gamma}
\end{align}
and
\begin{equation}
    \mathfrak R_{N,{\rm C}}^{\VR}
    \coloneq
    \frac{{\color{blue}\kappa_{\rm C}}{\color{red}\varepsilon_{\mathrm{x}}}d}{N_x}
    +\frac{{\color{blue}\kappa_{\rm C}}{\color{red}\varepsilon_{\mathrm{z}}}d}{N_x^2}
    +\frac{{\color{blue}\kappa_{\rm C}}{\color{red}\varepsilon_{\mathrm{a}}}d+\sigma^2{\color{blue}d_{\rm eff}}}{N}
    +\frac{{\color{red}\Rmax^2}{\color{blue}\alpha_0^2}{\color{blue}\kappa_{\rm C}}d}{N^2}.
    \label{eq:vr-linear-centered-estimation-rate}
\end{equation}

We now state the alternate error bound:
\begin{theorem}[Alternate Error Bound for Linear \VR]
\label{thm:linear-vr-alternate-error-bound}
Suppose that the assumptions of \Cref{thm:vr-linear} hold.
Then, with probability at least $1-\delta$,
\begin{align}
    \E_{x,a}\left[
        \bigl(\gC\hat f_N^{\VR}-\gC r^\star\bigr)^2
    \right]
    &\leq
    \left(
        \sqrt{{\color{red}\varepsilon_{\VR,C}}}
        +\Gamma_{\delta,{\rm C}}
    \right)^2
    \notag\\
    &\leq
    {\color{red}\varepsilon_{\VR,C}}
    +\widetilde{\gO}\left(
        \sqrt{
            {\color{red}\varepsilon_{\VR,C}}
            \mathfrak R_{N,{\rm C}}^{\VR}
        }
        +\mathfrak R_{N,{\rm C}}^{\VR}
    \right).
    \label{eq:linear-vr-centered-refinement}
\end{align}
\end{theorem}

\paragraph{Discussion.}
We compare the above with the original error bound of \VR (\Cref{thm:vr-linear}) and the error bound of \VDR (\Cref{thm:vdr-linear}).
First, in the well-specified case (${\color{red}\varepsilon_{\rm aprx}}=0$), the alternate error bound of \VR is
$\widetilde{\gO}\left(
\frac{\sigma^2{\color{blue}d_{\rm eff}}}{N}
+
\frac{{\color{red}\Rmax^2}{\color{blue}\alpha_0^2}{\color{blue}\kappa_{\rm C}}d}{N^2}
\right)$,
while the original error bound is
$\widetilde{\gO}\left(
\frac{\sigma^2d}{N}
+
\frac{{\color{red}\Rmax^2}{\color{blue}\alpha_0^2}d}{N^2}
\right)$
and the error bound of \VDR is
$\widetilde{\gO}\left(
\frac{\sigma^2{\color{blue}d_{\rm C}}}{N}
+
\frac{{\color{red}\Rmax^2}{\color{blue}\alpha_{\rm C}^2}{\color{blue}d_{\rm C}}}{N^2}
\right)$.
As ${\color{blue}d_{\rm eff}}\leq d$ and ${\color{blue}\kappa_{\rm C}}\leq1$, the alternate error bound is tighter than the original error bound, and the gap can be arbitrarily large.
For example, let $X,A\in\{-1,+1\}^d$ be independent Rademacher vectors and define
$\vphi(X,A)=\sqrt{1-\tau}\,X+\sqrt{\tau}\,A$ for some $\tau\in(0,1]$.
Then $\mSigma=\mI_d$ and $\mSigma_{\rm C}=\tau\mI_d$, so that
${\color{blue}d_{\rm eff}}=\tau d$ and ${\color{blue}\kappa_{\rm C}}=\tau$, while ${\color{blue}\alpha_0}=\gO(1)$.
Thus, the alternate error bound improves both terms of the original error bound by a factor of $\tau$.

Even compared with \VDR, the leading $N^{-1}$ term of the alternate error bound is no larger than the leading $N^{-1}$ term of the \VDR bound, since
${\color{blue}d_{\rm eff}}\leq{\color{blue}d_{\rm C}}$.\footnote{Indeed, $0\preceq\mM_{\rm C}\preceq\mI_d$ and $\mSigma\succ\vzero$, and hence
${\color{blue}d_{\rm eff}}=\tr(\mM_{\rm C})
\leq\operatorname{rank}(\mM_{\rm C})
=\operatorname{rank}(\mSigma_{\rm C})
={\color{blue}d_{\rm C}}$.
Moreover,
$\sigma^2d_{\rm C}/(N_x(N_a-1))
=\frac{N_a}{N_a-1}\sigma^2d_{\rm C}/N$,
where $N_a/(N_a-1)\leq2$ for $N_a\geq2$.}
This comparison is strict whenever ${\color{blue}d_{\rm eff}}<{\color{blue}d_{\rm C}}$.
On the other hand, the lower-order $N^{-2}$ terms of the alternate \VR and \VDR bounds are not uniformly ordered, since ${\color{blue}\alpha_0}$ and ${\color{blue}\alpha_{\rm C}}$ are not comparable in general.

There may also be cases in which
$0={\color{red}\varepsilon_{\VR,C}}<{\color{red}\varepsilon_{\rm aprx}}$, in which case the gap is even larger in favor of the alternate error bound.
For example, let $X$ and $A$ be independent Rademacher random variables, let $\vphi(X,A)=A$, and let $r^\star(X,A)=X$.
This satisfies weak realizability with $f^\star=0$ and $b^\star(X)=X$.
Moreover, the population \VR projection is $f^\circ_{\VR}=0$, and hence ${\color{red}e^\circ}(X,A)=-X$.
It follows that
${\color{red}\varepsilon_{\rm aprx}}=1$ but
${\color{red}\varepsilon_{\VR,C}}=0$, since ${\color{red}e^\circ}$ is action-independent.
In addition,
${\color{red}\varepsilon_{\mathrm{x}}}=0$ and
${\color{red}\varepsilon_{\mathrm{a}}}={\color{red}\varepsilon_{\mathrm{z}}}=1$.
Thus, in the noiseless case, the alternate bound is
$\widetilde{\gO}(N^{-1}+N_x^{-2})$, whereas the original bound contains the constant approximation-error term ${\color{red}\varepsilon_{\rm aprx}}=1$.

Finally, when ${\color{red}\varepsilon_{\VR,C}}>0$ is constant, the square-root interaction in the alternate bound may yield a slower estimation remainder.
For example, if
${\color{blue}\kappa_{\rm C}}{\color{red}\varepsilon_{\mathrm{x}}}d/N_x$
is the dominant term in $\mathfrak R_{N,{\rm C}}^{\VR}$, then the alternate bound scales, ignoring lower-order terms, as ${\color{red}\varepsilon_{\VR,C}}
+
\widetilde{\gO}\left(
\sqrt{
\frac{
{\color{blue}\kappa_{\rm C}}
{\color{red}\varepsilon_{\VR,C}}
{\color{red}\varepsilon_{\mathrm{x}}}d
}{N_x}
}
\right),$
whereas the original bound scales as ${\color{red}\varepsilon_{\rm aprx}}
+
\widetilde{\gO}\left(
\frac{{\color{red}\varepsilon_{\mathrm{x}}}d}{N_x}
\right).$
When
${\color{red}\varepsilon_{\VR,C}}
<{\color{red}\varepsilon_{\rm aprx}}$,
however, neither bound uniformly dominates: the smaller centered approximation error of the alternate bound must be compared with its slower square-root remainder.

We now present the proof of the alternate error bound:
\begin{proof}[Proof of \Cref{thm:linear-vr-alternate-error-bound}]
We retain the notation from the proof of \Cref{lem:linear-vr-modulus}.
By the arguments in \textbf{Steps~1--2} of that proof, Eqn.~\eqref{eq:linear-vr-lower-isometry} and Eqn.~\eqref{eq:linear-vr-mis-score}
hold simultaneously with probability at least $1-2\delta/3$.
On the lower-isometry event, condition on the complete context--action design and define
\begin{equation}
    \vy^{\rm C}_{i,j}
    \coloneq
    \frac1N
    \mM_{\rm C}^{1/2}\mH^{-1}
    \vz(x_i,a_{i,j})\eta_{i,j}.
\end{equation}
These vectors are conditionally independent and mean zero with conditional second moments bounded as
\begin{align}
    \sum_{i,j}
    \E\left[
        \bignorm{\vy^{\rm C}_{i,j}}_2^2
        \,\middle|\,
        \{(x_i,a_{i,j})\}_{i,j}
    \right]
    &\leq
    \frac{\sigma^2}{N^2}
    \tr\left(
        \mM_{\rm C}^{1/2}\mH^{-1}
        \left(\sum_{i,j}\vz_{i,j}\vz_{i,j}^\top\right)
        \mH^{-1}\mM_{\rm C}^{1/2}
    \right)
    \notag\\
    &=
    \frac{\sigma^2}{N}\tr(\mH^{-1}\mM_{\rm C})
    \leq
    \frac{3\sigma^2{\color{blue}d_{\rm eff}}}{N}.
    \label{eq:linear-vr-centered-noise-second-moment}
\end{align}
Moreover,
\begin{equation}
    \bignorm{\vy^{\rm C}_{i,j}}_2
    \leq
    \frac{
        6{\color{red}\Rmax}{\color{blue}\alpha_0}
        \sqrt{{\color{blue}\kappa_{\rm C}}d}
    }{N},
\end{equation}
where we used $\bignormop{\mH^{-1}}\leq3$,
$\bignormop{\mM_{\rm C}}=\kappa_{\rm C}$, and
$\bignorm{\vz}_2\leq{\color{blue}\alpha_0}\sqrt d$.
Thus, a conditional application of vector Bernstein's inequality (Eqn.~\eqref{eq:linear-vr-vector-bernstein}) gives, with conditional probability at least $1-\delta/3$,
\begin{equation}
    \bignorm{
        \mM_{\rm C}^{1/2}\mH^{-1}\vt_\eta
    }_2
    \leq
    \underbrace{
        \sqrt{\frac{3\sigma^2{\color{blue}d_{\rm eff}}}{N}}
        \left(1+\sqrt{8t_\delta}\right)
        +
        \frac{
            8{\color{red}\Rmax}{\color{blue}\alpha_0}
            \sqrt{{\color{blue}\kappa_{\rm C}}d}\,t_\delta
        }{N}
    }_{\eqqcolon B_{\eta,{\rm C}}}.
    \label{eq:linear-vr-centered-noise-score}
\end{equation}
The conditional failure probability remains at most $\delta/3$ after averaging over the design.
Hence, the three events above hold simultaneously with probability at least $1-\delta$.

Similarly, the empirical normal equation and Eqn.~\eqref{eq:linear-vr-scores} give
\begin{equation}
    \widehat \vw
    \coloneq
    \mSigma^{1/2}
    \bigl(\hat{\vtheta}_N^{\VR}-\vtheta^\circ\bigr) \Longrightarrow
    \mH\widehat \vw
    =
    \vt_\eta-\vt_{\rm mis}.
    \label{eq:linear-vr-empirical-normal-equation}
\end{equation}
Hence, by Eqn.~\eqref{eq:linear-vr-lower-isometry},~\eqref{eq:linear-vr-mis-score},~\eqref{eq:linear-vr-centered-noise-score},
\begin{align}
    \E_{x,a}\left[
        \bigl(\gC\hat f_N^{\VR}-\gC f^\circ_{\VR}\bigr)^2
    \right]^{1/2}
    &=
    \bignorm{\mM_{\rm C}^{1/2}\widehat \vw}_2
    \notag\\
    &\leq
    \bignorm{
        \mM_{\rm C}^{1/2}\mH^{-1}\vt_{\rm mis}
    }_2
    +
    \bignorm{
        \mM_{\rm C}^{1/2}\mH^{-1}\vt_\eta
    }_2
    \notag\\
    &\leq
    3\sqrt{{\color{blue}\kappa_{\rm C}}}\,B_{\rm mis}
    +B_{\eta,{\rm C}}
    =
    \Gamma_{\delta,{\rm C}}.
    \label{eq:linear-vr-centered-estimation}
\end{align}
The triangle inequality, Eqn.~\eqref{eq:vr-linear-centered-bias}, and Eqn.~\eqref{eq:linear-vr-centered-estimation} imply
\begin{align}
    \E_{x,a}\left[
        \bigl(\gC\hat f_N^{\VR}-\gC r^\star\bigr)^2
    \right]^{1/2}
    &\leq \E_{x,a}\left[
        \bigl(\gC\hat f_N^{\VR}-\gC f^\circ_{\VR}\bigr)^2
    \right]^{1/2} + \E_{x,a}\left[
        \bigl(\gC f^\circ_{\VR}-\gC r^\star\bigr)^2
    \right]^{1/2} \\
    &\leq
    \sqrt{{\color{red}\varepsilon_{\VR,C}}}
    +\Gamma_{\delta,{\rm C}}.
\end{align}
Finally, Eqns.~\eqref{eq:vr-linear-centered-Gamma}--\eqref{eq:vr-linear-centered-estimation-rate}
give $\Gamma_{\delta,{\rm C}}^2=\widetilde{\gO}(\mathfrak R_{N,{\rm C}}^{\VR})$.
This proves \Cref{thm:linear-vr-alternate-error-bound}.
\end{proof}

\subsection{\texorpdfstring{\ValueDifferenceRegression{} (\VDR)}{Value Difference Regression (VDR)}}

\subsubsection{\texorpdfstring{Proof of \Cref{thm:vdr-linear}: \VDR Bound for Linear $\gF_{\rm lin}$}{Proof of Theorem 4.2: VDR Bound for a Linear Function Class}}
\label{app:proof-thm-linear-vdr}

Throughout this subsection, $\gL^{\VDR}$ and $\widehat{\gL}^{\VDR}$
denote the generic losses in \Cref{sec:localised} instantiated with
$\gT=\gT^{\VDR}$, and the comparator is $f^\circ_{\VDR}=f^\star$.
For $h=f-f^\star$, let $h_{i,j}\coloneq h(x_i,a_{i,j})$. Under
\Cref{asm:weak_realizability}, the action-independent nuisance cancels
from every within-context difference, so
\begin{align}
    \Delta(h;f^\star)
    &\coloneq
    \gL^{\VDR}(f^\star+h)-\gL^{\VDR}(f^\star)
    \notag\\
    &=
    \E_{x,a,a'}\left[
        \bigl(h(x,a)-h(x,a')\bigr)^2
    \right],
    \label{eq:linear-vdr-Delta-recall}
\end{align}
and
\begin{align}
    \widehat\Delta(h;f^\star)
    &\coloneq
    \widehat{\gL}^{\VDR}(f^\star+h)
    -\widehat{\gL}^{\VDR}(f^\star)
    \notag\\
    &=
    \frac1{N_x\binom{N_a}{2}}
    \sum_{i=1}^{N_x}\sum_{1\leq j<k\leq N_a}
    \left[
        \bigl(h_{i,j}-h_{i,k}-(\eta_{i,j}-\eta_{i,k})\bigr)^2
        -
        (\eta_{i,j}-\eta_{i,k})^2
    \right].
    \label{eq:linear-vdr-Deltahat-recall}
\end{align}
The empirical quadratic term is
\begin{equation}
    \widehat Q(h)
    =
    \widehat Q^{\VDR}(h)
    \coloneq
    \frac1{N_x\binom{N_a}{2}}
    \sum_{i=1}^{N_x}\sum_{1\leq j<k\leq N_a}
    (h_{i,j}-h_{i,k})^2.
    \label{eq:linear-vdr-Qhat-recall}
\end{equation}
Thus, in this subsection,
\begin{equation}
    \Psi(r;1,f^\star)
    =
    \sup_{\substack{
        h\in\gF_{\rm lin}-f^\star\\
        \Delta(h;f^\star)\leq r
    }}
    \left[
        \Delta(h;f^\star)
        -\widehat\Delta(h;f^\star)
        -\widehat Q^{\VDR}(h)
    \right].
    \label{eq:linear-vdr-Psi-recall}
\end{equation}

For the proof, let us denote $t_\delta\coloneq\log\frac{2}{\delta}$ and
Define
\begin{equation}
    \Gamma_\delta^{\VDR}
    \coloneq
    \sqrt{
        \frac{3\sigma^2{\color{blue}d_{\rm C}}}{N_x(N_a-1)}
    }
    \left(1+\sqrt{8t_\delta}\right)
    +
    \frac{
        8\sqrt{2}\,{\color{red}\Rmax}{\color{blue}\alpha_{\rm C}}
        \sqrt {{\color{blue}d_{\rm C}}}\,t_\delta
    }{N}.
    \label{eq:vdr-linear-Gamma}
\end{equation}

\begin{lemma}[Localized modulus for linear \VDR]
\label{lem:linear-vdr-modulus}
Suppose the assumptions of \Cref{thm:vdr-linear} hold.
Then, with probability at least $1-\delta$,
\begin{equation}
    \Psi(r;1,f^\star)
    \leq
    \frac23r+\frac23(\Gamma_\delta^{\VDR})^2,
    \qquad
    \forall r\geq0.
    \label{eq:linear-vdr-modulus}
\end{equation}
\end{lemma}
Again, we follow the same recipe.

Since $\gF_{\rm lin}$ is parametrized by $\vtheta\in\sR^d$, let
$\widehat h_{\VDR}\coloneq\hat f_N^{\VDR}-f^\star$. The normal equation
for linear \VDR gives
\begin{equation}
    \widehat\Delta(\widehat h_{\VDR};f^\star)
    =
    \widehat{\gL}^{\VDR}(\hat f_N^{\VDR})
    -\widehat{\gL}^{\VDR}(f^\star)
    =
    -\widehat Q^{\VDR}(\widehat h_{\VDR}).
\end{equation}
Thus condition~$(i)$ of \Cref{prop:localised} holds with $\lambda=1$.
By \Cref{lem:linear-vdr-modulus}, condition~$(ii)$ holds on an event of
probability at least $1-\delta$. On this event,
\begin{align}
    \gL^{\VDR}(\hat f_N^{\VDR})
    -
    \gL^{\VDR}(f^\star)
    &\leq
    \sup
    \left\{
        r\geq0:
        r\leq
        \frac23r+\frac23(\Gamma_\delta^{\VDR})^2
    \right\}
    \notag\\
    &=2(\Gamma_\delta^{\VDR})^2.
\end{align}
\Cref{lem:vdr-centering} therefore yields
\begin{equation}
    \E_{x,a}
    \left[
        \bigl(\gC\hat f_N^{\VDR}-\gC r^\star\bigr)^2
    \right]
    \leq
    (\Gamma_\delta^{\VDR})^2
    =
    \widetilde{\gO}\left(
        \frac{\sigma^2{\color{blue}d_{\rm C}}}{N_x(N_a-1)}
        +
        \frac{{\color{red}\Rmax^2}{\color{blue}\alpha_{\rm C}^2}{\color{blue}d_{\rm C}}}{N^2}
    \right).
\end{equation}
This proves \Cref{thm:vdr-linear}.
\qed

\subsubsection{\texorpdfstring{Proof of \Cref{lem:linear-vdr-modulus}}{Proof of Lemma F.2}}
\label{app:proof-lem-linear-vdr-modulus}

Under \Cref{asm:weak_realizability}, let us write $f^\star(x,a)
=
\langle\vphi(x,a),\vtheta^\star\rangle$ for some $\vtheta^\star\in\sR^d$.
We first recall the centered features as defined in \Cref{subsec:linear_vdr}.

The centered feature is $\overline\vphi(x,a)
\coloneq
\vphi(x,a)-\E_{a\sim\piref(\cdot\mid x)}[\vphi(x,a)]$, and its covariance is $\mSigma_{\rm C}
\coloneq
\E_{x,a}\left[\overline\vphi(x,a) \overline\vphi(x,a)^\top\right].$
Let ${\color{blue}d_{\rm C}}\coloneq\operatorname{rank}(\mSigma_{\rm C})$.
Whenever ${\color{blue}d_{\rm C}}\geq1$, let $\mU_{\rm C}\in\sR^{d\times {\color{blue}d_{\rm C}}}$ satisfy
$\mU_{\rm C}^\top\mU_{\rm C}=\mI_{{\color{blue}d_{\rm C}}}$ and
$\operatorname{col}(\mU_{\rm C})=\operatorname{range}(\mSigma_{\rm C})$, and define
\begin{equation}
    \overline{\mSigma}_{\rm C}
    \coloneq
    \mU_{\rm C}^\top\mSigma_{\rm C}\mU_{\rm C}
    \succ\vzero,
    \qquad
    \vz_{\rm C}(x,a)
    \coloneq
    \overline{\mSigma}_{\rm C}^{-1/2}
    \mU_{\rm C}^\top\overline\vphi(x,a)
    \in\sR^{{\color{blue}d_{\rm C}}}.
\end{equation}
Then, $\E[\vz_{\rm C}\vz_{\rm C}^\top]=\mI_{{\color{blue}d_{\rm C}}}$ and
$\E_A[\vz_{\rm C}(X,A)\mid X]=\vzero$.

For each observed block $i\in[N_x]$, define
\begin{equation}
    {\overline{\vz}}_{i,\rm C}
    \coloneq
    \frac1{N_a}\sum_{j=1}^{N_a}\vz_{\rm C}(x_i,a_{i,j}), \quad
    \widetilde{\vz}_{i,j}
    \coloneq
    \frac1{\sqrt2}\left(
        \vz_{\rm C}(x_i,a_{i,j})-{\overline{\vz}}_{i,\rm C}
    \right), \quad
    \mG_i^{\rm C}
    \coloneq
    \frac2{N_a-1}
    \sum_{j=1}^{N_a}
    \widetilde{\vz}_{i,j}
    \widetilde{\vz}_{i,j}^\top,
\label{eq:vdr-linear-centered-feature}
\end{equation}
and
\begin{equation}
    \mH
    \coloneq
    \frac1{N_x}
    \sum_{i=1}^{N_x}\mG_i^{\rm C}, \quad
    \vs
    \coloneq
    \frac2{N_x(N_a-1)}
    \sum_{i=1}^{N_x}
    \sum_{j=1}^{N_a}
    \widetilde{\vz}_{i,j}\eta_{i,j}.
    \label{eq:vdr-linear-H}
\end{equation}

Using the empirical Gram matrix and score in Eqn.~\eqref{eq:vdr-linear-H}, we now
express the quantities recalled in
Eqns.~\eqref{eq:linear-vdr-Delta-recall}--\eqref{eq:linear-vdr-Qhat-recall}
in the identifiable centered-feature coordinates. Fix any
$h=f-f^\star$, where
\begin{align}
    f(x,a)
    &=
    \langle\vphi(x,a),\vtheta\rangle,
    &
    \vw
    &\coloneq
    \sqrt2\,\overline{\mSigma}_{\rm C}^{1/2}
    \mU_{\rm C}^\top(\vtheta-\vtheta^\star)
    \in\sR^{{\color{blue}d_{\rm C}}}.
\end{align}
For every $\vv\in\ker(\mSigma_{\rm C})$,
$\E[(\vv^\top\gC\vphi)^2]=0$; hence
$\overline\vphi\in\operatorname{range}(\mSigma_{\rm C})$ almost surely.
It follows that
\begin{equation}
    \bigl(\vphi(x,a)-\vphi(x,a')\bigr)^\top
    (\vtheta-\vtheta^\star)
    =
    \frac1{\sqrt2}
    \bigl(\vz_{\rm C}(x,a)-\vz_{\rm C}(x,a')\bigr)^\top\vw
    \qquad\text{almost surely}.
    \label{eq:vdr-linear-reduced-coordinate}
\end{equation}
By \Cref{lem:vdr-centering},
\begin{equation}
    \Delta(h;f^\star)
    =
    \E_{x,a,a'}
    \left[
        \left(h(x,a)-h(x,a')\right)^2
    \right]
    =
    \bignorm{\vw}_2^2.
    \label{eq:vdr-linear-Delta}
\end{equation}
Moreover,
\begin{equation}
    \widehat Q^{\VDR}(h)
    =
    \vw^\top\mH\vw.
    \label{eq:vdr-linear-Qhat}
\end{equation}
Since the action-independent nuisance $b^\star(x_i)$ cancels from every
within-context difference, direct expansion of the empirical \VDR excess loss gives
\begin{equation}
    \widehat\Delta(h;f^\star)
    =
    \vw^\top\mH\vw
    -
    2\vw^\top\vs.
    \label{eq:vdr-linear-empirical-Delta}
\end{equation}
Therefore, the process to be controlled is
\begin{equation}
    \Delta(h;f^\star)
    -
    \widehat\Delta(h;f^\star)
    -
    \widehat Q^{\VDR}(h) =
    \bignorm{\vw}_2^2
    -
    2\vw^\top\mH\vw
    +
    2\vw^\top\vs.
    \label{eq:vdr-linear-modulus-process}
\end{equation}
It is therefore enough to show, uniformly over $h$, that the
right-hand side of Eqn.~\eqref{eq:vdr-linear-modulus-process} is at most
$2\Delta(h;f^\star)/3+2(\Gamma_\delta^{\VDR})^2/3$; restricting this
bound to $\Delta(h;f^\star)\leq r$ concludes the proof.

Similarly, we prove the uniform bound through a lower-isometry event for $\mH$ and concentration of $\vs$ in the corresponding empirical norm.

\paragraph{Step 1: lower isometry of the empirical Gram matrix.}
The matrices $\mG_i^{\rm C}$ are i.i.d. copies of
$\mG^{\rm C}_{N_a}$ and satisfy $\E[\mG_i^{\rm C}]=\mI_{{\color{blue}d_{\rm C}}}$.
Moreover, $\mG_i^{\rm C}\succeq\vzero$, and hence
$\lambda_{\max}(\mI_{{\color{blue}d_{\rm C}}}-\mG_i^{\rm C})\leq1$.

The one-sided matrix Bernstein inequality
\citep[Theorem 6.6.1]{tropp2015survey} therefore gives, with probability
at least $1-\delta/2$,
\begin{align}
    \lambda_{\max}(\mI_{{\color{blue}d_{\rm C}}}-\mH)
    &\leq
    \sqrt{
        \frac{
            2\bignormop{\Cov(\mG^{\rm C}_{N_a})}
            \log(2{\color{blue}d_{\rm C}}/\delta)
        }{N_x}
    }
    +
    \frac{2\log(2{\color{blue}d_{\rm C}}/\delta)}{3N_x}
    \notag\\
    &\leq\frac23,
\end{align}
where the last step uses the sample-size condition in \Cref{thm:vdr-linear}. Thus,
\begin{equation}
    \mH\succeq\frac13\mI_{{\color{blue}d_{\rm C}}}.
    \label{eq:vdr-linear-lower-isometry}
\end{equation}

\paragraph{Step 2: control of the reward-noise score.}
On the event in Eqn.~\eqref{eq:vdr-linear-lower-isometry}, condition on
the complete context--action design $\{(x_i,a_{i,j})\}_{i,j}$.
Then, the vectors
\begin{equation}
    \vy_{i,j}
    \coloneq
    \frac2{N_x(N_a-1)}
    \mH^{-1/2}\widetilde{\vz}_{i,j}\eta_{i,j}
\end{equation}
are conditionally independent and mean zero. By Eqn.~\eqref{eq:vdr-linear-centered-feature},
\begin{equation}
    \sum_{i,j}
    \E\left[
        \bignorm{\vy_{i,j}}_2^2
        \,\middle|\,
        \{(x_i,a_{i,j})\}_{i,j}
    \right]
    \leq
    \frac{4\sigma^2}{N_x^2(N_a-1)^2}
    \tr\left(
        \mH^{-1}
        \sum_{i,j}
        \widetilde{\vz}_{i,j}
        \widetilde{\vz}_{i,j}^\top
    \right)
    =
    \frac{2\sigma^2{\color{blue}d_{\rm C}}}{N_x(N_a-1)}.
    \label{eq:vdr-linear-score-variance}
\end{equation}
On the other hand, Eqn.~\eqref{eq:vdr-linear-centered-feature} gives
\begin{equation}
    \widetilde{\vz}_{i,j}
    =
    \frac1{N_a}\sum_{k\neq j}
    \frac{\vz_{\rm C}(x_i,a_{i,j})-\vz_{\rm C}(x_i,a_{i,k})}{\sqrt2}.
\end{equation}
Thus, the triangle inequality and \Cref{asm:linear-vdr-design} imply
\begin{equation}
    \bignorm{\widetilde{\vz}_{i,j}}_2
    \leq
    \frac{N_a-1}{N_a}
    {\color{blue}\alpha_{\rm C}}\sqrt {{\color{blue}d_{\rm C}}}.
\end{equation}
Using $|\eta_{i,j}|\leq2{\color{red}\Rmax}$ and
$\bignormop{\mH^{-1/2}}\leq\sqrt3$ on
Eqn.~\eqref{eq:vdr-linear-lower-isometry}, we obtain
\begin{equation}
    \bignorm{\vy_{i,j}}_2
    \leq
    \frac{
        4\sqrt3\,{\color{red}\Rmax}{\color{blue}\alpha_{\rm C}}\sqrt {{\color{blue}d_{\rm C}}}
    }{N_xN_a}
    =
    \frac{
        4\sqrt3\,{\color{red}\Rmax}{\color{blue}\alpha_{\rm C}}\sqrt {{\color{blue}d_{\rm C}}}
    }{N}.
\end{equation}
Applying the vector Bernstein inequality (Eqn.~\eqref{eq:linear-vr-vector-bernstein}) conditionally on the design
gives, with conditional probability at least $1-\delta/2$,
\begin{align}
    \bignorm{\vs}_{\mH^{-1}}
    \leq
    \sqrt{
        \frac{
            2\sigma^2{\color{blue}d_{\rm C}}
        }{
            N_x(N_a-1)
        }
    }
    \left(1+\sqrt{8t_\delta}\right)
    +
    \frac{
        16\sqrt3\,{\color{red}\Rmax}{\color{blue}\alpha_{\rm C}}
        \sqrt {{\color{blue}d_{\rm C}}}\,t_\delta
    }{
        3N
    }
    =
    \sqrt{\frac23}\,
    \Gamma_\delta^{\VDR}.
    \label{eq:vdr-linear-D}
\end{align}
The conditional failure probability remains at most $\delta/2$ after
averaging over the design. Hence,
Eqns.~\eqref{eq:vdr-linear-lower-isometry}--\eqref{eq:vdr-linear-D}
hold simultaneously with probability at
least $1-\delta$.

\paragraph{Step 3: combining the two events.}
Let $D\coloneq\bignorm{\vs}_{\mH^{-1}}$. By Cauchy--Schwarz and AM-GM inequalities,
\begin{equation}
    2\vw^\top\vs
    \leq
    2\bignorm{\vw}_\mH \bignorm{\vs}_{\mH^{-1}}
    \leq
    \bignorm{\vw}_\mH^2+D^2.
\end{equation}
Hence, using
$\mH\succeq\mI_{{\color{blue}d_{\rm C}}}/3$ and
Eqn.~\eqref{eq:vdr-linear-Delta},
\begin{align}
    \Delta(h;f^\star)
    -
    \widehat\Delta(h;f^\star)
    -
    \widehat Q^{\VDR}(h)
    &\leq
    \bignorm{\vw}_2^2
    -
    \vw^\top\mH\vw
    +
    D^2
    \notag\\
    &\leq
    \frac23\Delta(h;f^\star)
    +
    D^2
    \notag\\
    &\leq
    \frac23\Delta(h;f^\star)
    +
    \frac23(\Gamma_\delta^{\VDR})^2.
\end{align}
Restricting to
$\Delta(h;f^\star)\leq r$ and taking the supremum in the definition
of $\Psi(r;1,f^\star)$ yields
\begin{equation}
    \Psi(r;1,f^\star)
    \leq
    \frac23r+\frac23(\Gamma_\delta^{\VDR})^2,
    \qquad
    \forall r\geq0.
\end{equation}
This proves the claim.
\qed

%% file: 906LowerBounds.tex
\section{\texorpdfstring{Lower Bounds}{Lower Bounds}}
\label{app:lower-bound}
We first give an estimator-specific separation:
under Assumption~\ref{asm:weak_realizability}, \VR can incur constant regret while \VDR followed by greedy selection succeeds on the same instance.
We then remove Assumption~\ref{asm:weak_realizability} and establish an agnostic minimax lower bound for arbitrary learners.
Appendix~\ref{app:vr_failure} proves the separation; Appendix~\ref{app:agnostic} states the minimax result along with a comparison to prior lower bounds.
Lastly, Appendix~\ref{app:agnostic_proof} proves it.

Throughout, $N = N_x N_a$ is the total number of observations in the grouped dataset.
Unless otherwise specified, $\E_{x,a}$ denotes $\E_{x\sim\rho,a\sim\piref(\cdot\mid x)}$.
All policy values and regrets are evaluated using the true reward $r^\star$.

\subsection{\texorpdfstring{Proof of \Cref{thm:vr_failure}: Failure Mode of \VR}{Proof of Theorem 3.2: Failure Mode of VR}}
\label{app:vr_failure}

The proof first identifies a predictor $g$ that fits absolute rewards better than $f^\star$ but reverses the true action ordering.
We then show that empirical \VR selects $g$ with high probability by concentrating over the $N_x$ independent context blocks.
On that event, greedy selection chooses the suboptimal action.
Finally, we show that \VDR followed by greedy selection succeeds on the same instance when $N_a \geq 2$.

\paragraph{Step 1: construction and population failure.}

First, recall that the approximation error is defined as:
\[
    \varepsilon_{\rm aprx} \coloneq  \inf_{f \in \gF} \E_{x \sim \rho, a \sim \piref(\cdot \mid x)}
    \left[(f(x, a) - r^\star(x, a))^2\right].
\]
While $\varepsilon_{\rm aprx}$ measures how well the model class approximates the \emph{absolute reward values},
greedy decision making depends only on the \emph{relative ordering of actions} within each context.
Consider the weak realizability model
\[
    r^\star(x,a) = f^\star(x,a) + b^\star(x)
\]
and choose noiseless observations, so $r_{i,j} = r^\star(x_i,a_{i,j})$ and $\eta_{i,j} = 0$.
The action-independent baseline $b^\star(x)$ does not itself change the true action ordering.
Nevertheless, ordinary \ValueRegression can still fail when three ingredients are present: (1) the baseline varies across contexts, (2) the behavior policy correlates actions with contexts, and (3) the value class cannot represent the context-dependent baseline.
In this case, different actions are observed under different mixtures of contexts, so the baseline can be absorbed into the fitted action values and reverse their ordering.

To make this precise, consider two contexts $\mathcal{X} = \{x_1, x_2\}$, with $\rho(x_1) = \rho(x_2) = \tfrac{1}{2}$, and two actions $\mathcal{A} = \{a_1, a_2\}$.
Fix a baseline magnitude $B > 0$ and an action gap $\Delta > 0$, to be specified below, and define
\[
    f^\star(x, a_1) = \Delta, \quad
    f^\star(x, a_2) = 0, \quad
    \forall x \in \mathcal{X},
\]
and let
\[
    b^\star(x_1) = B, \quad
    b^\star(x_2) = 0.
\]
Now one can readily see that
\(
    r^\star(x, a) = f^\star(x, a) + b^\star(x),
\)
and $a_1$ is the optimal action at both contexts, therefore
\[
    r^\star(x, a_1) - r^\star(x, a_2) = \Delta.
\]
Now choose a behavior policy that correlates actions with contexts:
\[
    \piref(a_i \mid x_j) \coloneq  p\1\{i = j\} + (1 - p)\1\{i \neq j\}, \quad p \in (0, 1/2). 
\]
Here one can observe that the policy makes the two actions appear in systematically different contexts.
In particular,
\[
\begin{aligned}
    \piref(a_1 \mid x_1) &= p,
    & \piref(a_1 \mid x_2) &= 1 - p \\
    \piref(a_2 \mid x_1) &= 1 - p,
    & \piref(a_2 \mid x_2) &= p.
\end{aligned}
\]
Since $\rho(x_1) = \rho(x_2) = \tfrac{1}{2}$, $\Pr(A = a_1) = \Pr(A = a_2) = \tfrac{1}{2}$, and therefore by Bayes' rule,
\[
\begin{aligned}
    \Pr(X = x_1 \mid A = a_1) &= p,
    & \Pr(X = x_2 \mid A = a_1) &= 1 - p \\
    \Pr(X = x_1 \mid A = a_2) &= 1 - p,
    & \Pr(X = x_2 \mid A = a_2) &= p.
\end{aligned}
\]
Hence the action-wise conditional mean rewards under the behavior distribution are
\begin{align}
    \E[r^\star(x, a) \mid A = a_1] &= p(B + \Delta)+(1 - p)\Delta
    = \Delta + pB, \\
    \E[r^\star(x, a) \mid A = a_2] &= (1 - p)B + p \cdot 0
    = (1 - p)B.
\end{align}
Motivated by this, consider the predictor $g$ such that
\[
    g(x, a_1) \coloneq  \Delta + pB, \quad
    g(x, a_2) \coloneq  (1 - p)B, \quad \forall x \in \mathcal{X},
\]
and take the finite value function class $\gF \coloneq  \{f^\star, g\}$.

We now compare the population squared risks of the two candidates.
To this end, define the population squared-loss objective
\[
    \gL^{\VR}(f) \coloneq  \E_{x \sim \rho, a \sim \piref(\cdot \mid x)}\left[\left(f(x, a) - r^\star(x, a)\right)^2\right],
\]
and let
\(
    f^\circ_{\VR} \in \argmin_{f\in\mathcal{F}}\gL^{\VR}(f)
\)
denote the population \ValueRegression solution.
For $g$, a direct calculation yields:
\[
    \gL^{\VR}(g) = \frac{p}{2}(1 - p)^2B^2 + \frac{1 - p}{2}p^2B^2 + \frac{1 - p}{2}p^2B^2 + \frac{p}{2}(1 - p)^2B^2
    = p(1 - p)B^2.
\]
On the other hand, since $f^\star(x,a) - r^\star(x,a) = -b^\star(x)$,
\[
    \gL^{\VR}(f^\star) = \E_x[b^\star(x)^2] = \frac{B^2}{2}.
\]
Since $p(1 - p) \leq \tfrac{1}{4} < \tfrac{1}{2}$, we have
\(
    \gL^{\VR}(g) < \gL^{\VR}(f^\star)
\)
and therefore
\[
    f^\circ_{\VR} = g, \quad \varepsilon_{\rm aprx} = p(1 - p)B^2.
\]
Furthermore, we note that
\begin{align}
    g(x, a_2) - g(x, a_1) &= (1 - p)B - (\Delta + pB) \\
    &= (1 - 2p)B - \Delta.
\end{align}
Hence, whenever $(1 - 2p)B > \Delta$, the \VR solution $g$ prefers $a_2$ at every context, while $a_1$ is optimal at every context.

For a concrete instance at the natural reward scale, take
\begin{align}
    p = \frac{1}{4}, \quad 
    B = \frac{3}{4}R_{\max}, \quad
    \Delta = \frac{1}{8}R_{\max}.
\end{align}
Then a straightforward yet tedious calculation yields:
\[
    B + \Delta = \frac{7}{8}R_{\max} \leq R_{\max},
\]
so all rewards are in $[0, R_{\max}]$, while
\[
    (1 - 2p)B = \frac{3}{8}R_{\max} > \frac{1}{8}R_{\max} = \Delta,
\]
so the population \VR chooses the wrong action.
Moreover, $\varepsilon_{\rm aprx} = p(1 - p)B^2 = \tfrac{27}{256}R_{\max}^2$, implying that it is of order $R_{\max}^2$, while the induced \Greedy policy incurs
\[
    \Reg(f^\circ_{\VR}) = \Delta = \frac{1}{8}R_{\max}.
\]

\paragraph{Step 2: empirical \VR selects the same wrong predictor.}
It remains to transfer the strict population comparison to empirical \VR.
We keep the instance and the finite class $\mathcal{F} = \{f^\star, g\}$ unchanged.
The grouped dataset contains $N = N_xN_a$ observations, but the $N_a$ observations within a block share the random context $x_i$.
We therefore apply concentration to the $N_x$ independent blocks, retaining the within-block average over all $N_a$ actions.
For block $i$, define
\[
    D_i = \frac{1}{N_a}\sum_{j=1}^{N_a}\left[\left(g(x_i,a_{i,j}) - r_{i,j}\right)^2 - \left(f^\star(x_i,a_{i,j}) - r_{i,j}\right)^2\right].
\]
Then $D_1, \ldots, D_{N_x}$ are independent and
\begin{align}
    \widehat{\gL}^{\VR}(g) - \widehat{\gL}^{\VR}(f^\star) &= \frac{1}{N_x}\sum_{i=1}^{N_x} D_i, \\
    \E[D_i] &= \gL^{\VR}(g) - \gL^{\VR}(f^\star)
    = -\frac{45}{256}R_{\max}^2.
\end{align}
The single-observation loss differences at $(x_1,a_1)$, $(x_1,a_2)$, $(x_2,a_1)$, and $(x_2,a_2)$ are, respectively, $(-63, -135, 9, 81)R_{\max}^2/256$.
Every block average therefore satisfies
\begin{align}
    -\frac{135}{256}R_{\max}^2 &\leq D_i \leq \frac{81}{256}R_{\max}^2.
\end{align}
Applying Hoeffding's inequality to the $N_x$ independent blocks gives
\begin{align}
    \Pr\left(\widehat{\gL}^{\VR}(g) \geq \widehat{\gL}^{\VR}(f^\star)\right)
    &= \Pr\left(\frac{1}{N_x}\sum_{i=1}^{N_x}D_i \geq 0\right) \notag \\
    &\leq \exp\left(-\frac{2N_x(45R_{\max}^2/256)^2}{(216R_{\max}^2/256)^2}\right) \notag \\
    &= \exp\left(-\frac{25N_x}{288}\right).
    \label{eq:gmin-vr-block-concentration}
\end{align}
Consequently, on the event
\(
    \mathcal{E} = \left\{\widehat{\gL}^{\VR}(g) < \widehat{\gL}^{\VR}(f^\star)\right\},
\)
the unique empirical \VR minimizer is $g$, regardless of the ERM tie-breaking rule.
The event has probability at least $1 - \exp(-25N_x/288)$, uniformly over $N_a \geq 1$.
On this event, greedy selection from the empirical \VR predictor chooses $a_2$ at both contexts.
Consequently,
\[
    \Reg\left(\widehat{\pi}_N^{\VRGreedy}\right) = \Delta = \frac{R_{\max}}8.
\]
In conclusion, the \VR claim of Theorem~\ref{thm:vr_failure} holds with probability at least
$1 - \delta$ whenever
\(
    N_x \geq \frac{288}{25}\log\frac{1}{\delta}.
\)

\paragraph{Step 3: \VDR succeeds on the same instance.}
For $N_a \geq 2$, one block containing both actions is enough to distinguish $f^\star$ from $g$ using reward differences.
In particular, the baseline cancels from every within-context difference, so
\(
    \widehat{\gL}^{\VDR}(f^\star) = 0.
\)
In contrast,
\begin{align}
    g(x,a_1) - g(x,a_2) &= -\frac{1}{4} R_{\max}, \\
    r^\star(x,a_1) - r^\star(x,a_2) &= \frac{1}{8} R_{\max}.
\end{align}
Therefore, if any context block contains both actions, $g$ has strictly positive \VDR loss and $f^\star$ is the unique empirical \VDR minimizer.
For either context, the probability that all $N_a$ sampled actions agree is $p^{N_a} + (1 - p)^{N_a}$.
Thus, for any ERM tie-breaking rule, independence across blocks gives
\begin{align}
    \Pr\left(\widehat{f}_N^{\VDR} \ne f^\star\right)
    \leq \left[p^{N_a} + (1 - p)^{N_a}\right]^{N_x}
    \leq \left(\frac{5}{8}\right)^{N_x}.
\end{align}
Therefore, \VDR followed by \Greedy action selection has zero regret with probability at least $1 - (5/8)^{N_x}$.
This completes the proof.

\subsection{\texorpdfstring{Agnostic Minimax Lower Bound}{Agnostic Minimax Lower Bound}}
\label{app:agnostic}

We now ask whether the approximation and statistical terms in the regret bound can be improved without Assumption~\ref{asm:weak_realizability}.
The answer is negative in a worst-case sense:
for every learner, some admissible true reward forces regret of order $\sqrt{C\varepsilon} + R_{\max}\sqrt{C\log|\mathcal{F}|/N}$ with probability bounded away from zero.
Unlike the preceding example, the learner may output any policy, not only one obtained from a predictor in $\mathcal{F}$.
We write $\Reg_{r^\star}$ when the dependence of regret on the true reward needs to be explicit.

For comparison, combining the finite-class \VR prediction bound with the \Pessimism policy guarantee in \Cref{app:offline-regret} yields, under the corresponding implementation assumptions, a regret upper bound whose leading terms are
\[
    \sqrt{C^\star \varepsilon_{\rm aprx}} + R_{\max}\sqrt{\frac{C^\star \log |\mathcal F|}{N}}
\]
when $N_a$ is bounded.
The lower bound below shows that these two dependencies are unavoidable in a worst-case sense over a coverage-budget class.

\begin{theorem}[Agnostic minimax lower bound]
\label{thm:agnostic_lower}
Let $C \geq 2$, $R_{\max} > 0$, and $d \geq 1$ be an integer, and let
\(
    \varepsilon \in \left[0,\frac{R_{\max}^2}{32C}\right].
\)
For every pair of positive integers $N_x,N_a$ with $N = N_xN_a \geq Cd$, there exists an instance
\(
    (\mathcal{X}, \mathcal{A}, \rho, \piref, \mathcal{F}, \mathcal{R}^\star)
\)
with
\(
    (|\mathcal{A}|, |\mathcal{F}|, |\mathcal{X}|) = \left(2, 2^d, d + \max\left\{1, \left\lceil\frac{43N\varepsilon}{R_{\max}^2}\right\rceil\right\}\right).
\)
All observed rewards and all functions in $\mathcal{F}$ take values in $[0,R_{\max}]$, and the following hold.

\smallskip
\noindent\textup{(i)} The coverage coefficients satisfy $C_{\mathcal{F}}^{\infty} = C$ and $C^\star(r^\star) \leq C$ for every $r^\star \in \mathcal{R}^\star$.
Every true reward has approximation error exactly $\varepsilon$:
\[
    \inf_{f \in \mathcal{F}}\E_{x,a}\left[\left(f(X,A) - r^\star(X,A)\right)^2\right]
    = \varepsilon,
\]
with the infimum attained.
Conditional on the complete context--action design, the reward noise variables are independent, have mean zero, and are bounded in absolute value by $R_{\max}$.

\smallskip
\noindent\textup{(ii)} The expected minimax regret satisfies
\begin{equation}
    \inf_{\mathsf L}\max_{r^\star \in \mathcal{R}^\star}
    \E_{r^\star}^{\mathsf L}\left[\Reg_{r^\star}(\widehat{\pi}_N)\right]
    \geq \frac{1}{48}\left(\sqrt{C\varepsilon} + R_{\max}\sqrt{\frac{C\log_2|\mathcal{F}|}{N}}\right).
    \label{eq:lb-expectation}
\end{equation}

\smallskip
\noindent\textup{(iii)} The constant-probability version also holds:
\begin{equation}
    \inf_{\mathsf L}\max_{r^\star \in \mathcal{R}^\star}
    \Pr_{r^\star}^{\mathsf L}\left(\Reg_{r^\star}(\widehat{\pi}_N)
    \geq \frac1{96}\left[\sqrt{C\varepsilon} + R_{\max}\sqrt{\frac{C\log_2|\mathcal{F}|}{N}}\right]\right)
    \geq \frac{1}{8}.
    \label{eq:gmin-lb-probability}
\end{equation}
Here the infimum is over all measurable learners $\mathsf L$ that map the grouped dataset to a possibly randomized policy $\widehat{\pi}_N : \mathcal{X} \to \Delta(\mathcal{A})$.
The outer expectations and probabilities include both the data and the learner's internal randomness.
\end{theorem}

\paragraph{Related work.}
Coverage-dependent amplification of prediction error appears in the error-propagation analyses of approximate dynamic programming and in batch RL~\citep{munos2003error,munos2007performance,farahmand2010error,chen2019information}.
\citet{amortila2023optimal} characterize optimal approximation factors for misspecified linear off-policy value-function estimation, which is distinct from the offline policy-learning problem considered here.
Our statistical term is related to the single-policy-concentrability rates of~\citet{rashidinejad2022offline}, but their pointwise coverage condition differs from our expectation-form coefficient.
Theorem~\ref{thm:agnostic_lower} establishes a lower bound over a coverage-budget class; it does not identify $C^\star$ with $C_{\mathcal{F}}^{\infty}$ at every member of the family.

For the misspecification term, two related lines of work are particularly relevant.
First, lower bounds for misspecified linear models use small uniform approximation errors to make identifying a near-optimal action difficult~\citep{du2020good,lattimore2020learning,vanroy2019comments}.
Our construction shares the idea of distributing uncertainty over many coordinates, but uses a finite function class, offline observations, and an average squared-error budget.
Second, other lower bounds apply to restricted classes of learners or policies.
\citet[Theorem~2]{krishnamurthy2021adapting} prove an average-misspecification lower bound for randomized policies induced by kernels in the convex hull of a model-induced kernel class.
\citet[Proposition~2.2]{amortila2024mitigating} establish an asymptotic ERM lower bound under uniform misspecification; their subsequent discussion gives an obstruction for proper learners under $L_2$ misspecification.
Here the lower bound applies to arbitrary measurable learners, including learners that do not output a function in $\mathcal{F}$.
The number of misspecified contexts may grow with $N$, so this is a finite-sample minimax statement rather than an asymptotic impossibility result at a fixed finite instance.

\begin{remark}[No conflict with $L_\infty$-based possibility results]
\label{rem:gmin-uniform-budget}
\citet{amortila2024mitigating} show that disagreement-based regression can avoid coverage amplification relative to the uniform misspecification budget
\[
    \varepsilon_\infty = \inf_{f \in \mathcal{F}}\|f - r^\star\|_\infty.
\]
This does not contradict Theorem~\ref{thm:agnostic_lower}, which is parameterized by average squared approximation error.
On our instances, the approximation error of the matching in-class function is supported on the pairs $(y_k,a_2)$, whose total probability under $\rho\otimes\piref$ is $1/(2C)$.
Consequently,
\[
    \varepsilon_\infty = \sqrt{2C\varepsilon_{\rm aprx}}.
\]
The misspecification contribution to our lower bound is therefore of order $\varepsilon_\infty$ itself, not an additional coverage amplification of that uniform budget.
Thus the lower bound in terms of $(C,\varepsilon_{\rm aprx})$ does not imply adaptation to the finer $L_\infty$ approximation budget.
\end{remark}

\subsection{\texorpdfstring{Proof of \Cref{thm:agnostic_lower}: Agnostic Minimax Lower Bound}{Proof of Theorem G.1: Agnostic Minimax Lower Bound}}
\label{app:agnostic_proof}

Fix the parameters in Theorem~\ref{thm:agnostic_lower} and write $\bar\varepsilon = \varepsilon/R_{\max}^2$.
For any fixed learner, it suffices to find a true reward for which regret is large.
We encode the action preference at each context by an unknown sign: positive favors $a_2$, and negative favors $a_1$.
At zero gap, either action is optimal and that coordinate contributes no regret.
Considering all sign assignments gives a hypercube, reducing policy learning to recovering the unknown action-preference signs.
The reduction has two properties:
First, expected regret equals a weighted sum of sign-error probabilities, whose weights sum to $S$.
Second, data distributions for neighboring sign vectors have total variation distance at most $1/2$.
Weighted Assouad testing then gives a true reward with expected regret at least $S/4$.
Because regret is always at most $S$, the same true reward has regret at least $S/8$ with probability at least $1/8$.
The construction below makes these thresholds at least the two rates in the theorem.

We first specify the instance and derive this regret reduction.
We then verify the approximation and coverage requirements and bound the information in the grouped data.
Constants are not optimized.

\paragraph{Step 1: construct the two context groups.}
The first block will carry $d$ signs represented in $\mathcal{F}$; the second will carry $m$ signs absent from the modeled action differences.
We use normalized gaps $\Delta$ and $\Delta_0$ to control their statistical and approximation contributions, respectively.
Choose
\begin{align}
    \Delta &= \frac{1}{3}\sqrt{\frac{Cd}{N}} \leq \frac{1}{3}, \\
    \Delta_0 &= \sqrt{8C\bar\varepsilon} \leq \frac{1}{2}, \\
    m &= \max\left\{1, \left\lceil43N\bar\varepsilon\right\rceil\right\}.
\end{align}
The inequalities use $N \geq Cd$ and $C\bar\varepsilon \leq 1/32$.
The choice of $m$ spreads the misspecification over enough contexts to keep each sign difficult to identify.
When $\varepsilon = 0$, we have $\Delta_0 = 0$, and the misspecification block contributes no regret.

Set $\mathcal{X} = \{x_1,\ldots,x_d\} \cup \{y_1,\ldots,y_m\}$, $\mathcal{A} = \{a_1,a_2\}$,
and choose the context distribution and behavior policy as:
\[
    \rho(x_j) = \frac1{2d}, \qquad
    \rho(y_k) = \frac1{2m}, \qquad
    \piref(a_2 \mid x) = \frac{1}{C} \quad
    \forall x \in \mathcal{X}.
\]
The $x$-contexts yield the statistical lower bound, while the $y$-contexts yield the misspecification lower bound.
Since $C \geq 2$, we have
\begin{align}
    \piref(a_1 \mid x)
    = 1 - \frac{1}{C}
    \geq \frac{1}{C}
    = \piref(a_2 \mid x).
\end{align}

For each $\sigma \in \{-1,+1\}^d$, define
\begin{align}
    f_\sigma(x_j,a_1) &= \frac{R_{\max}}2, \\
    f_\sigma(x_j,a_2) &= \frac{R_{\max}}2(1 + \sigma_j\Delta), \\
    f_\sigma(y_k,a_1) &= f_\sigma(y_k,a_2) \\
    &= \frac{R_{\max}}2.
\end{align}
For $\tau \in \{-1,+1\}^m$, let $r^\star_{\sigma,\tau}$ agree with $f_\sigma$ everywhere except at the pairs $(y_k,a_2)$, where
\[
    r^\star_{\sigma,\tau}(y_k,a_2) = \frac{R_{\max}}2(1 + \tau_k\Delta_0).
\]
Set
\begin{align}
    \mathcal{F} &\coloneq  \{f_\sigma : \sigma \in \{-1,+1\}^d\}, \\
    \mathcal{R}^\star &\coloneq  \{r^\star_{\sigma,\tau}: (\sigma,\tau) \in \{-1,+1\}^d \times \{-1,+1\}^m\}.
\end{align}
Then $|\mathcal{F}| = 2^d$, so $d = \log_2|\mathcal{F}|$.
All function values and true mean rewards lie in $[R_{\max}/4,3R_{\max}/4]$.
When $\varepsilon > 0$, the classes $\mathcal{F}$ and $\mathcal{R}^\star$ are disjoint; when $\varepsilon = 0$, $\mathcal{R}^\star = \mathcal{F}$ as sets.

Generate the contexts and actions according to the grouped sampling scheme.
Conditional on the complete context--action design, let the rewards be independent with
\[
    r_{i,j} = R_{\max} B_{i,j}, \qquad
    B_{i,j} \sim \Ber\left(\frac{r^\star(x_i,a_{i,j})}{R_{\max}}\right).
\]
Then the conditional reward mean is $r^\star(x_i,a_{i,j})$, and
\[
    \eta_{i,j} = r_{i,j} - r^\star(x_i,a_{i,j})
\]
is conditionally mean zero and bounded in absolute value by $R_{\max}$.

\paragraph{Step 2: reduce regret to sign recovery.}
The construction makes the optimal action depend on one sign at each context.
This turns the learner's regret into a weighted testing loss.
Write $a^{+1} = a_2$ and $a^{-1} = a_1$.
Under $r^\star_{\sigma,\tau}$, the optimal action at $x_j$ is $a^{\sigma_j}$, with gap $R_{\max}\Delta/2$.
At $y_k$, action $a^{\tau_k}$ is optimal, with gap $R_{\max}\Delta_0/2$; when $\Delta_0 = 0$, both actions are optimal and this gap is zero.
For every possibly randomized policy $\widehat{\pi} : \mathcal{X} \to \Delta(\mathcal{A})$,
\begin{equation}
    \Reg_{\sigma,\tau}(\widehat{\pi}) = \sum_{j=1}^d \alpha_j\widehat{\pi}(a^{-\sigma_j} \mid x_j) + \sum_{k=1}^m \beta_k\widehat{\pi}(a^{-\tau_k} \mid y_k),
    \label{eq:app-regret}
\end{equation}
where
\[
    \alpha_j = \frac{R_{\max}\Delta}{4d}, \qquad
    \beta_k = \frac{R_{\max}\Delta_0}{4m}.
\]
Each weight is the probability of its context multiplied by the action gap, so a wrong sign contributes exactly the corresponding regret.
In particular, every policy satisfies $0 \leq \Reg_{\sigma,\tau}(\widehat{\pi}) \leq S$, where
\begin{align}
    S = \sum_{j=1}^d \alpha_j + \sum_{k=1}^m \beta_k
    = \frac{R_{\max}(\Delta + \Delta_0)}4.
\end{align}
To lower-bound the weighted error, we use the following weighted version of Assouad's lemma \citep{assouad1983,yu1997lecam}.
In particular, we compare pairs of true rewards that differ at only one context.
If their data distributions are close, a learner cannot reliably choose the correct action under both rewards.

\begin{lemma}[Weighted Assouad's lemma]
\label{lem:gmin-assouad}
Let $\{P_\omega : \omega \in \{-1,+1\}^K\}$ be probability measures on a common measurable space, and let $w_1,\ldots,w_K \geq 0$.
Suppose that, for each coordinate $\ell$, there is $\delta_\ell \in [0,1]$ such that
\[
    D_{\mathrm{TV}}(P_\omega,P_{\omega'}) \leq \delta_\ell
\]
whenever $\omega$ and $\omega'$ differ only at coordinate $\ell$.
Then every possibly randomized estimator $\widehat\omega \in \{-1,+1\}^K$ based on an observation from $P_\omega$ satisfies
\[
    \max_{\omega \in \{-1,+1\}^K}\sum_{\ell=1}^K w_\ell \Pr_\omega(\widehat\omega_\ell \ne \omega_\ell)
    \geq \frac{1}{2}\sum_{\ell=1}^K w_\ell(1 - \delta_\ell).
\]
\end{lemma}

\begin{proof}
Pair each $\omega$ with its neighbor obtained by flipping coordinate $\ell$.
The sum of their testing error probabilities is at least $1 - D_{\mathrm{TV}}(P_\omega,P_{\omega'})$.
Averaging over all sign vectors gives coordinate-wise error at least $(1 - \delta_\ell)/2$.
Multiply by $w_\ell$, sum over $\ell$, and use that the maximum is at least the uniform average.
Randomized tests obey the same inequality because an independent random seed does not change total variation.
\end{proof}

To apply the lemma to an arbitrary learner, condition on its output policy and draw $A_j \sim \widehat{\pi}_N(\cdot \mid x_j)$.
Set $\widehat{\sigma}_j = +1$ if $A_j = a_2$, and set $\widehat{\sigma}_j = -1$ otherwise.
Define $\widehat\tau_k$ analogously at each $y_k$.
Equation~\eqref{eq:app-regret} shows that the weighted testing risk equals the learner's expected regret.
This reduction imposes no restriction to predictors in $\mathcal{F}$ or to policies induced by such predictors.
It remains to verify the theorem's budgets and the neighboring-distribution condition in Lemma~\ref{lem:gmin-assouad}.

\paragraph{Step 3: verify coverage and exact misspecification.}
We now check that the instance belongs to the class in the theorem.
At each $x_j$, both actions maximize some member of $\mathcal{F}$.
At every $y_k$, both actions tie under every member of $\mathcal{F}$.
Thus $\mathcal{A}_{\mathcal{F}}(x) = \mathcal{A}$ at every context, and
\begin{align}
    C_{\mathcal{F}}^{\infty} = \E_X\left[\max_{a \in \mathcal{A}}\frac1{\piref(a \mid X)}\right]
    = C.
\end{align}
For every deterministic policy $\pi$, including any optimal policy,
\[
    C(\pi) \leq C.
\]
This verifies the coverage-budget assertion in part~\textup{(i)}.
The precise dependence of $C^\star$ on the true reward is recorded in Remark~\ref{rem:gmin-coverage} below.

\smallskip
\noindent\emph{Exact misspecification.}
For any $r^\star = r^\star_{\sigma,\tau}$ and any $f_{\sigma'} \in \mathcal{F}$,
\begin{align}
    \E_{x,a}\left[(f_{\sigma'} - r^\star)^2\right] &= \sum_{j=1}^d\frac1{2d}\frac{1}{C}\frac{R_{\max}^2\Delta^2}{4}(\sigma_j' - \sigma_j)^2 + \sum_{k=1}^m\frac1{2m}\frac{1}{C}\frac{R_{\max}^2\tau_k^2\Delta_0^2}{4} \notag \\
    &= \frac{R_{\max}^2\Delta^2}{2Cd}d_{\rm H}(\sigma',\sigma) + \frac{R_{\max}^2\Delta_0^2}{8C},
    \label{eq:gmin-exact-misspecification}
\end{align}
where $d_{\rm H}$ denotes the unnormalized Hamming distance.
The infimum is attained at $f_\sigma$, and hence
\begin{align}
    \varepsilon_{\rm aprx} &= \frac{R_{\max}^2\Delta_0^2}{8C}
    = \varepsilon.
\end{align}
This completes part~\textup{(i)}.

\smallskip
\noindent\emph{Uniform approximation error.}
The same construction gives
\begin{align}
    \inf_{f \in \mathcal{F}}\|f - r^\star\|_\infty
    = \frac{R_{\max}\Delta_0}{2}
    = \sqrt{2C\varepsilon}.
\end{align}
In particular, every candidate has this absolute error on every $(y_k,a_2)$, and $f_\sigma$ agrees with $r^\star_{\sigma,\tau}$ on all other pairs.

\paragraph{Step 4: bound the information in the grouped data.}
We next show that changing one sign changes the data distribution only slightly.
Although the independent sampling units are context blocks, conditional independence of the rewards lets the KL divergence scale with the total number of observations $N = N_xN_a$.
Let $P_{\sigma,\tau}$ denote the joint law of the grouped dataset under $r^\star_{\sigma,\tau}$.
The $N$ triplets are not generally independent, but the context--action design has the same law under every hypothesis.
Conditional on this design, the reward law is a product by construction.
Consequently, writing $u = (\sigma,\tau)$ and $v = (\sigma',\tau')$, the KL chain rule gives
\begin{align}
    \KL(P_u\|P_v) &= \E_{\mathrm{design}}\sum_{i=1}^{N_x}\sum_{j=1}^{N_a}\KL\left(\frac{r_u^\star(x_i,a_{i,j})}{R_{\max}}, \frac{r_v^\star(x_i,a_{i,j})}{R_{\max}}\right) \notag \\
    &= N\E_{x,a}\left[\KL\left(\frac{r_u^\star(X,A)}{R_{\max}}, \frac{r_v^\star(X,A)}{R_{\max}}\right)\right].
    \label{eq:gmin-grouped-kl}
\end{align}
The final equality uses linearity of expectation and the common marginal law of each observed context--action pair.
It does not require the $N$ triplets to be independent.
For Bernoulli means $p,q \in (0,1)$, we use
\[
    \KL(p,q) \leq \frac{(p - q)^2}{q(1 - q)},
\]
which follows by applying $\log t \leq t - 1$ to the two terms in Bernoulli KL.
We also use Pinsker's inequality,
\[
    D_{\mathrm{TV}}(P,Q) \leq \sqrt{\frac{\KL(P\|Q)}2}.
\]

\smallskip
\noindent\emph{$\sigma$-neighbors.}
Write $\sigma^{\oplus j}$ for $\sigma$ with its $j$th sign flipped.
If two sign vectors differ only at coordinate $j \in [d]$, the reward laws differ only at $(x_j,a_2)$, an event of probability $1/(2dC)$.
At that pair, the Bernoulli parameters are $(1 + \Delta)/2$ and $(1 - \Delta)/2$.
Since $\Delta^2 = Cd/(9N) \leq 1/9$,
\begin{align}
    \KL(P_{\sigma,\tau}\|P_{\sigma^{\oplus j},\tau})
    &\leq \frac{N}{2dC}\frac{4\Delta^2}{1 - \Delta^2} \\
    &\leq \frac{N}{2dC}\frac{4Cd}{9N}\frac{9}{8} \\
    &= \frac{1}{4}.
\end{align}
Therefore,
\begin{align}
    D_{\mathrm{TV}}(P_{\sigma,\tau},P_{\sigma^{\oplus j},\tau}) \leq \sqrt{\frac{1}{8}} < \frac{1}{2}.
\end{align}

\smallskip
\noindent\emph{$\tau$-neighbors.}
Write $\tau^{\oplus k}$ for $\tau$ with its $k$th sign flipped.
If the sign vectors differ only at coordinate $k \in [m]$, the reward laws differ only at $(y_k,a_2)$, an event of probability $1/(2mC)$.
Since $\Delta_0^2 = 8C\bar\varepsilon \leq 1/4$ and $m \geq (128/3)N\bar\varepsilon$,
\begin{align}
    \KL(P_{\sigma,\tau}\|P_{\sigma,\tau^{\oplus k}})
    &\leq \frac{N}{2mC}\frac{4\Delta_0^2}{1 - \Delta_0^2} \\
    &\leq \frac{N}{2mC}\frac{4\cdot 8C\bar\varepsilon}{3/4} \\
    &= \frac{64N\bar\varepsilon}{3m} \\
    &\leq \frac{1}{2}.
\end{align}
Thus their total variation distance is at most $1/2$.
The condition on $m$ follows from $m \geq 43N\bar\varepsilon$ and $43 \geq 128/3$.
When $\varepsilon = 0$, neighboring laws in the $\tau$ coordinates coincide, and the bound is immediate.

\paragraph{Step 5: conclude the expectation and probability bounds.}
Apply Lemma~\ref{lem:gmin-assouad} with $K = d + m$ and parameter $\omega = (\sigma,\tau)$.
Use weights
\[
    (\alpha_1,\ldots,\alpha_d,\beta_1,\ldots,\beta_m)
\]
and $\delta_\ell = 1/2$ for every coordinate.
Use the sign estimator from Step~2, whose weighted testing risk equals the expected policy regret.
Hence, for every learner,
\[
\begin{aligned}
    \max_{\sigma,\tau}\E_{\sigma,\tau}\left[\Reg_{\sigma,\tau}(\widehat{\pi}_N)\right] &\geq \frac{1}{4}\left(\sum_{j=1}^d\alpha_j + \sum_{k=1}^m\beta_k\right) \\
    &= \frac{S}{4} \\
    &= \frac{R_{\max}(\Delta + \Delta_0)}{16} \\
    &= \frac{R_{\max}}{48}\sqrt{\frac{Cd}{N}} + \frac{\sqrt{C\varepsilon}}{4\sqrt2} \\
    &\geq \frac{1}{48}\left(\sqrt{C\varepsilon} + R_{\max}\sqrt{\frac{Cd}{N}}\right).
\end{aligned}
\]
Substituting $d = \log_2|\mathcal{F}|$ and taking the infimum over learners proves part~\textup{(ii)}.
The argument also allows the learner to know $(\mathcal{X},\mathcal{A},\rho,\piref,\mathcal{F},\mathcal{R}^\star,\varepsilon,N_x,N_a)$ in advance, but not the unknown true reward.

\smallskip
\noindent\emph{Constant-probability lower bound.}
Fix a learner, choose $(\sigma^\star,\tau^\star)$ attaining the preceding maximum, and write
\[
    Z = \Reg_{\sigma^\star,\tau^\star}(\widehat{\pi}_N).
\]
The randomness in $Z$ includes the dataset and any internal randomness used to produce the policy.
Action randomization is already averaged in the definition of policy regret.
By Eqn.~\eqref{eq:app-regret}, $0 \leq Z \leq S$ almost surely and $\E[Z] \geq S/4$.
For every $t \in (0,S)$,
\begin{align}
    \E[Z] &\leq t\Pr(Z < t) + S\Pr(Z \geq t) \\
    &\leq t + S\Pr(Z \geq t).
\end{align}
Taking $t = S/8$ gives
\begin{align}
    \Pr(Z \geq S/8) &\geq \frac{S/4 - S/8}{S}
    = \frac{1}{8}.
\end{align}
Moreover,
\begin{align}
    \frac{S}{8} &= \frac{R_{\max}(\Delta + \Delta_0)}{32} \\
    &= \frac{R_{\max}}{96}\sqrt{\frac{Cd}{N}} + \frac{\sqrt{C\varepsilon}}{8\sqrt2} \\
    &\geq \frac{1}{96}\left(\sqrt{C\varepsilon} + R_{\max}\sqrt{\frac{Cd}{N}}\right).
\end{align}
Taking the infimum over learners proves part~\textup{(iii)}.
This completes the proof.

\begin{remark}[The expectation-form coverage coefficient]
\label{rem:gmin-coverage}
The behavior policy is identical across contexts, so the constructed family satisfies $C_{\mathcal{F}}^{\infty} = C$.
For an optimal deterministic policy, write
\[
    t = \rho\{x : \pi^\star(x) = a_2\}.
\]
Then
\begin{align}
    C^\star &= tC + (1 - t)\frac{C}{C - 1}
    \leq C.
\end{align}
For $\varepsilon > 0$, the optimal actions are unique and
\[
    t = \frac1{2d}\sum_{j=1}^d\mathbf1\{\sigma_j = +1\} + \frac1{2m}\sum_{k=1}^m\mathbf1\{\tau_k = +1\}.
\]
Thus $C^\star$ varies with the true reward, whereas $C_{\mathcal{F}}^{\infty}$ does not.
For example, when $\varepsilon > 0$ and all signs are negative, $C^\star = C/(C - 1)$ while $C_{\mathcal{F}}^{\infty} = C$, so the two coefficients can differ arbitrarily as $C$ grows.
This observation does not assert that this particular member attains the minimax lower bound: a learner that always chooses $a_1$ is optimal on it.
The theorem establishes a worst-case lower bound over the class $C^\star \leq C$, not a hard-instance lower bound with prescribed $C^\star = C \ll C_{\mathcal{F}}^{\infty}$.
\end{remark}

%% file: 907Experiments.tex
\section{\texorpdfstring{Deferred Experimental Details}{Deferred Experimental Details}}

Here, we provide the missing details from \Cref{subsec:linear_comparison}.
We first give the complete construction and diagnostic results for the controlled synthetic experiment, then describe the LLM response-selection experiment, and finally present the downstream \MCTS mathematical-reasoning experiment as promised.
The first two experiments use the squared-loss \VR and \VDR objectives studied in the main text.
The \MCTS experiment instead uses a binary-preference variant of \VDR and provides a complementary evaluation of relative supervision over absolute supervision.

\subsection{\texorpdfstring{Controlled Synthetic Linear-Bandit Experiment}{Controlled Synthetic Linear-Bandit Experiment}}
\label{app:controlled-synthetic-experiments}

\Cref{subsec:linear_comparison} identifies two effects that can decide
between \VR and \VDR for linear classes.  First, even when the model is well
specified, \VDR learns only from within-context differences and can be hurt
when the behavior policy rarely compares the actions that matter.  Second, an
action-independent nuisance $b^\star$ misspecifies \VR but cancels from every
\VDR target.  This experiment constructs a family of linear contextual bandits
in which the two effects are controlled by separate parameters: the
cross-face probability $\delta$ for the first and the nuisance magnitude
$\beta$ for the second.  A third parameter, the alignment $\rho$, controls the
direction of the true value.
We first specify the problem, then explain which quantities of the theory each
parameter controls, and finally describe the estimators, protocol, and results.
Throughout this subsection, $\rho$ and $\delta$ denote these experimental
parameters, not the context distribution or the confidence level of the main
text.

\paragraph{Contexts, actions, and features.}
Let $d_x=5$ and $d_a=3$.  The context space is the unit sphere
$\gX=\sS^{d_x-1}=\{x\in\sR^{d_x}:\lVert x\rVert_2=1\}$, and contexts are drawn
from the uniform distribution $\rho = \operatorname{Unif}(\sS^{d_x-1})$.
The context
space is therefore continuous, and each data set contains $N_x$ independent
draws from it; we sample $X=Z/\lVert Z\rVert_2$ with
$Z\sim\mathcal N(\vzero,\mI_{d_x})$.
The action space consists of the eight normalized vertices of the cube in
$\sR^{d_a}$,
\begin{equation}
    \gA
    =\left\{
        a=\tfrac{1}{\sqrt 3}(s_1,s_2,s_3)^\top:
        s_1,s_2,s_3\in\{-1,+1\}
    \right\},
\end{equation}
and we identify each action with its vector, writing $s_k(a)\coloneq\sign(a_k)$.
We use the bilinear feature map and the corresponding linear class
\begin{equation}
    \vphi(x,a)=\operatorname{vec}(x a^\top)\in\sR^{d},
    \qquad
    f_{\mW}(x,a)=\langle\vphi(x,a),\operatorname{vec}(\mW)\rangle
    =x^\top\mW a,
    \qquad
    \mW\in\sR^{d_x\times d_a},
\end{equation}
so that $d=d_xd_a=15$.

\paragraph{Reward model.}
In each independently generated problem instance, we draw
$\vu^\star\sim\operatorname{Unif}(\sS^{d_x-1})$, independently of all
contexts, and keep it fixed within the instance. For $\rho\in[0,1]$, let
\begin{equation}
    \vc_\rho=\sqrt\rho\,\ve_1+\sqrt{1-\rho}\,\ve_2 \in \sR^{d_a},
    \qquad
    \mW_\rho^\star=\sqrt{d_xd_a}\,\vu^\star \vc_\rho^\top,
    \qquad
    f^\star\coloneq f_{\mW_\rho^\star}.
    \label{eq:synthetic-true-weight}
\end{equation}
Thus $f^\star(x,a)=\sqrt{15}\,(x^\top\vu^\star)(\vc_\rho^\top a)$ depends on
the first action coordinate when $\rho=1$ and on the second when $\rho=0$.
The factor $\sqrt{d_xd_a}$ normalizes the signal:
$\E[f^\star(X,A)^2]=1$ when $X\sim\operatorname{Unif}(\sS^{d_x-1})$ and $A$
is uniform on $\gA$, and also under the behavior policy below.
The reward is $r^\star=f^\star+b^\star$, where the nuisance
$b^\star(x)=\beta h(x)$ with $\beta\geq0$ does not depend on the action and
\begin{equation}
    h(x)=
    \frac{x_1^2-1/d_x}
    {\sqrt{2(d_x-1)/(d_x^2(d_x+2))}}.
\end{equation}
The constants are chosen so that $\E[h(X)]=0$ and $\E[h(X)^2]=1$ for
$X\sim\operatorname{Unif}(\sS^{d_x-1})$, so $\beta$ is the root mean square of
$b^\star$ under the context distribution.
This reward satisfies \Cref{asm:weak_realizability} with
$f^\star\in\gF_{\rm lin}$ for every $\beta$.  However, $\gF_{\rm lin}$ contains
no nonzero action-independent function (because
$x^\top\mW a=-x^\top\mW(-a)$ and $-a\in\gA$), so $r^\star\in\gF_{\rm lin}$ if
and only if $\beta=0$.  In other words, the model is well specified for \VR
exactly when $\beta=0$.

\paragraph{Behavior policy and data.}
A \emph{face} of the cube is the set of four actions sharing the same first
sign $s_1$.  Each context favors the face
$\tau(x)\coloneq\sign(x_1)\in\{-1,+1\}$ (with $\sign(0)\coloneq1$, a
probability-zero event).  For $\delta\in[0,1/2]$, the behavior policy is
\begin{equation}
    \piref(a\mid x)
    =
    \begin{cases}
        (1-\delta)/4, & s_1(a)=\tau(x),\\
        \delta/4, & s_1(a)=-\tau(x).
    \end{cases}
    \label{eq:synthetic-face-policy}
\end{equation}
That is, $A\sim\piref(\cdot\mid x)$ lies on the favored face with probability
$1-\delta$ and on the opposite face with probability $\delta$, and, given its
face, $s_2(A)$ and $s_3(A)$ are independent uniform signs.  Every action has
positive probability whenever $\delta>0$; the case $\delta=0$, in which the
opposite face is never observed, is included only as a limiting case.
Following \Cref{sec:introduction}, each data set consists of
$x_i\overset{\rm i.i.d.}{\sim}\operatorname{Unif}(\sS^{d_x-1})$ for
$i\in[N_x]$, $a_{i,j}\overset{\rm i.i.d.}{\sim}\piref(\cdot\mid x_i)$ for
$j\in[N_a]$, and
\begin{equation}
    r_{i,j}
    =f^\star(x_i,a_{i,j})+\beta h(x_i)+\eta_{i,j},
    \qquad
    \eta_{i,j}=\sigma\xi_{i,j},
    \label{eq:synthetic-reward}
\end{equation}
where the $\xi_{i,j}$ are independent Rademacher signs, independent of all
contexts and actions.  We use $(N_x,N_a)=(200,8)$ and $\sigma=0.5$.

We now explain $\delta$, $\rho$, and $\beta$.
Below, $X\sim\operatorname{Unif}(\sS^{d_x-1})$ and $A,A'\overset{\rm i.i.d.}{\sim}\piref(\cdot\mid X)$.

\paragraph{Role of $\delta$: design constants.}
The conditional action moments are
\begin{equation}
    \E[A\mid X=x]
    =\frac{(1-2\delta)\tau(x)}{\sqrt3}\ve_1,
    \qquad
    \Cov(A\mid X=x)
    =\frac13\operatorname{diag}\!\left(4\delta(1-\delta),1,1\right)
    \eqqcolon\mD_\delta.
    \label{eq:synthetic-action-moments}
\end{equation}
Hence $\E[AA^\top\mid X]=\mI_{d_a}/d_a$ for every $\delta$, so that
$\mSigma=\mI_d/d$ and $\lVert\vz(X,A)\rVert_2=\sqrt d$: the \VR design
constant is ${\color{blue}\alpha_0}=1$ in \Cref{asm:linear-vr-design},
regardless of $\delta$.  The centered covariance, in contrast, is
$\mSigma_{\rm C}=\mD_\delta\otimes(\mI_{d_x}/d_x)$ (in the ordering of
$\operatorname{vec}$), and its eigenvalue along the first action coordinate is
proportional to $4\delta(1-\delta)$.  For $\delta>0$, we have
${\color{blue}d_{\rm C}}=d$, and the smallest constant admissible in
\Cref{asm:linear-vdr-design} is
\begin{equation}
    {\color{blue}\alpha_{\rm C}}
    =\sqrt{\frac{8+1/(\delta(1-\delta))}{6}}
    =\Theta(\delta^{-1/2})
    \quad\text{as }\delta\to0.
\end{equation}
At $\delta=0$, the first action coordinate lies entirely in
$\ker(\mSigma_{\rm C})$ and ${\color{blue}d_{\rm C}}=10$.  This is a
contextual version of the first example in \Cref{subsec:linear_comparison}.
Note that small $\delta$ does not mean little within-context variation: the
total variation $\operatorname{tr}\Cov(A\mid X)=\frac23+\frac43\delta(1-\delta)$
is at least $2/3$ for every $\delta$, but almost all of it lies in the second
and third action coordinates.

\paragraph{Role of $\rho$: alignment with within-context variation.}
Whether the first coordinate matters depends on the direction of the true
value. Let $m(x)\coloneq\E[f^\star(x,A)\mid X=x]$, so that the centered value
is $\gC f^\star(x,a)=f^\star(x,a)-m(x)$. Since $\E[f^\star(X,A)^2]=1$, the
conditional mean and centered component account for the following fractions
of the target's second moment:
\begin{align}
    \kappa_{\rm mean}(\rho,\delta)
    &\coloneq\E[m(X)^2]
    =\rho(1-2\delta)^2,\\
    \kappa_{\rm diff}(\rho,\delta)
    &\coloneq\E\bigl[(\gC f^\star(X,A))^2\bigr]
    =\frac12\E\bigl[(f^\star(X,A)-f^\star(X,A'))^2\bigr]
    =1-\rho(1-2\delta)^2,
    \label{eq:synthetic-directional-information}
\end{align}
which satisfy $\kappa_{\rm mean}+\kappa_{\rm diff}=1$.  Because the \VDR loss
depends on the data only through within-context reward differences,
$\kappa_{\rm diff}$ is the fraction of the signal that \VDR observes; the
remaining fraction $\kappa_{\rm mean}$ is visible only through absolute reward
levels, which \VR also uses.  When $\rho=1$, the optimal face at $x$ is
$\sign(x^\top\vu^\star)$, which differs from the favored face $\tau(x)$ for a
constant fraction of contexts, yet only $\kappa_{\rm diff}=4\delta(1-\delta)$
of the signal distinguishes the two faces within a context.  At
$(\rho,\delta)=(1,0)$, $\kappa_{\rm diff}=0$: all observed actions at a
context have the same true value, and \VDR sees only noise.

\paragraph{Role of $\beta$: misspecification of \VR.}
The nuisance affects \VR only through its correlation with the features.  By
\eqref{eq:synthetic-action-moments} and symmetry,
\begin{equation}
    \E[b^\star(X)\vphi(X,A)]
    =\frac{\beta(1-2\delta)}{\sqrt3}\,
    \E\bigl[h(X)\,|X_1|\bigr]\,
    \operatorname{vec}(\ve_1\ve_1^\top),
\end{equation}
where $\E[h(X)|X_1|]>0$ because $h$ is centered and increasing in $|X_1|$.
Hence the population \VR fit is $\mW_\rho^\star$ plus a multiple of
$\ve_1\ve_1^\top$ proportional to $\beta(1-2\delta)$.  This shift acts on the
first action coordinate, which is exactly the coordinate that separates the two
faces when $\rho=1$, and it vanishes only when $\beta=0$ or $\delta=1/2$.  The
same nuisance cancels from every \VDR target.

Together, these calculations suggest two regimes at $\rho=1$.  When $\beta=0$
and $\delta$ is small, \VR is well specified with ${\color{blue}\alpha_0}=1$,
whereas \VDR has a large ${\color{blue}\alpha_{\rm C}}$ and observes little of
the signal, so \VR should be favored.  As $\beta$ grows, the bias of \VR grows
in proportion while \VDR is unaffected, so \VDR should eventually be favored.

\paragraph{Estimators.}
Let $L_{\VR}(\mW)\coloneq\widehat\gL(f_{\mW})$ with $\gT^{\VR}$ as in
\Cref{sec:localised}, and let $L_{\VDR}(\mW)\coloneq\frac12\widehat\gL(f_{\mW})$
with $\gT^{\VDR}$.  Writing $e_{i,j}(\mW)\coloneq f_{\mW}(x_i,a_{i,j})-r_{i,j}$
and $\bar e_i(\mW)\coloneq N_a^{-1}\sum_{j}e_{i,j}(\mW)$,
\begin{equation}
    L_{\VR}(\mW)
    =\frac{1}{N_xN_a}\sum_{i=1}^{N_x}\sum_{j=1}^{N_a}e_{i,j}(\mW)^2,
    \qquad
    L_{\VDR}(\mW)
    =\frac{1}{N_x(N_a-1)}\sum_{i=1}^{N_x}\sum_{j=1}^{N_a}
    \bigl(e_{i,j}(\mW)-\bar e_i(\mW)\bigr)^2,
    \label{eq:synthetic-losses}
\end{equation}
where the second expression follows from
$\sum_{j<k}(e_{i,j}-e_{i,k})^2=N_a\sum_j(e_{i,j}-\bar e_i)^2$.
The extra factor $1/2$ in $L_{\VDR}$ does not change its minimizer; it gives the two
losses the same expected value, $\sigma^2$, when evaluated at the true reward
with $\beta=0$, so that they are on a common scale when mixed below.  The \VR and \VDR estimators of
\Cref{subsec:linear_vr,subsec:linear_vdr} minimize these losses over $\mW$;
both are computed in closed form from the corresponding sufficient statistics.

\paragraph{Data-driven selection between \VR and \VDR.}
Since neither objective dominates, a practitioner who does not know $\beta$ or
$\delta$ would want to choose between them from the data.  We therefore also
consider the mixed objective
\begin{equation}
    \widehat \mW_\alpha
    \in\operatorname*{arg\,min}_{\mW}
    \left\{(1-\alpha)L_{\VR}(\mW)
    +\alpha L_{\VDR}(\mW)
    +\lambda_\alpha\lVert \mW\rVert_F^2\right\},
    \qquad
    \alpha\in\{0,0.05,\ldots,0.95,1\},
    \label{eq:synthetic-mixed-loss}
\end{equation}
which interpolates between \VR ($\alpha=0$) and \VDR ($\alpha=1$).  Here
\begin{equation}
    \lambda_\alpha
    =10^{-8}\max\left\{\frac{\operatorname{tr}(\mG_\alpha)}{d},1\right\},
    \qquad
    \mG_\alpha\coloneq(1-\alpha)\mG_{\VR}+\alpha\mG_{\VDR},
\end{equation}
where $\mG_{\VR}$ and $\mG_{\VDR}$ are the empirical Gram matrices of the two
quadratic losses in \eqref{eq:synthetic-losses}.
The $\ell_2$-regularization serves only
to stabilize the linear solve; in our design $\operatorname{tr}(\mG_\alpha)/d<1$,
so $\lambda_\alpha=10^{-8}$ in practice.  The same term is used at the endpoints
$\alpha\in\{0,1\}$, so the reported \VR and \VDR fits are unregularized up to
this negligible ridge.  We choose
$\alpha$ by five-fold cross-validation (CV) over contexts.  Specifically, we
partition the $N_x=200$ context blocks into five folds of 40 contexts each;
for each fold, we fit $\widehat\mW_\alpha$ on the other four folds and compute
$L_{\VDR}$ on the held-out fold.  The selected $\alpha$ minimizes the average
of the five held-out losses, and the reported estimator is
$\widehat\mW_\alpha$ at the selected $\alpha$, fitted on all $N_x$ contexts.
Folds are formed over whole context blocks so that held-out actions never share
a context with training actions.
We validate with the \VDR loss rather than the \VR loss because the held-out
\VR loss contains the nuisance $b^\star$ and would reward fits that absorb it,
whereas the held-out \VDR loss is unaffected by $b^\star$ and, up to scaling and a noise
constant, measures the centered prediction error that controls regret in
\Cref{prop:regret}.  We call this rule \algname{difference-CV}
(\textsf{CV-best} in \Cref{fig:directional-mechanism}).
As a reference that uses information unavailable in practice,
\algname{Oracle} chooses $\alpha$ to minimize the true regret, computed with
$\mW_\rho^\star$, on 2,000 additional simulated contexts.

\paragraph{Evaluation and protocol.}
Each fitted matrix $\widehat\mW$ defines the greedy policy
$x\mapsto\argmax_{a\in\gA}f_{\widehat\mW}(x,a)$.  We estimate its regret
\eqref{eqn:offline-regret} on 10,000 fresh contexts by enumerating all eight
actions; since $b^\star$ does not depend on the action, it does not affect the
regret.  We run two experiments.
\begin{itemize}
    \item \textbf{Crossover grid.}  With $\rho=1$, we vary
    $\delta\in\{0,0.02,0.05,0.1,0.2,0.5\}$ and
    $\beta\in\{0,0.25,0.5,1,2\}$ to test the two regimes predicted above.
    \item \textbf{Alignment sweep.}  With $(\delta,\beta)=(0.02,0)$ fixed,
    we vary $\rho\in\{0,0.25,0.5,0.75,0.9,1\}$.  The behavior policy, and
    hence the total within-context variation, is unchanged; only the direction
    of the true value moves toward the weak first coordinate.  This tests
    whether what matters is the within-context variation \emph{in the direction
    of the true value} (measured by $\kappa_{\rm diff}$) rather than its
    total amount.
\end{itemize}
For each condition, we generate eight independent problem instances and 20
training data sets per instance.  Repetitions are first averaged within each
problem instance, and the reported 95\% intervals are computed from the eight
independent instance means rather than by treating repetitions within an
instance as independent.

\paragraph{Crossover results.}
Let
$\Delta_{\rm Reg}=\Reg_{\VR}-\Reg_{\VDR}$, so that negative values favor \VR
and positive values favor \VDR.  \Cref{tab:synthetic-crossover} reports its
mean over the grid, and the top-left panel of
\Cref{fig:directional-crossover} shows the same values as a heat map.

Both predicted regimes appear.  When $\beta=0$, \VR has lower regret in all
eight problem instances for every tested $\delta>0$, so the \VR-favored regime
is not an artifact of the degenerate case $\delta=0$; the advantage shrinks as
$\delta$ grows, consistent with ${\color{blue}\alpha_{\rm C}}$ decreasing and
$\kappa_{\rm diff}$ increasing.  For each $\delta\in\{0.02,0.05,0.1,0.2\}$,
the sign of $\Delta_{\rm Reg}$ flips as $\beta$ grows.  At $\delta=0.02$, for
example, the paired mean difference and its 95\% interval are
$-0.0040\pm0.0005$ at $\beta=0$, $+0.0014\pm0.0014$ at $\beta=0.25$, and
$+0.0188\pm0.0050$ at $\beta=0.5$.  Over the same range, the population bias of \VR
grows as $0$, $0.129$, $0.259$, $0.518$, and $1.036$ for
$\beta=0,0.25,0.5,1,2$, i.e., proportional to $\beta$ as predicted.  Here the
population bias is the root-mean-square centered error of the population \VR
fit (the minimizer of the expected \VR loss under $\piref$, with the same
negligible ridge), computed on the 2,000 validation contexts, where
predictions are centered by their mean over all eight actions;\footnote{This
uniform centering differs from the centering operator $\gC$ of
\Cref{sec:introduction}, which subtracts the mean under $\piref(\cdot\mid x)$.
Both remove action-independent terms and preserve all action gaps, so neither
affects the regret.  The population \VR bias derived in the paragraph on the
role of $\beta$ is a multiple of $\ve_1\ve_1^\top$ proportional to
$\beta(1-2\delta)$, so its centered magnitude is proportional to $\beta$ under
either centering; only the constant differs.} averaged over
the eight problem instances.  The
row $\delta=0.5$ is nearly flat because the \VR bias vanishes at $\delta=1/2$.

The row $\delta=0$ illustrates the effect of missing cross-face comparisons.
Only the favored face is observed, and at $\rho=1$, the centered target
$\gC f^\star$ is identically zero because $s_1(A)=\tau(X)$ almost surely.
Thus, \VDR cannot identify the cross-face value gaps needed by the greedy
policy, and the regret guarantee of \Cref{prop:regret} is vacuous.
When $\beta=0$, \VR can still identify $\mW^\star_\rho$ because
$\mSigma\succ\vzero$: variation in the favored face across contexts allows
absolute reward levels to identify the coefficients associated with the
first action coordinate. We include $\delta=0$ as a boundary case and focus
our comparisons on $\delta\geq0.02$.

\paragraph{Data-driven selection results.}
The remaining panels of \Cref{fig:directional-crossover} evaluate
\algname{difference-CV}.  For $\delta>0$, its selected $\alpha$ follows the
same pattern as the \algname{Oracle} choice: it leans toward \VR when
$\beta=0$ and moves to $\alpha\approx1$ as $\beta$ grows.  Its regret is then
close to that of the better of \VR and \VDR throughout the grid.  At
$\delta=0$, however, \algname{difference-CV} falls behind the better
endpoint, increasingly so for larger $\beta$.  This is expected: its validation
loss is itself a \VDR loss and cannot assess a direction that no observed
difference reveals.

\paragraph{Alignment sweep results.}
In this experiment, the behavior policy and the nuisance are fixed, and only
the direction of the true value changes.  At $\rho=0$, the true value varies
only within each face, $\kappa_{\rm diff}=1$, and both methods have mean regret
$0.0003$.  For every $\rho>0$, part of the true value moves into the weak first
coordinate, and \VR has lower mean regret than \VDR
(\Cref{fig:directional-mechanism}, left).  As $\rho$ increases from $0$ to
$1$, $\kappa_{\rm diff}$ decreases from $1$ to $0.0784$ and $\kappa_{\rm mean}$
increases from $0$ to $0.9216$, matching
\eqref{eq:synthetic-directional-information} (right panel).
The results are consistent with a growing estimation disadvantage for \VDR as the
true value becomes more aligned with the weakly sampled first action
direction.

\paragraph{Summary.}
The crossover and alignment experiments illustrate the geometric
bias--variance tradeoff discussed in \Cref{subsec:linear_comparison}.
In this construction, \VR benefits from absolute reward levels under
correct specification when cross-face comparisons are rare and the true
value is aligned with the first action direction. For fixed $\delta<1/2$,
its population coefficient bias grows linearly with the nuisance magnitude
$\beta$, whereas the \VDR objective is unchanged. The observed regret
crossover shows how increasing misspecification can outweigh the estimation
advantage of using absolute levels. Across the tested grid with $\delta>0$,
cross-validation with the \VDR loss achieves regret close to that of the
better objective without knowing which regime applies.

\begin{table}[H]
    \centering
    \small
    \setlength{\tabcolsep}{5.5pt}
    \begin{tabular}{@{}rrrrrr@{}}
        \toprule
        $\delta$ & $\beta=0$ & $\beta=0.25$ & $\beta=0.5$ & $\beta=1$ & $\beta=2$ \\
        \midrule
        0.00 & $-0.8390$ & $-0.8330$ & $-0.8154$ & $-0.7472$ & $-0.5686$ \\
        0.02 & $-0.0040$ & $+0.0014$ & $+0.0188$ & $+0.0813$ & $+0.2485$ \\
        0.05 & $-0.0013$ & $+0.0036$ & $+0.0189$ & $+0.0737$ & $+0.2298$ \\
        0.10 & $-0.0007$ & $+0.0033$ & $+0.0157$ & $+0.0592$ & $+0.1978$ \\
        0.20 & $-0.0002$ & $+0.0019$ & $+0.0090$ & $+0.0340$ & $+0.1250$ \\
        0.50 & $-0.0001$ & $+0.0000$ & $+0.0002$ & $+0.0012$ & $+0.0045$ \\
        \bottomrule
    \end{tabular}
    \caption{Mean greedy-regret difference $\Delta_{\rm Reg}$ in the crossover
    grid ($\rho=1$).  Negative values favor \VR and positive values favor
    \VDR.}
    \label{tab:synthetic-crossover}
\end{table}

\begin{figure}[H]
    \centering
    \includegraphics[width=0.92\linewidth]
    {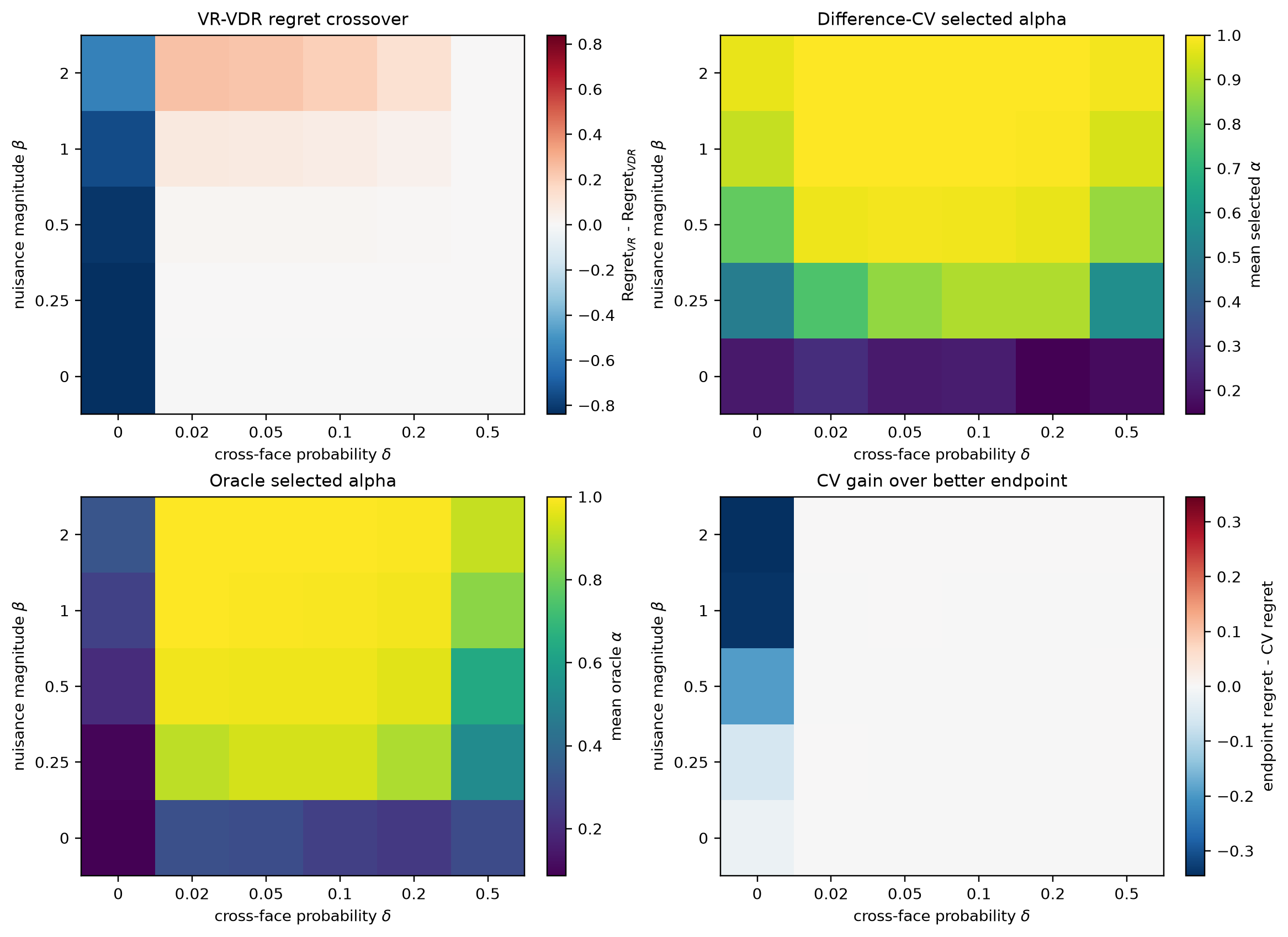}
    \caption{Crossover grid ($\rho=1$).  Top left:
    $\Reg_{\VR}-\Reg_{\VDR}$; negative values favor \VR and positive values
    favor \VDR.  Top right and bottom left: mean $\alpha$ selected by
    \algname{difference-CV} and by \algname{Oracle}.  Bottom right: regret of
    the better of \VR and \VDR minus that of \algname{difference-CV}; positive
    values mean that \algname{difference-CV} improves on both.}
    \label{fig:directional-crossover}
\end{figure}

\begin{figure}[H]
    \centering
    \includegraphics[width=0.73\linewidth]
    {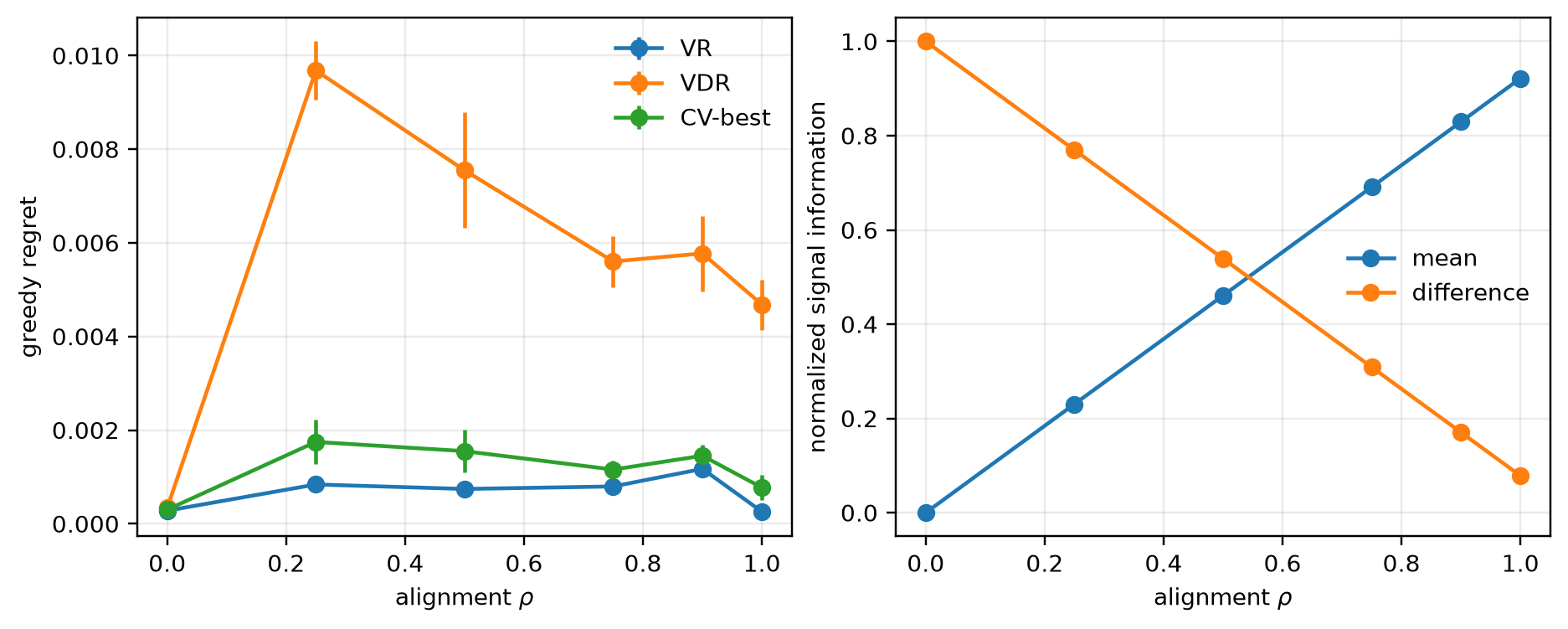}
    \caption{Alignment sweep at fixed $(\delta,\beta)=(0.02,0)$.
    \textbf{Left:} greedy-policy regret as the true value direction moves toward the
    weak first action coordinate; \textsf{CV-best} denotes
    \algname{difference-CV}.
    \textbf{Right:} the fractions $\kappa_{\rm mean}$ (``mean'') and $\kappa_{\rm diff}$ (``difference'') of the target's second moment contributed by its conditional mean and centered component, respectively; see \eqref{eq:synthetic-directional-information}.
    }
    \label{fig:directional-mechanism}
\end{figure}

\subsection{\texorpdfstring{LLM Response-Selection Experiment}{LLM Response-Selection Experiment}}
\label{app:controlled-real-experiments}

\paragraph{Data and representations.}
We curate 1,000 English prompts from
WildChat~\citep{zhao2024wildchat},
UltraFeedback~\citep{cui2023ultrafeedback},
GSM8K~\citep{cobbe2021gsm8k},
MATH~\citep{hendrycks2021math},
MBPP~\citep{austin2021mbpp},
HelpSteer2~\citep{wang24helpsteer}, and
TL;DR~\citep{stiennon2020summarize},
and cache 1,000 responses per prompt generated by
Llama-3.2-3B-Instruct~\citep{meta2024llama32}.
Each response is represented by a 3,072-dimensional feature vector
obtained by mean-pooling the generator's final-layer hidden states
over its response tokens.
Keeping the prompts, responses, and representations fixed, we repeat
the experiment with reward models
Skywork-Reward-V2-Llama-3.1-8B and
Skywork-Reward-V2-Qwen3-8B~\citep{liu2026skyworkv2}, and
ArmoRM-Llama3-8B-v0.1~\citep{wang2024armorm}.

\begin{figure}[t]
  \centering
  \includegraphics[width=\linewidth]
  {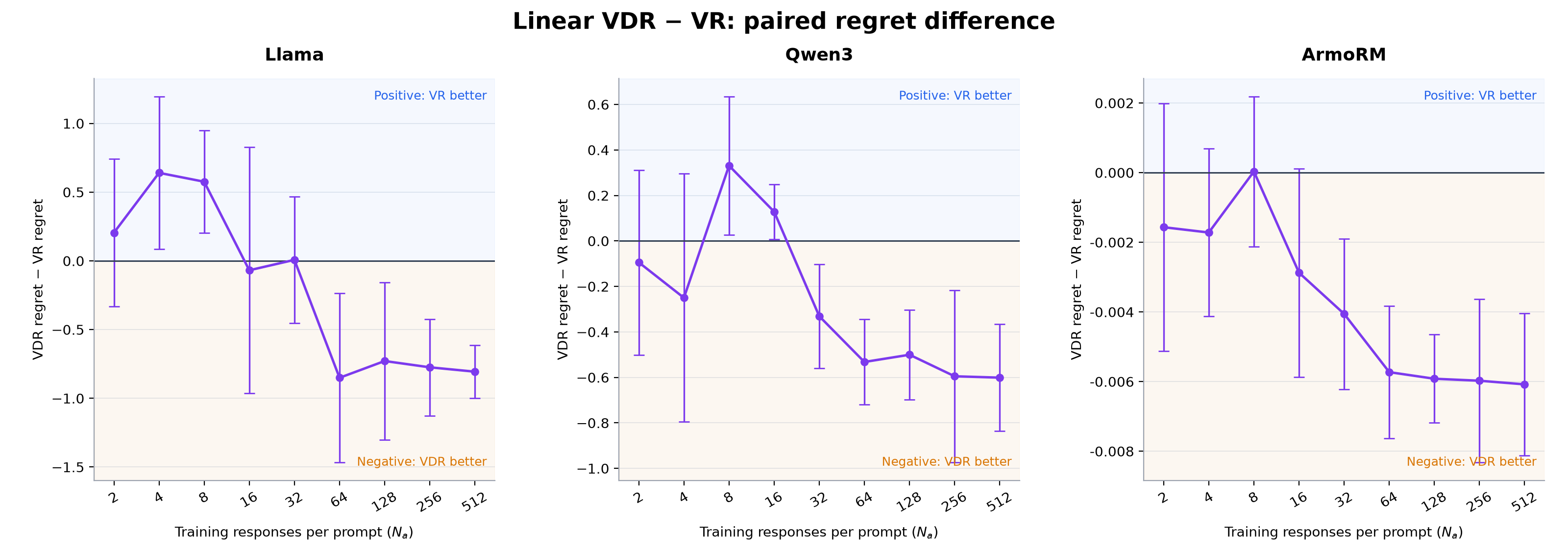}
  \caption{Paired regret differences on the 200 evaluation prompts,
  $\Reg_{\VDR}-\Reg_{\VR}$, with pointwise 95\% Student-$t$ confidence
  intervals over five seeds. Negative values favor \VDR.}
  \label{fig:empirical-paired-three-rewards}
\end{figure}

\paragraph{Training protocol.}
For each of five seeds, we split the prompts into $N_x=800$ training prompts
and 200 evaluation prompts that are not used for fitting or selecting the
regularization strengths.
We vary the number of labeled responses per training prompt over $N_a\in\{2,4,8,16,32,64,128,256,512\}.$
Within each seed and response budget, \VR and \VDR use the same sampled
responses.
To stabilize estimation in the 3,072-dimensional feature space, we use
$\ell_2$-regularized versions of the objectives in
\Cref{subsec:linear_vr,subsec:linear_vdr}.
Writing $\vphi_{i,j}=\vphi(x_i,a_{i,j})$, the estimators minimize
\begin{align}
(\widehat{\vtheta}_{\VR},\widehat c)
&\in
\operatorname*{arg\,min}_{\vtheta,c}
\left\{
\sum_{i=1}^{N_x}\sum_{j=1}^{N_a}
\bigl(r_{i,j}-\vtheta^\top\vphi_{i,j}-c\bigr)^2
+\lambda_{\VR}\lVert\vtheta\rVert_2^2
\right\},
\label{eq:empirical-vr-ridge}\\
\widehat{\vtheta}_{\VDR}
&\in
\operatorname*{arg\,min}_{\vtheta}
\left\{
\sum_{i=1}^{N_x}\sum_{(j,k)\in\mathcal P_i}
\bigl[(r_{i,j}-r_{i,k})
-\vtheta^\top(\vphi_{i,j}-\vphi_{i,k})\bigr]^2
+\lambda_{\VDR}\lVert\vtheta\rVert_2^2
\right\},
\label{eq:empirical-vdr-ridge}
\end{align}
where $\mathcal P_i$ is the set of within-prompt pairs used for prompt $x_i$.
\VDR uses all $\binom{N_a}{2}$ pairs when $N_a\leq32$ and samples
$16N_a$ pairs without replacement otherwise.
Features are standardized before fitting.
We select $\lambda_{\VR}$ and $\lambda_{\VDR}$ separately for each
reward model, method, and response budget using validation regret on a
640/160 split of the training prompts, and then refit on all 800 training
prompts.
The \VR intercept $c$ is not penalized; \VDR requires no intercept because it
cancels under differencing.

\paragraph{Evaluation and results.}
Each fitted scorer ranks all 1,000 cached responses for every evaluation
prompt and selects the response with the highest predicted score.
For an evaluation prompt $x$, offline regret is the difference between the
reward of its best cached response and that of the selected response.
We average this quantity over the 200 evaluation prompts.
The primary figure reports means and pointwise 95\% Student-$t$ confidence
intervals over the five seeds.

Because both methods use the same train--test split and sampled responses
within each seed, we also compute the seed-wise paired difference
$\Reg_{\VDR}-\Reg_{\VR}$.
\Cref{fig:empirical-paired-three-rewards} reports its mean and pointwise
95\% Student-$t$ confidence interval.
For every $N_a\geq64$, these intervals lie below zero under all three reward
models.

\subsection{\texorpdfstring{\MCTS-Based Mathematical Reasoning Experiment}{MCTS-Based Mathematical Reasoning Experiment}}
\label{app:math-reasoning-experiment}

We complement our theoretical results with an LLM-based mathematical
reasoning experiment, examining whether relative supervision can improve
downstream decision quality even when numerical rewards are available.
We compare \VR with a binary-preference variant of \VDR (which we refer to \VDR for simplicity) for guiding
Monte Carlo tree search (\MCTS).
Following the process-preference approach of \citet{guan2025rstar},
we train reward models using either search-derived $Q$-values or
within-context preferences constructed from those values.
Rather than evaluating selection from a fixed candidate pool, we measure
how these learned models guide a sequence of reasoning steps toward
correct final answers.

\paragraph{Data, labels, and representations.}

We use 1,000 competition mathematics problems from the American Mathematics
Competitions~\citep{maa_amc}: 800 AMC12 problems from the 2002--2024 contests
and 200 AIME problems from the 2000--2024 contests.\footnote{Problem statements
are taken from the public archive at
\url{https://live.poshenloh.com/past-contests}, and answer keys are
cross-checked against published solution compilations.}
 A context $x_i$ contains a problem and its reasoning prefix; $a_{i,j}$ is a
candidate next step with label $r_{i,j} \coloneq  Q(x_i,a_{i,j})$. Labels are taken
from an eight-rollout data-generation snapshot.

During data generation, final-outcome rewards ($+1$ for a correct answer and
$-1$ for a recorded unsuccessful outcome) are backed up through the search
tree. For a visited node $v$, the $Q$-value estimates eventual success under the generating policy and budget, and is defined as
\begin{equation}
  Q(v)
  = \frac{S(v)}{n(v)}
  = \frac{N_+(v)-N_-(v)}{N_+(v)+N_-(v)}
  = 2\widehat{p}(v)-1,
  \label{eq:mcts-backed-up-q}
\end{equation}
where $S(v)$ and $n(v)$ are the reward sum and visit count; $N_+$ and $N_-$
count successful and unsuccessful continuations, and $\widehat{p}$ is their
success frequency.

For each candidate step $a$, we encode the parent context $x$ (the problem,
preceding reasoning steps, and prior execution observations) followed by $a$
in the generation format, excluding $a$'s own execution output. We extract
the final-layer hidden state at the last nonpadding token of frozen
Qwen2.5-Coder-7B-Instruct~\citep{hui2024qwen25coder}. The resulting 3,584-dimensional vectors are
standardized to obtain $\phi(x,a)$, using statistics fitted on training
candidates during model selection and on training-plus-validation
candidates for the final refit.

\paragraph{Training objectives and protocol.}

A trainable linear or one-hidden-layer GELU head maps $\phi(x,a)$ to a scalar
score $f_\theta(x,a)$, while the Qwen backbone remains frozen. For \VDR, this
encoder--head combination is the process preference model (PPM); paired
candidates are scored separately by the same head. \VR learns an unpenalized
global intercept; \VDR fixes it to zero.

For each context $x_i$, let $P_i$ contain preferred--rejected pairs,
with $a_{i,j}$ preferred to $a_{i,k}$ for $(j,k)\in P_i$.
Let $E_i$ denote their unique endpoint indices, $N_{a,i}=|E_i|$,
and $m_i^{\mathrm{pos}}$ and $m_i^{\mathrm{neg}}$ denote the
dataset-recorded positive and negative role counts.
\VR regresses the numerical labels $r_{i,j}=Q(x_i,a_{i,j})$,
whereas \VDR learns pairwise orderings using the ranking loss of
\citet{guan2025rstar}, with dataset-specific normalization: with $\sigma(z) = (1 + e^{-z})^{-1}$,
\begin{align}
  \widehat{L}_{\VR}(f_\theta)
  &= \frac{1}{N_x}\sum_{i=1}^{N_x}\frac{1}{N_{a,i}}
     \sum_{j\in E_i}
     \bigl(f_\theta(x_i,a_{i,j})-r_{i,j}\bigr)^2,
     \label{eq:mcts-vr-loss} \\[3pt]
  \widehat{L}_{\VDR}(f_\theta)
  &= -\frac{1}{N_x}\sum_{i=1}^{N_x}
     \frac{1}{m_i^{\mathrm{pos}}m_i^{\mathrm{neg}}}
     \sum_{(j,k)\in P_i}
     \log \sigma\!\left(
       f_\theta(x_i,a_{i,j})-f_\theta(x_i,a_{i,k})
     \right).
     \label{eq:mcts-vdr-binary-loss}
\end{align}

We train the linear heads using ridge/SVD with penalty $\lambda\lVert w\rVert_2^2$ for \VR, and L-BFGS~\citep{l-bfgs} with $(\lambda/2)\lVert w\rVert_2^2$ penalty for \VDR; both select $\lambda=1$. MLPs use AdamW~\citep{adamw}, averaging context-wise errors
for \VR and summing role-count-weighted pair losses for \VDR.
We select the hyperparameters by validation loss, averaged over three training
seeds for MLPs, then refit on training plus validation data.

We evaluate all 100 test problems with Qwen2.5-Coder-7B-Instruct using one
\MCTS tree, two iterations, four children per expansion, maximum depth 16,
and token limits of 1,024 per step and 16,384 per context.

\begin{table}[!t]
  \centering
  \caption{Final-answer accuracy on 100 test problems under a common \MCTS
  budget. Each method uses one search seed; MLPs use one prespecified refit.}
  \label{tab:mcts-vr-vdr-accuracy}
  \begin{tabular}{@{}lc@{}}
    \toprule
    Reward head & Accuracy (\%) \\
    \midrule
    Linear \VR  & 15 \\
    Linear \VDR & 16 \\
    MLP \VR     & 10 \\
    MLP \VDR    & 13 \\
    \bottomrule
  \end{tabular}
\end{table}

\paragraph{Results and discussion.}
\VDR achieves higher observed accuracy than \VR with both linear heads
(16\% versus 15\%) and MLPs (13\% versus 10\%;
Table~\ref{tab:mcts-vr-vdr-accuracy}).

The synthetic experiment in \Cref{app:controlled-synthetic-experiments}
examines the bias--variance tradeoff by separately controlling nuisance
magnitude and the geometry of within-context comparisons.
The
response-selection experiment in \Cref{app:controlled-real-experiments}
compares the objectives on real data without isolating the mechanism
responsible for their relative performance. The present experiment asks
whether preferences derived from numerical feedback can guide search
toward correct final answers. Its observed gains provide additional
evidence that relative supervision can be useful even when numerical
rewards are available.